\documentclass[10pt,twocolumn,letterpaper]{article} 

\usepackage{iccv}      

\usepackage[dvipsnames]{xcolor}
\usepackage{lipsum}  
\usepackage{comment}
\usepackage[accsupp]{axessibility}

\newcommand{\zrv}{\mathbf{z}}

\usepackage{amsmath}
\usepackage{stmaryrd}
\usepackage{mathtools}
\usepackage{enumitem}

\newcommand{\Rcal}{\mathcal{R}}
\newcommand{\Zcal}{\mathcal{Z}}

\newcommand{\vc}[1]{\mathbf{#1}} 

\newcommand{\E}{\mathbb{E}}

\newcommand{\x}{\vc{x}}
\newcommand{\w}{\vc{w}}
\newcommand{\z}{\vc{z}}

\newcommand{\R}{{\mathbb{R}}}

\newcommand{\CSI}{\mathrm{CSI}}

\DeclareMathOperator*{\argmax}{arg\,max}

\newcommand{\changedc}[1]{{#1}}

\renewcommand{\paragraph}[1]{{\vspace{0.3mm}\noindent \bf #1}.}

\newcommand{\paragraphQ}[1]{{\vspace{0.3mm}\noindent \bf #1}?}

\newcommand{\tr}{\textrm{train}}
\newcommand{\te}{\textrm{test}}

\newcommand{\kprune}{$\kappa\textrm{-Prunability}$ }
\newcommand{\qkcomp}{$(q, \kappa)\textrm{-Compressibility}$}

\renewcommand{\paragraph}[1]{{\vspace{0.3mm}\noindent \bf #1}.}

\usepackage{xr-hyper}

\makeatletter
\newcommand*{\addFileDependency}[1]{
	\typeout{(#1)}
	\@addtofilelist{#1}
	\IfFileExists{#1}{}{\typeout{No file #1.}}
}

\makeatother

\definecolor{iccvblue}{rgb}{0.21,0.49,0.74}
\usepackage[pagebackref,breaklinks,colorlinks,allcolors=iccvblue]{hyperref}
\usepackage[capitalize]{cleveref}

\crefname{equation}{eq.}{eq.}
\Crefname{equation}{Eq.}{Eq.}
\crefname{theorem}{thm.}{thms.}
\Crefname{Theorem}{Thm.}{Thms.}
\crefname{conjecture}{conj.}{conjs.}
\Crefname{Conjecture}{Conj.}{Conjs.}
\crefname{proposition}{prop.}{props.}
\Crefname{proposition}{Prop.}{Props.}
\crefname{definition}{dfn.}{dfn.}
\Crefname{definition}{Dfn.}{Dfn.}
\crefname{remark}{remark}{remark}
\Crefname{Remark}{Remark}{Remark}
\Crefname{algorithm}{Alg.}{Alg.}

\newtheorem{thm}{Theorem}

\newtheorem{prop}{Proposition}
\newtheorem{dfn}{Definition}

\newtheorem{proof}{Proof}

\crefname{section}{Sec.}{Secs.}
\Crefname{section}{Sec.}{Secs.}
\crefname{equation}{Eq.}{Eqs.}
\Crefname{equation}{Eq.}{Eqs.}
\crefname{figure}{Fig.}{Figs.}
\Crefname{figure}{Fig.}{Figs.}
\crefname{table}{Tab.}{Tabs.}
\Crefname{table}{Tab.}{Tabs.}
\crefname{thm}{Thm.}{Thms.}
\Crefname{thm}{Thm.}{Thms.}
\crefname{conj}{Conj.}{Conjs.}
\Crefname{conj}{Conj.}{Conjs.}
\crefname{dfn}{Dfn.}{Dfns.}
\crefname{dfn}{Dfn.}{Dfns.}
\crefname{remark}{remark}{remarks}
\Crefname{Remark}{Remark}{Remarks}
\crefname{prop}{Prop.}{Prop.}
\Crefname{prop}{Prop.}{Prop.}
\Crefname{algorithm}{Alg.}{Alg.}
\crefname{appendix}{App.}{Apps.}
\Crefname{appendix}{App.}{Apps.}
\crefname{appsec}{appendix}{appendices}
\Crefname{appsec}{Appendix}{Appendices}

\def\paperID{6813} 
\def\confName{ICCV}
\def\confYear{2025}

\title{Large Learning Rates Simultaneously Achieve\\ Robustness to Spurious Correlations and Compressibility}
\def\authorBlock{
    Melih Barsbey \qquad
    Lucas Prieto \qquad
    Stefanos Zafeiriou \qquad
    Tolga Birdal \\
    Imperial College London
}

\begin{document}
\author{\authorBlock}
\maketitle
\begin{abstract}
Robustness and resource-efficiency are two highly desirable properties for modern machine learning models. However, achieving them jointly remains a challenge. In this paper, we identify high learning rates as a facilitator for simultaneously achieving robustness to spurious correlations and network compressibility. We demonstrate that large learning rates also produce desirable representation properties such as invariant feature utilization, class separation, and activation sparsity. Our findings indicate that large learning rates compare favorably to other hyperparameters and regularization methods, in consistently satisfying these properties in tandem. In addition to demonstrating the positive effect of large learning rates across diverse spurious correlation datasets, models, and optimizers, we also present strong evidence that the previously documented success of large learning rates in standard classification tasks is related to addressing hidden/rare spurious correlations in the training dataset. Our investigation of the mechanisms underlying this phenomenon reveals the importance of confident mispredictions of bias-conflicting samples under large learning rates.
\end{abstract}
    
\vspace{-4mm}\section{Introduction}
\label{sec:intro}
\begin{figure}[t]
  \begin{subfigure}{0.585\linewidth}
    \centering
\includegraphics[width=\textwidth]{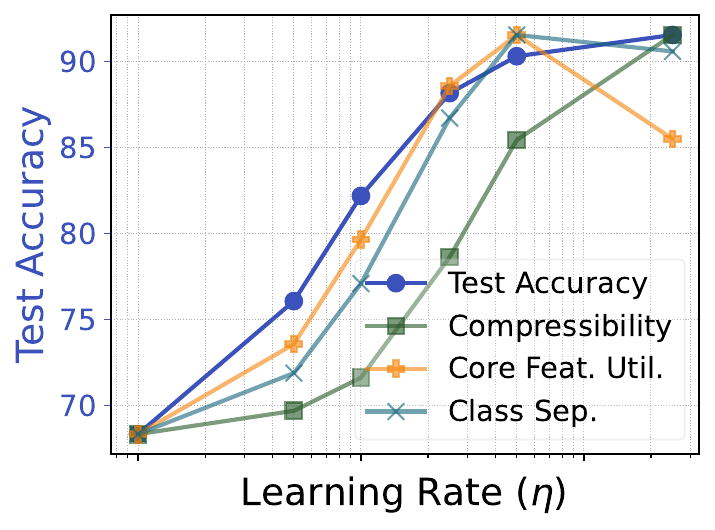}
\label{fig:teaser_cmnist_stats}
  \end{subfigure}
  \hfil
    \begin{subfigure}{0.39\linewidth}
    \centering
    \includegraphics[width=\linewidth]{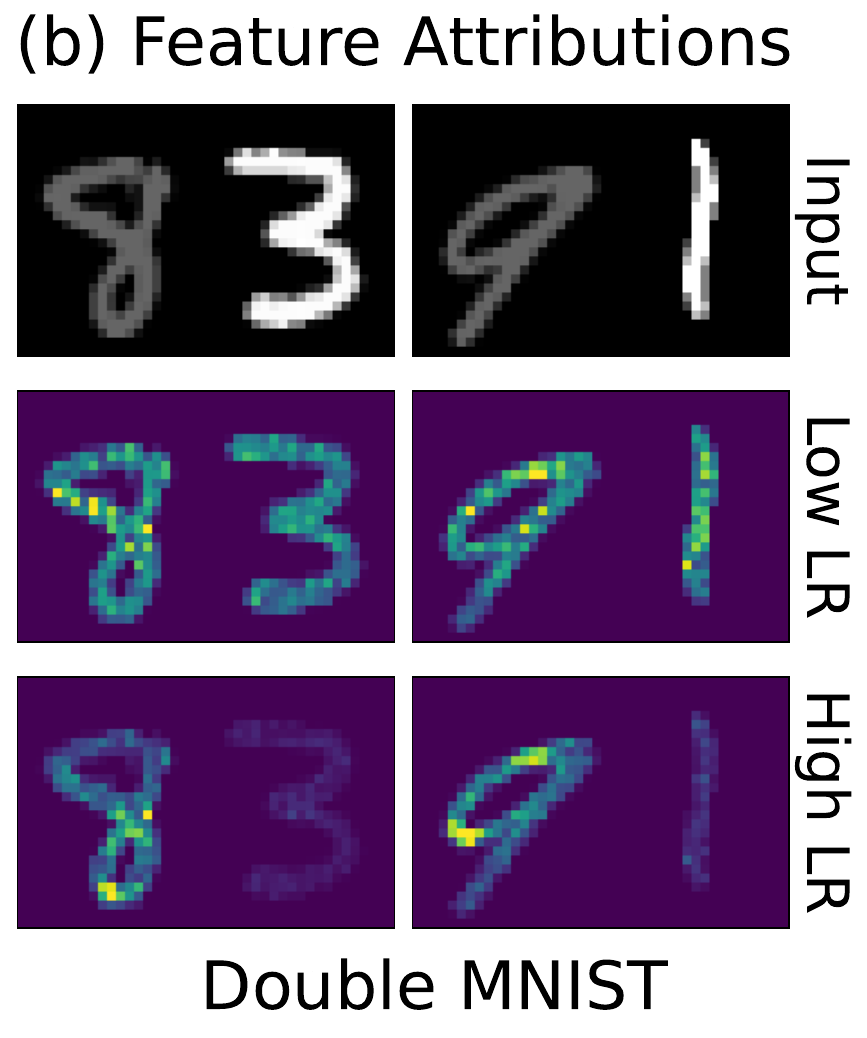}
    \vspace{0.27em}
\label{fig:teaser_cmnist_attr}
  \end{subfigure}
  
    \vspace{-2em}
    
    \begin{subfigure}{0.585\linewidth}
    \vspace{0.75em}
    \centering
\includegraphics[width=\textwidth]{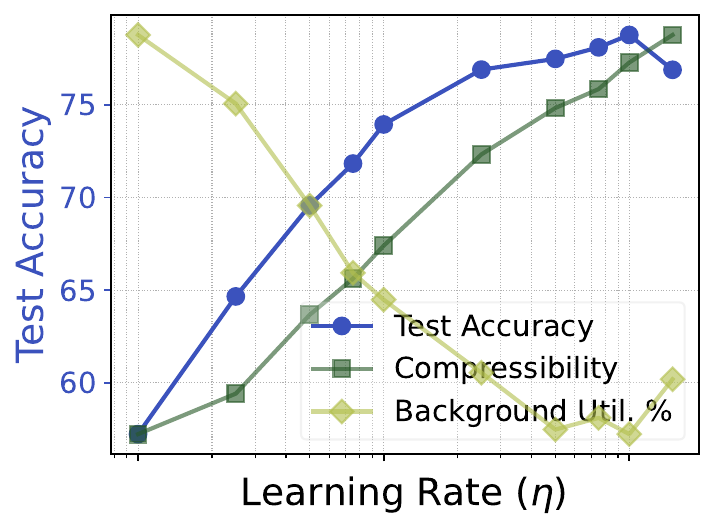}
\label{fig:teaser_cifar10_stats}
  \end{subfigure}
  \hfil
  \hspace{0.03em}
    \begin{subfigure}{0.397\linewidth}
    \centering
\includegraphics[width=\linewidth]{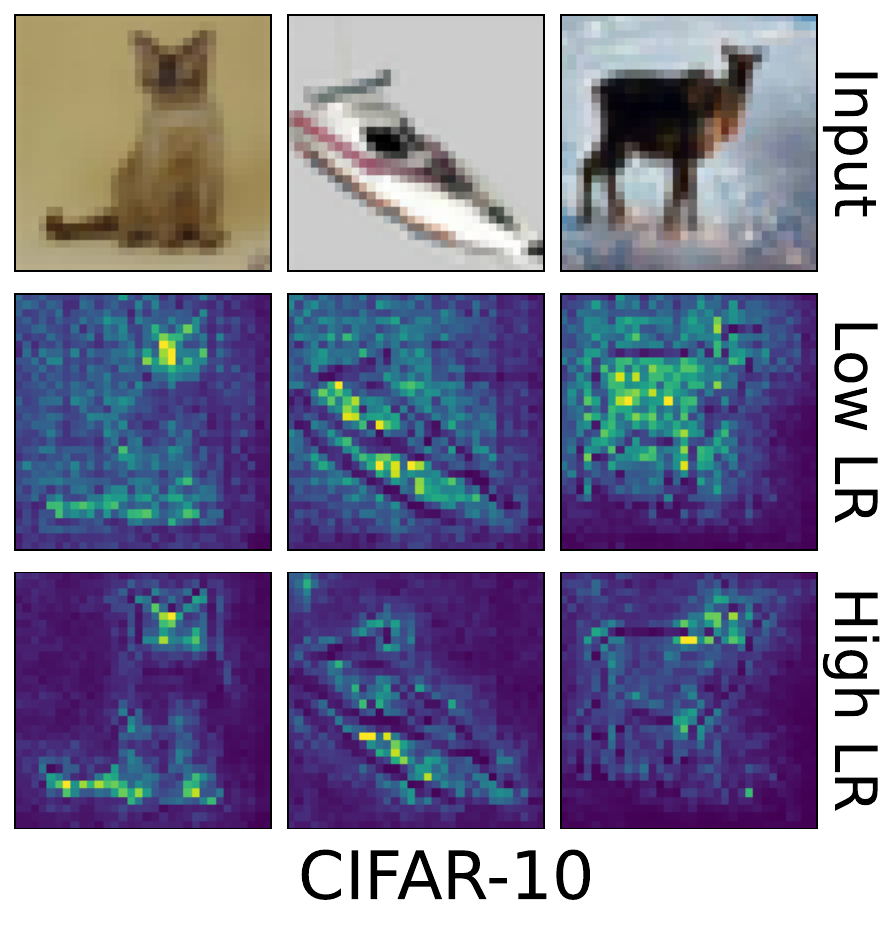}    
    \vspace{0.35em}        
\label{fig:teaser_cifar10_attr}
  \end{subfigure}
\vspace{-10mm}
    \caption{(\textbf{Top left}) When trained on the Double MNIST spurious correlation (SC) dataset with higher learning rates (LR), models tend to be more robust against SCs (higher accuracy), more compressible, and have more favorable core/invariant feature utilization and class separation. (\textbf{Top right}) Given images containing core and spurious features (bright vs. dim digit), high LR models are more likely to be attuned to the core feature vs. spurious feature, unlike low LR models. (\textbf{Bottom})
    Our findings extend to standard classification tasks since vulnerability to SCs, such as background, strongly associate with accuracy drops. All ranges normalized in $y$-axis to highlight relationship to LR. \vspace{-2em}} 
    \label{fig:teaser}
\end{figure}

The requirement to function well in novel circumstances and being resource-efficient are two central challenges for modern machine learning (ML) systems, which are expected to perform critical functions with less resources and on smaller hardware, in the face of limited access to computational resources, environmental concerns about energy consumption, and tightening bottlenecks around computational hardware \citep{constineAICompute2023, colemanAIClimate2023,fungBigBottleneck2023}. 
As ML-based technologies increasingly permeate every aspect of daily and industrial life \citep{haanMostCommon2024}, it becomes urgent to determine the inductive biases that help the models address both challenges together. 


Although no learner can be expected to perform well under arbitrary changes in the environment \citep{wolpertNoFree1997}, they are expected to do so under ``reasonable'' distribution shifts -- a capability commonly exhibited by numerous animal species. Such robustness to distributions that differ from the training data in principled ways have been studied under the umbrella term out-of-distribution (OOD) generalization \cite{wilesFineGrainedAnalysis2021}. Although scaling models and datasets, as well as targeting specific aspects of this issue have shown some success \cite{liuJustTrain2021,bubeckUniversalLaw2021}, OOD generalization remains an open problem both theoretically and in practice \cite{dherinWhyNeural2022}.

A critical obstacle for OOD generalization is the presence of spurious correlations\footnote{Unless noted otherwise, throughout the paper any references to robustness or OOD generalization will be in the context of spurious correlations.} (SCs) in training data, which can intuitively be described as the relationships between the input features and output label in the training set that do not transfer to the test set. A canonical example is highlighted by \cite{beeryRecognitionTerra2018}, where models are likely to misclassify a cow on sand as a camel, due to learning the misleading statistical association between camels and desert backgrounds in the training data. \cite{sagawaInvestigationWhy2020} show empirically and theoretically how overparameterization can cause pathological learning dynamics, amplifying the emphasis on SCs and degrading OOD performance. Similarly, \cite{lacerdaEmpiricalAnalysis2023,duShortcutLearning2023a} show that AI systems and large language models (LLMs) can be susceptible to large performance drops if  training data contains easily exploitable patterns not carrying over to new environments.

The challenge of OOD generalization becomes even more pronounced when coupled with the need for model compressibility, given the increased demands of resource-efficiency on modern ML systems. It remains unclear whether these two objectives --robustness and efficiency-- are inherently in conflict or can actually complement each other \cite{duRobustnessChallenges2023, diffenderferWinningHand2021}. 
We posit that understanding this interplay is critical for the design of next generation ML models. 


To date, explaining the counterintuitive observation that overparameterization often improves rather than harms generalization, especially under SCs, remains to be illuminated, as does understanding when and why models become compressible~\cite{dinhSharpMinima2017, aroraStrongerGeneralization2018a, nakkiranDeepDouble2020, barsbeyHeavyTails2021, birdalIntrinsicDimension2021, rajAlgorithmicStability2023}. One key factor in this puzzle is the learning rate (LR), or step size, used during gradient-based training. Various studies have highlighted the impact of large LRs on generalization \cite{liExplainingRegularization2019, lewkowyczLargeLearning2020a,mohtashamiSpecialProperties2023}, model compressibility \cite{barsbeyHeavyTails2021}, and representation sparsity \cite{andriushchenkoSGDLarge2023}.  \changedc{Recent research confirms that LR remains a crucial hyperparameter in the training of multimodal foundational models \cite{udandaraoPractitionersGuide2024}}.

In light of these, in this paper we hypothesize that \emph{in deep neural networks, learning rate plays a pivotal role in achieving robustness to SCs without sacrificing efficiency}.
We then confirm this hypothesis via extensive analyses, making the following key contributions:
\begin{enumerate}[noitemsep,leftmargin=*,topsep=0mm] 
\item \changedc{We establish that \textbf{large LRs simultaneously and consistently promote both compressibility and robustness} specifically to SCs across a wide range of architectures, datasets, and training schemes.}
\item We identify that these effects are accompanied by \textbf{improved core feature utilization, class separation}, and \textbf{compressibility} in the learned representations. 
\item We show that large LRs produce \textbf{a unique combination of the aforementioned desirable properties}, in comparison to other major hyperparameters and regularizers.
\item We provide strong evidence that the robustness against SCs that large LRs confer, \textbf{contribute} to their previously documented success in \textbf{standard generalization tasks}.
\item \changedc{Our investigation into the mechanisms reveal the importance of \textbf{confident mispredictions of bias-conflicting samples} under high LRs.}
\end{enumerate}
In \cref{fig:teaser} we share an overview of our results.

\section{Related Work}
\paragraph{Robustness to Spurious Correlations}
Overreliance on simple, easily exploitable features that have limited bearing on a test set have been pointed out by \cite{geirhosShortcutLearning2020} and \cite{shahPitfallsSimplicity2020}. \cite{sagawaInvestigationWhy2020} examine the role of overparametrization in producing models that rely on spurious features, and
\cite{nagarajanUnderstandingFailure2022} point out how two features of the data distribution (that they name geometrical and statistical skews) might lead to a max-margin classifier ending up utilizing spurious features. \cite{pezeshkiGradientStarvation2021} point out the importance of features learned in early training, where easy-to-learn, spurious features might not be replaced by better-generalizing features. Various methods have been previously proposed to alleviate this problem. 
Methods that assume access to spurious feature labels/annotations exploit this information in different ways to improve worst group or unbiased test set performance \cite{sagawaDistributionallyRobust2020a, idrissiSimpleData2022}. In the absence of group annotations, alternative methods rely on assumptions about the nature of the spurious features and the inductive biases of the learning algorithms \cite{liuJustTrain2021, puliDonBlame2023, tiwariOvercomingSimplicity2023}.

\paragraph{Inductive Bias of Large Learning Rates, Model Compressibility}
\cite{liExplainingRegularization2019} provide one of the earliest set of findings regarding the inductive bias of LRs in standard machine learning tasks, and examine how large vs. small LRs lead to qualitatively different features to be learned by the neural network.  \cite{jastrzebskiBreakEvenPoint2020} point out how LRs in early training prevents the iterates from being locked into narrow valleys in the loss landscape, where the curvature in certain directions are high, to the detriment of the conditioning of the gradient covariance matrix. \cite{ lewkowyczLargeLearning2020a, mohtashamiSpecialProperties2023} point out the importance of large LRs in early training
\cite{keskarLargeBatchTraining2017}.  
\cite{rosenfeldOutliersOpposing2024} demonstrates the crucial role of spurious / opposing signals in early training, and how progressive sharpening \cite{cohenGradientDescent2021, liAnalyzingSharpness2022} of the loss landscape in the directions that pertain to the representation of these features lead to the eventual down-weighting of such non-robust features. 
\cite{diffenderferWinningHand2021} find that lottery-ticket style pruning methods present the most advantageous trade-off between performance, robustness, and compressibility. 
To date, no study has systematically explored how LR influences robustness to SCs or how it interacts with compressibility. See our suppl. material for an extended literature review.
\vspace{-2mm}\section{Setup}\vspace{-1mm}
We consider a classification setting, where the task of a model (\emph{\eg,}, a neural network) $f$ is to predict the discrete label $y \in \mathcal{Y}$, given a $d$-dimensional input $\vc{x}\in\mathbb{R}^d$. In order to assess the quality of a neural network represented by its weights $\vc{w}$, we consider a loss function $\ell : \mathcal{Y} \times \mathcal{Y} \mapsto \R_{\geq0}$, such that $\ell({y}, f_{\vc{w}}(\vc{x}))$ measures the error incurred by predicting the label of $\vc{x}$ as $\argmax_{j} f_{\vc{w}}(\vc{x})[j]$, when the true label is $y$. We define a \emph{data distribution} $\mu_Z$ over $\Zcal$, and a \emph{training dataset} with $n$ elements, \emph{i.e.},, $S = \{\vc{z}_1,\dots,\vc{z}_n\}$, where each $\vc{z}_i := (\vc{x}_i, y_i) \stackrel{\mathclap{\mbox{\tiny{i.i.d.}}}}{\sim}  \mu_Z$. 
We then denote the \emph{population} and \emph{empirical risks} as $\Rcal(\vc{w})\coloneqq\E_{\vc{x},y}\left[\ell(y,f_{\vc{w}}(\vc{x}))\right]$ and $\widehat{\Rcal}(\vc{w})\coloneqq \frac{1}{n}\sum_{i=1}^n \ell(y_i,f_{\vc{w}}(\vc{x}_i))$. Unless otherwise noted, we utilize the (minibatch) stochastic gradient descent (SGD) algorithm for empirical risk minimization: $
\vc{w}^{t+1} = \vc{w}^t - \eta \nabla_{\vc{w}^t} \frac{1}{b}\sum_{i\in\Omega^t} \ell(y_i,f_{\vc{w}}(\vc{x}_i)),
$
where $\Omega^t$ is a randomly sampled fixed-size subset of the training set, $b := |\Omega^t|$  is batch size, and $\eta$ is learning rate (LR; aka step size).



\vspace{-1.5mm}
\subsection{Spurious Correlations}
OOD generalization describes the case whenever we have access to  $S = \{\vc{z}_1,\dots,\vc{z}_n\}$, $\vc{z}_i := (\vc{x}_i, y_i) \stackrel{\mathclap{\mbox{\tiny{i.i.d.}}}}{\sim}  \mu^{\mathrm{train}}_Z$, yet we are interested in evaluating risk under a different distribution $\Rcal(\vc{w})\coloneqq\E_{\vc{x},y \sim \mu^{\mathrm{test}}_Z}\left[\ell(y,f_{\vc{w}}(\vc{x}))\right]$, where $\mu^{\mathrm{train}}_Z \neq \mu^{\mathrm{test}}_Z$. This difference between training and test distributions is called \textit{distribution shift}. Some examples of distribution shift include \textit{subpopulation shift}, where $\mu^{\mathrm{train}}_Z(y) \neq \mu^{\mathrm{test}}_Z(y)$, and \textit{mechanism shift}, where $\mu^{\mathrm{train}}_Z(\vc{x}|y) \neq \mu^{\mathrm{test}}_Z(\vc{x}|y)$. A subtype of mechanism shift is of special importance in the literature and of particular importance for this paper: A model trained on $S$ is said to potentially be vulnerable to \textit{spurious correlation}s\footnote{As is common in the literature, here we use the term spurious \textit{correlations} to also include non-linear relationships induced by a confounder.} in the training dataset when it is possible to define a decomposition of $\vc{x} = (\vc{x}^c, \vc{x}^s)$, such that $\mu^{\mathrm{train}}_Z(\vc{x}^c|y) = \mu^{\mathrm{test}}_Z(\vc{x}^c|y)$, and $\mu^{\mathrm{train}}_Z(\vc{x}^s|y) \neq \mu^{\mathrm{test}}_Z(\vc{x}^s|y)$.
In such cases, $\vc{x}^c$ are frequently called \textit{core feature}s (aka invariant features), and $\vc{x}^s$ are frequently called \textit{spurious feature}s.

In previous research that addressed this problem, the assumed data distributions can vary considerably \cite{wilesFineGrainedAnalysis2021, namLearningFailure2020}. 
In this study we focus on one of the most canonical cases, previously called ``perception tasks'' \cite{puliDonBlame2023} or ``easy-to-learn'' tasks \cite{nagarajanUnderstandingFailure2022}, where the core features are perfectly informative with respect to the label in the training set, and spurious features imperfectly so. In such tasks, the former are construed to be more difficult to learn. We assume a separate generative model for spurious features. We define a bias label $b \in \mathcal{Y}$, as opposed to the class label $y \in \mathcal{Y}$ for core features, admitting the decomposition $\mu_Z(\vc{x}^c, \vc{x}^s, y, b) = \mu_Z(\vc{x}^c|y)\mu_Z(\vc{x}^s|b)\mu_Z(y, b)$, and while $\mu_Z^{\mathrm{test}}(y, b) = \mu_Z^{\mathrm{test}}(y) \mu_Z^{\mathrm{test}}(b)$, we have $\mu_Z^{\mathrm{train}}(y, b) \neq \mu_Z^{\mathrm{train}}(y) \mu_Z^{\mathrm{train}}(b)$. More specifically, we further assume $\mu_Z^{\mathrm{train}}(y=a, b=a) \gg \mu_Z^{\mathrm{train}}(y=a) \mu_Z^{\mathrm{train}}(b=a)$, where the mutual information between $y$ and $b$ in the training set presents a challenge for the learner. The value $\rho^{\mathrm{train}}:=1-\mu^{\mathrm{train}}_Z(y=b)$ determines the rate of \textit{bias-conflicting} examples in the training dataset, which a learner can exploit to avoid utilizing spurious features.  Although there does not exist a canonical definition of easy vs. difficult-to-learn features in the literature \cite{qiu2024complexity}, they can be construed as the difficulty of estimating $\mu^{\mathrm{train}}_Z(y|\vc{x}^c)$ vs. $\mu^{\mathrm{train}}_Z(b|\vc{x}^s)$. 

\begin{figure}[t]
    \centering
    \begin{minipage}[b]{0.235\textwidth}
        \centering
        \includegraphics[width=\textwidth]{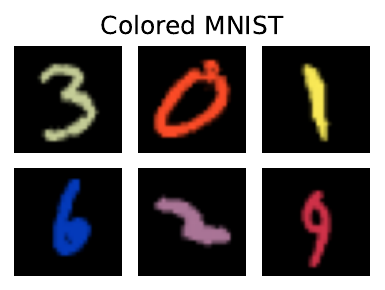}
        \label{fig:dataset_cmnist}
        \vspace{-1.5em}
    \end{minipage}
    \hfil
    \begin{minipage}[b]{0.235\textwidth}
        \centering
        \includegraphics[width=\textwidth]{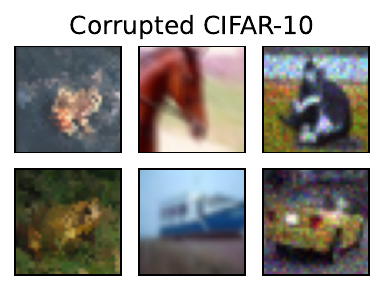}
        \label{fig:dataset_corcifar}
        \vspace{-1.5em}
    \end{minipage}
    \begin{minipage}[b]{0.235\textwidth}
        \centering
        \includegraphics[width=\textwidth]{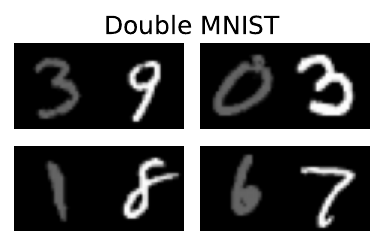}
        \label{fig:dataset_dmnist}
        \vspace{-1.5em}
    \end{minipage}
    \hfil
    \begin{minipage}[b]{0.235\textwidth}
        \centering
        \includegraphics[width=\textwidth]{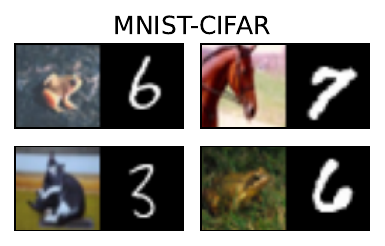}
        \label{fig:dataset_mnistcifar}
        \vspace{-1.5em}
    \end{minipage}
    \vspace{-4mm}
    \caption{Example images from semi-synthetic SC datasets.\vspace{-6mm}} 
    \label{fig:datasets}
\end{figure}
\subsection{Compressibility}
A frequently used metric for characterizing compressibility, especially for activations, is \textit{sparsity}. Given a $d$-dimensional vector $\zrv\in\R^d$ it can be defined as $\mathrm{Sparsity}(\zrv) = 1 - \|\zrv\|_0/d.$ Although this is an intuitive and useful metric, more nuanced notions of compressibility are desirable, e.g. when considering activation functions like sigma that does not output $0$ values. This is because a (deterministic or probabilistic) vector might be ``summarizable'' with a small subset of its elements regardless of the number of entries in the vector that \textit{exactly} equal $0$. Therefore, we use $(q,\kappa)$-Compressibility inspired by \cite{gribonval2012compressible}: $(q,\kappa)\mathrm{-Compressibility}(\zrv) = 1-\frac{\inf_{\|\mathbf{y}\|_0 \leq \lceil\kappa d\rceil}\|\zrv - \mathbf{y}\|_q}{\|\zrv\|_q}.$
      Intuitively, if a vector's $(2,0.1)$-Compressibility is high, this means that this vector can be approximated with little error (in an $\ell_2$ sense) by using $0.1$ of its $d$ elements.
\begin{figure*}[t]
    \centering
    \begin{subfigure}{0.20\linewidth}
      \centering
\includegraphics[width=\linewidth]{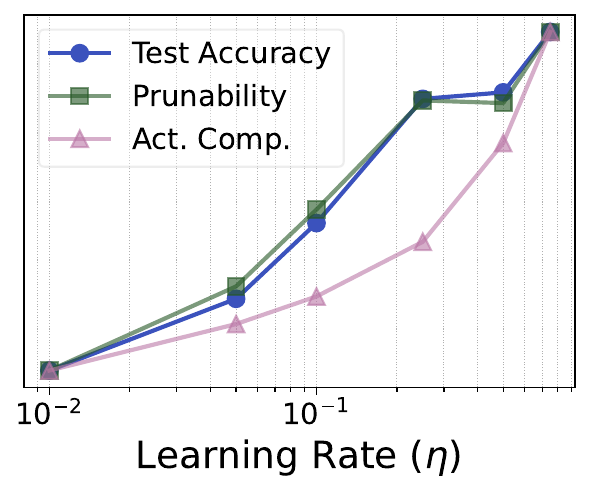}
    \end{subfigure}
    \centering
    \begin{subfigure}{0.20\linewidth}
      \centering
\includegraphics[width=\linewidth]{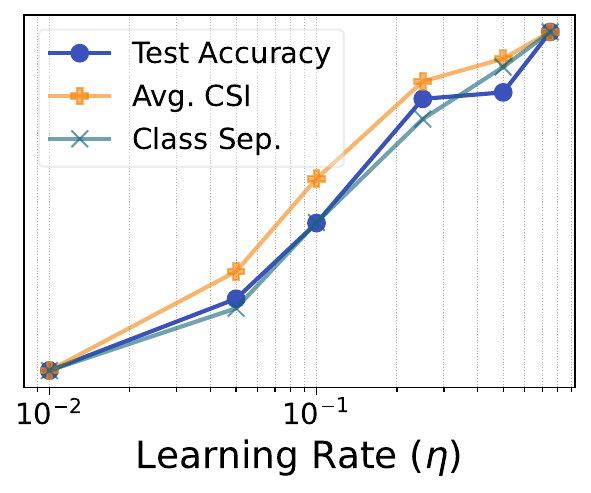}
    \end{subfigure}
    \begin{subfigure}{0.58\linewidth}
          \centering
    \includegraphics[width=\linewidth]{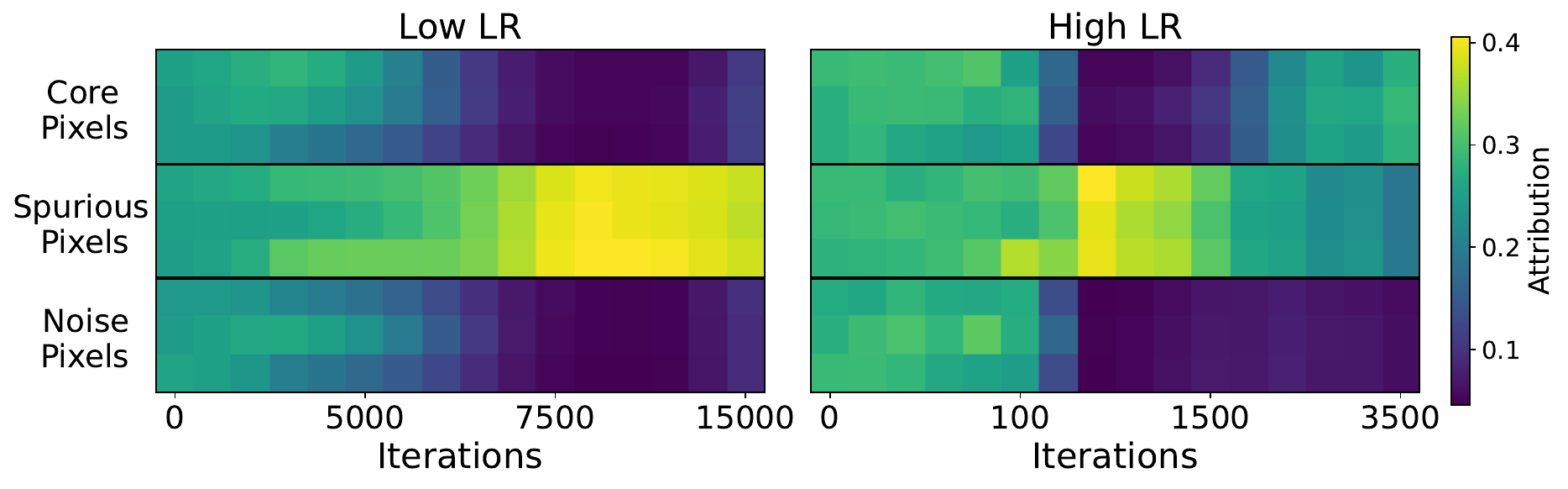}
    \end{subfigure}
    \vspace{-3mm}
    \caption{(Left) Effects of learning rate on OOD performance (unbiased test acc.), network prunability, and representation properties with the parity dataset. See suppl. material for min. and max. values. (Right) Prediction attributions (i.e. feature importance) for core, spurious, and noise pixels throughout training for low and high LR models.\vspace{-5mm}}
    \label{fig:synth-data}
\end{figure*}
 
 While sparsity or $(q,\kappa)$-Compressibility have merit as compressibility metrics, given our interest in compressibility \textit{in relation to} OOD generalization performance, quantifying network/parameter compressibility while taking the downstream effects of compression on performance into account is crucial. For this purpose we define $\kappa$-Prunability; given a predictor $f$ and its (unbiased) test accuracy $\mathrm{Acc}_{\mu^{\mathrm{test}}}(f)$ it is defined as $
    \kappa\mathrm{P}(f) = {\mathrm{Acc}_{\mu^{\mathrm{test}}}(f^{(\kappa)})}\,/\,{\mathrm{Acc}_{\mu^{\mathrm{test}}}(f)}$,
 where $f^{(\kappa)}$ corresponds to $f$ with $\kappa \in [0,1]$ of its parameters pruned (set to $0$). 
 \changedc{This measures the retention of (unbiased) test accuracy after pruning a fraction $\kappa$ of parameters. Here, we utilize structured (neuron/filter) pruning given its computationally desirable properties \cite{liPruningFilters2017b}. 
 Instead of committing to a particular $\kappa$, we will use $\sum_{\kappa \in \mathbf{k}}\kappa\mathrm{P}(f)$, $\mathbf{k} = (0.1, 0.2, \dots, 0.9)$. }
 Given our chosen metric, we will use the terms network prunability and network compressibility interchangeably throughout the rest of the text. \changedc{See suppl. material for qualitatively identical results with alternative metrics.}

\subsection{Metrics for Representation Analysis}
We now define metrics we use when analyzing learned representations beyond compressibility. Unless otherwise noted, all analyses on representations target post-activation values in the penultimate neural network layer, often referred to as \textit{learned representations} in the literature.

\paragraph{Class Separation}
We use different metrics to quantify representations' class-sensitivity. The first metric we utilize is class separation $R^2$, where we adopt the definition of
\cite{kornblithWhyBetter2021}, who investigate this notion in relation to model generalization performance as well as representations' transferability under different training losses:  
$R^2 = 1 - {\bar{d}_{\mathrm{within}}}/{\bar{d}_{\mathrm{total}}}$, 
where $\bar{d}_{\mathrm{within}} = \sum_{k=1}^{K} \sum_{m=1}^{N_k} \sum_{n=1}^{N_k} \frac{1 - \mathrm{sim}(x_{k,m}, x_{k,n})}{K N_k^2}$ and $
\bar{d}_{\mathrm{total}} = \sum_{k=1}^{K} \sum_{j=1}^{K} \sum_{m=1}^{N_j} \sum_{n=1}^{N_k} \frac{1 - \mathrm{sim}(x_{j,m}, x_{k,n})}{K^2 N_j N_k}$,
with $K$ denoting number of classes, $N_j$ number of samples with label $j$, and $\mathrm{sim}$ a similarity metric such as cosine similarity.

\paragraph{Core vs. Spurious Feature Utilization}
Though important for a classification task, class separation does not quantify sensitivity of specific neurons to classes. For this, we use \emph{class-selectivity index} ($\CSI$) \cite{leavittSelectivityConsidered2021, ranadiveSpecialRole2023a}: $\CSI = \rho_{\mathrm{max}}\frac{\pi_\mathrm{max} - \pi_\mathrm{-max}}{\pi_\mathrm{max} + \pi_\mathrm{-max} + \epsilon}.$
As in prior works, we take $\pi_\mathrm{max}$ to be largest class-conditional activation mean for a given neuron, where the means are computed over a sample of inputs, and $\pi_\mathrm{-max}$ is the mean of the remaining classes. In contrast to previous use of this metric, given the prevalence of activation sparsity in our experiments, we scale this fraction with $\rho_{\mathrm{max}}$, which corresponds to the ratio of instances belonging to the said class for which the neuron is activated. This factors in the ``coverage'' of a particular neuron for the instances belonging to a class, and prevents computing high class-selectivity for a neuron that fires rarely for the members of the said class. When computed over an unbiased test set, this metric serves to characterize a neuron's \textit{sensitivity to core features}. This is especially useful when the core features and spurious features overlap and attribution methods cannot be reliably used to investigate a specific neuron's responsiveness to core features. We average over neurons' class sensitivities (in learned representations) to characterize a network's core feature utilization.
Note that if accessible, spurious feature labels can be used to compute the counterpart bias-selectivity index (BSI).


\paragraph{Input Image Attributions}
To quantize the influence of particular input pixels or features over models' predictions, we utilize the commonly used interpretability method Integrated Gradients (IG) \cite{sundararajanAxiomaticAttribution2017b} to compute pixel-wise attributions over 10 random seeds. See suppl. material for details of IG, identical results under other interpretability methods, and \changedc{how it provides convergent results with the aforementioned CSI regarding core/spurious feature utilization}.

\vspace{-3mm}
\section{Datasets, Models, and Training Procedure}
\label{sec:datasets-models}
\vspace{-1mm}
\subsection{Datasets}
\vspace{-1mm}
We investigate the effect of LR on the model behavior on four classes of datasets: 1- Synthetic SC, 2- Semi-synthetic SC, 3- Naturalistic SC, and 4- Naturalistic classification. We describe each class of datasets below. As described in our setup, all SC datasets will involve a simple (spurious) feature that is predictive of the true label in the training set but not in the test set, and a more complex (core) feature that is predictive of the label in both training and test. 

\begin{figure*}[t]
  \centering
\begin{subfigure}{\textwidth}
  \centering
  \includegraphics[width=\textwidth]{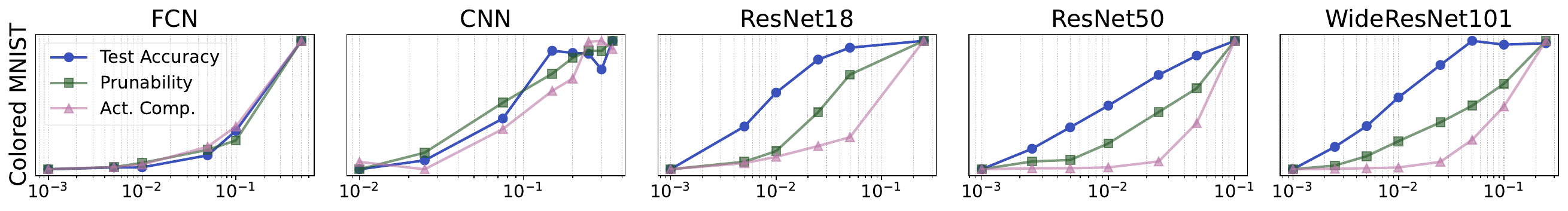}
\vspace{-6mm}
\end{subfigure}
\begin{subfigure}{\textwidth}
  \centering
\includegraphics[width=\textwidth]{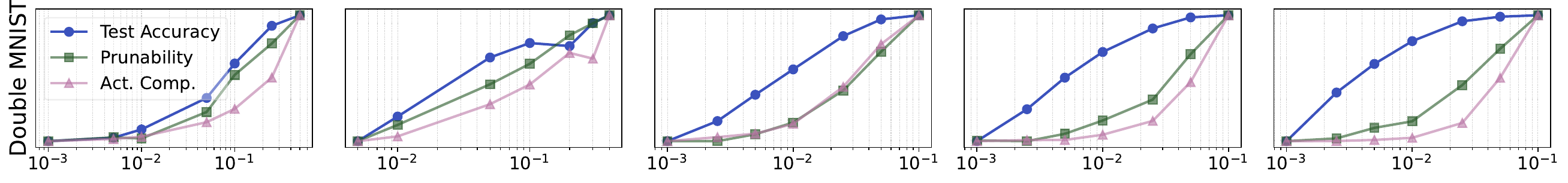}
\vspace{-6mm}
\end{subfigure}
\begin{subfigure}{\textwidth}
  \centering 
  \includegraphics[width=\textwidth]{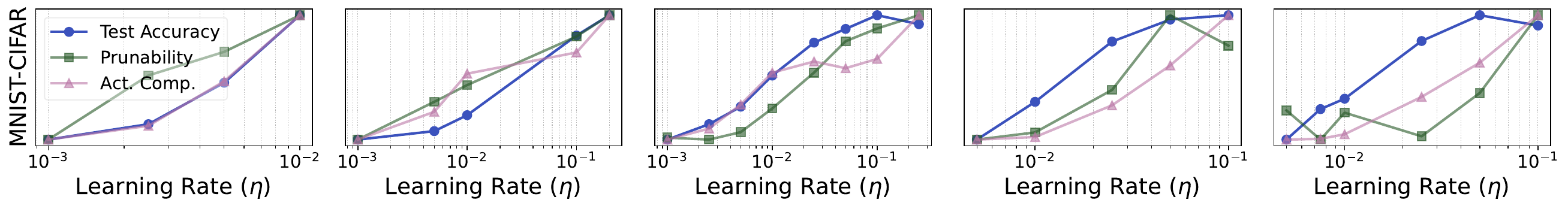}
\end{subfigure}
\vspace{-6mm}
\caption{Effects of learning rate on OOD performance (unbiased test acc.), network prunability, and representation (activation) compressibility in semi-synthetic SC data. $y$-axes are normalized within each figure for each variable, see suppl. material for min. and max. values.\vspace{-4mm}}
  \label{fig:main_results_comp}
\end{figure*}
\paragraph{Synthetic SC data}
We utilize two synthetic datasets from the literature to investigate the effects of LR on robustness to SCs and compressibility. These are the \textit{parity dataset} proposed in \cite{qiu2024complexity} and the \textit{moon-star dataset} proposed in \cite{geirhosShortcutLearning2020}. The advantage of utilizing synthetic datasets is the clear definition of \emph{simple} vs. \emph{complex} features they enable, as well as total control they afford over bias-conflicting sample ratio ($\rho$). In the parity dataset, the core and spurious features are binary vectors of size $C$ and $S$, with their parity bits (i.e. whether the vectors include odd number of 1's) corresponding to the true label $y$ vs. spurious label $b$. Setting $C > S$ leads the spurious feature to be simpler than the complex feature. The moon-star dataset is a binary classification task where the classifier is expected to distinguish moon shaped objects from star shaped objects, with the spurious feature being the quadrant of the image on which the object is located. For both datasets we set bias-conflicting sample ratios as $\rho^{\tr}=0.1$ and $\rho^{\te}=0.5$. 

\paragraph{Semi-synthetic SC data} These are arguably the most commonly used dataset types in research on SCs \cite{ sagawaInvestigationWhy2020,namLearningFailure2020, kirichenkoLastLayer2022, puliDonBlame2023,qiu2024complexity}. The reason for their popularity is the attractive combination they provide in the form of having relatively realistic inputs with a decent control over dataset properties such as bias-conflicting sample ratios and the complexity of the core vs. spurious features. We present the bulk of our results on four such datasets: Colored MNIST, Corrupted CIFAR-10, MNIST-CIFAR, and Double MNIST (see \cref{fig:datasets} for examples from each). 
The former two were proposed by \cite{namLearningFailure2020}, and have been frequently used in the literature. MNIST-CIFAR dataset is the extension of the domino dataset proposed by \cite{shahPitfallsSimplicity2020} to all 10 classes of MNIST and CIFAR-10,  while Double MNIST has been designed by the authors for this paper. Each dataset includes some combination and/or modification of the well-known image datasets MNIST \cite{lecunnMNIST1998} and CIFAR-10 \cite{krizhevskyImageNetClassification2017}. 
The core and spurious features for these datasets can be specified as digit shape vs. digit color for  Colored MNIST, object vs. corruption type for Corrupted CIFAR-10, left digit vs. (brighter) right digit for Double MNIST, and CIFAR-10 vs. MNIST targets for MNIST-CIFAR.  We set $\rho^{\tr}=0.025$ for Colored MNIST and Double MNIST datasets, and $\rho^{\tr}=0.1$ for the remaining two, given the higher baseline difficulty thereof.

\begin{figure}[t]
  \centering
\begin{subfigure}{\columnwidth}
  \centering
  \includegraphics[width=0.99\columnwidth]{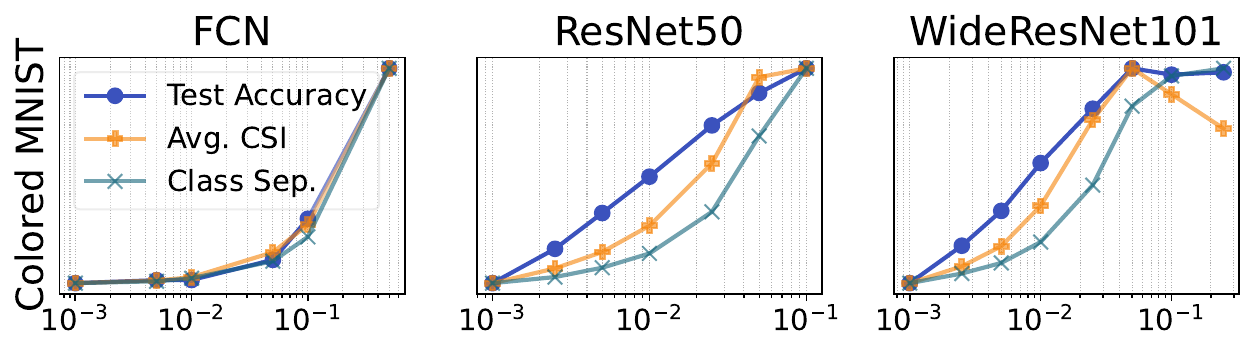}
\vspace{-1mm}
\end{subfigure}

\begin{subfigure}{\columnwidth}
  \centering
\includegraphics[width=\columnwidth]{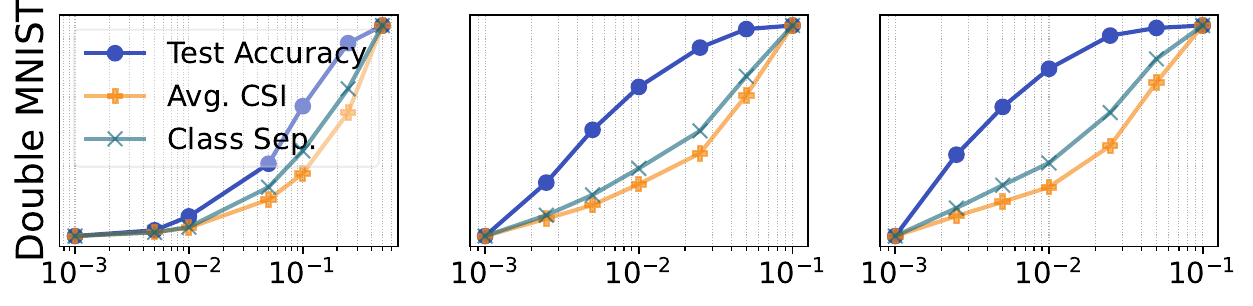}
\vspace{-6mm}
\end{subfigure}

\begin{subfigure}{\columnwidth}
  \centering
  \includegraphics[width=\columnwidth]{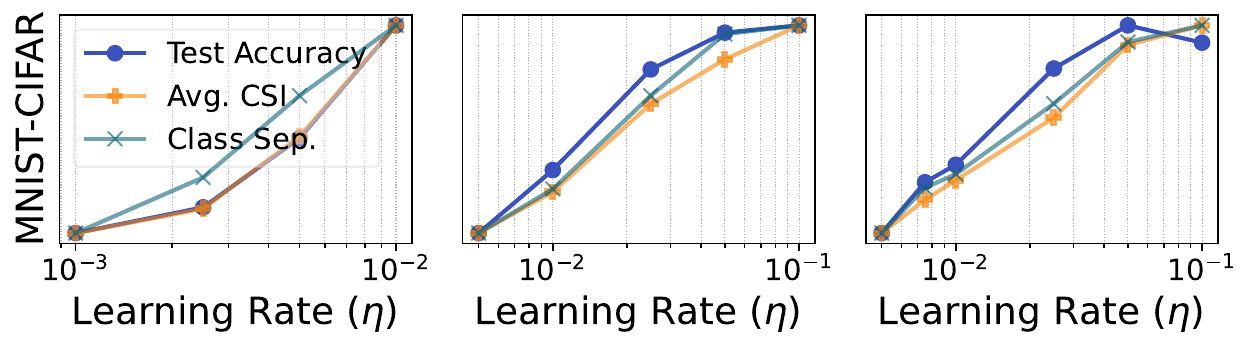}
\end{subfigure}
\vspace{-6mm}
\caption{Effects of learning rate  on OOD performance (unbiased test acc.) and representation properties in semi-synthetic SC datasets. $y$-axes are normalized within each figure for each variable, see suppl. material for min. and max. values. \vspace{-4mm}}
  \label{fig:main_results_repr}
\end{figure}
\begin{figure*}[t]
\begin{subfigure}{\textwidth}
  \centering 
\includegraphics[width=\textwidth]{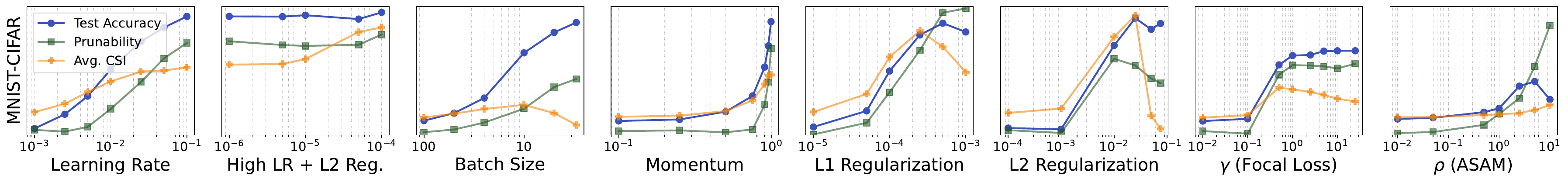}
\end{subfigure}
\vspace{-6mm}
\caption{Comparing hyperparameters, regularization methods, and losses in terms of OOD robustness, compressibility, and core feature utilization. $y$-axes are comparable within variables across rows, see suppl. material for min. and max. values.\vspace{-4mm} } 
  \label{fig:main_results_reg}
\end{figure*}
\paragraph{Naturalistic SC data} We also investigate the effect of LR on two naturalistic image classification datasets CelebA \cite{liuDeepLearning2015}  and Waterbirds \cite{wahWaterbirds2011}, both of which have been frequently used in research on SCs \cite{sagawaDistributionallyRobust2020a}. In the CelebA dataset, spurious and core features are the hair color and gender of the person in the images respectively. In the Waterbirds dataset, these correspond to the background of the pictured bird and their natural habitat (water vs. land). 
\changedc{In keeping with literature \cite{yeSpuriousCorrelations2024}, we examine worst-group test accuracy in our experiments with CelebA and Waterbirds \cite{sagawaDistributionallyRobust2020a}, while for the rest we use unbiased test set accuracy \cite{namLearningFailure2020}.}

 \begin{figure}[t]
    \centering
    \begin{subfigure}{\linewidth}
      \centering
      \includegraphics[width=\linewidth]{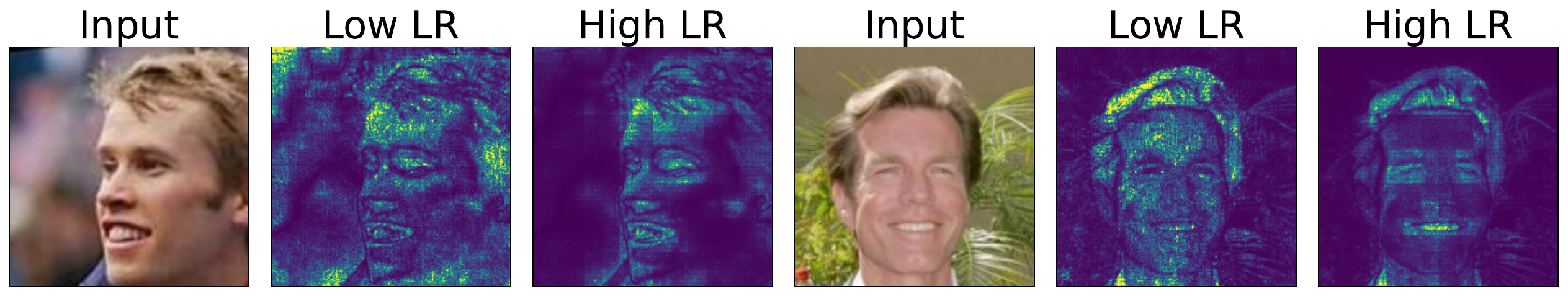}
    \end{subfigure}
    \begin{subfigure}{0.49\linewidth}\vspace{1mm}
      \centering
      \includegraphics[width=\linewidth]{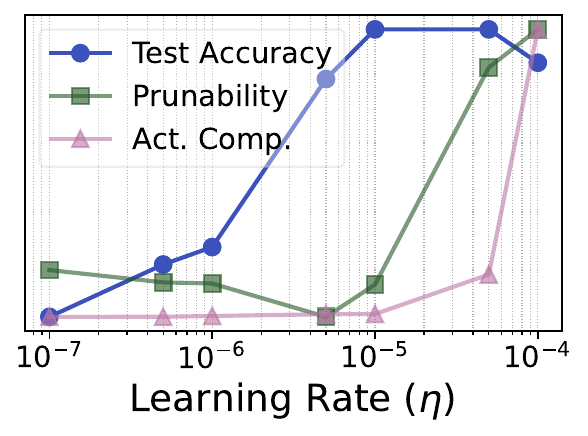}
    \end{subfigure}
    \hfill
    \begin{subfigure}{0.49\linewidth}
      \centering
    \includegraphics[width=\linewidth]{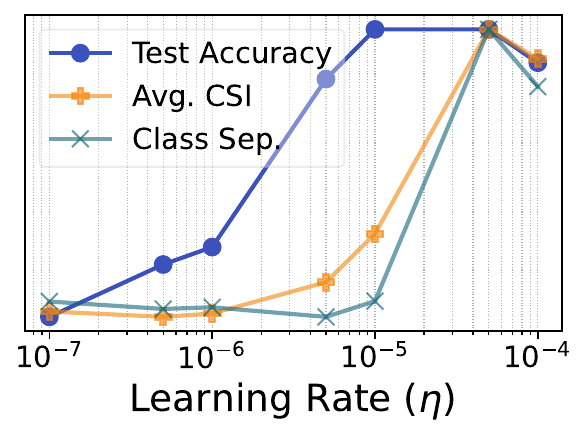}
    \end{subfigure}
    \vspace{-1mm}
    \caption{Effects of LR with a Swin Transformer and CelebA dataset on (top) model attributions and (bottom) network statistics. \vspace{-8mm}}
    \label{fig:transformer}
\end{figure}

\paragraph{Standard classification data} As we extend our inquiry from SC datasets to naturalistic classification tasks, we investigate the effects of LR on model behavior using CIFAR-10, CIFAR-100, and ImageNet-1k datasets \cite{krizhevskyImageNetClassification2017,dengImageNetLargescale2009}. The former two include $32 \times 32$ images of 10 and 100 classes of objects or animals, respectively. We use the traditional train-test splits in our experiments, consisting of 50000 and 10000 samples. ImageNet-1k consists of 1,331,167 color images scaled to $256 \times 256$ belonging to 1000 classes. We use the traditional train-validation split with 1,281,167 and 50000 samples each.

\vspace{-1mm}\subsection{Architectures, Training, \& Implementation}\vspace{-1mm}
\label{sec:arch_training_impl}
To investigate the diversity of contexts in which LR has the aforementioned effects, we utilize a variety of architectures.
These include fully connected networks (FCN) with ReLU activation, convolutional neural networks (CNN; a slightly simplified variant of VGG11 \cite{simonyanVeryDeep2015}), ResNet18 and ResNet50 \cite{heDeepResidual2016}, and Wide ResNet-101-2 \cite{zagoruykoWideResidual2017}. We further include a vision transformer model, namely the Swin Transformer by \cite{liuSwinTransformer2021}.
Unless otherwise stated, all models are trained with a constant LR until 100\% accuracy with no explicit L1/L2 regularization. Batch sizes are set to 16 for CelebA, Waterbirds, and ImageNet experiments, and 100 for the remaining experiments. An exception to the training scheme above is the training of the Swin Transformer on CelebA and Waterbirds datasets, where we utilize an AdamW optimizer combined with an ImageNet-1K pretraining initialization, to investigate the effects of LRs under more realistic training schemes. Our suppl. material provides  additional details for our experiments and implementation, \changedc{while showing that using different convergence criteria, optimizers (Adam  \cite{kingmaAdamMethod2015}, PSGD \cite{liPreconditionedStochastic2018}), or LR schedulers (Cosine annealing \cite{loshchilovSGDRStochastic2017}, WSD \cite{huMiniCPMUnveiling2024}) leads to qualitatively identical, and sometimes stronger results}. Unless otherwise noted, the results presented are the average of experiments run with at least 3 different random seeds. \changedc{Since high LRs can lead to divergence and/or drastic performance loss at the extremes, we utilize LRs up to the point of divergence or drastic performance loss ($>25\%$ in test accuracy) among any one of the seeds.} Lastly, since our results involve the examination of a number of variables simultaneously, in figures $y$-axes are normalized for each variable to highlight the effects of LR, and value ranges for all variables are separately provided in the supp. material for readability. \changedc{See our public repository\footnote{\url{https://github.com/mbarsbey/spurious_comp}} for additional implementation details.}

\vspace{-2mm}\section{Results}\vspace{-2mm}

\label{sec:results}



\paragraphQ{How does LRs affect robustness to SC and compressibility}
We first investigate the effects of LR on compressibility and robustness to SCs, using the parity dataset in \cref{fig:synth-data}, with a fully connected network (FCN) with 3 hidden layers of width 200. The results show a clear effect of LR on robustness to SC and compressibility: Robustness and compressibility increases as a function of LR. The same figure also shows how the attributions to core, spurious, and noise pixels/features evolve for low and high LR models. Moving on to semi-synthetic SC datasets. \cref{fig:main_results_comp} shows that across datasets and architectures, larger LRs lead to increased robustness to SCs, strongly supporting our central hypothesis.
This increase in robustness is also accompanied by increased network and activation compressibility. Recall that the notion of \kprune we employ here is computed under an unbiased test set, emphasizing the high LRs effect on achieving these aims jointly. See suppl. material for additional details and similar results with moon-star and Corrupted CIFAR-10 datasets.
   
\begin{figure*}[t]
    \centering
    \begin{subfigure}{0.32\textwidth}
  \centering
\includegraphics[width=\linewidth]{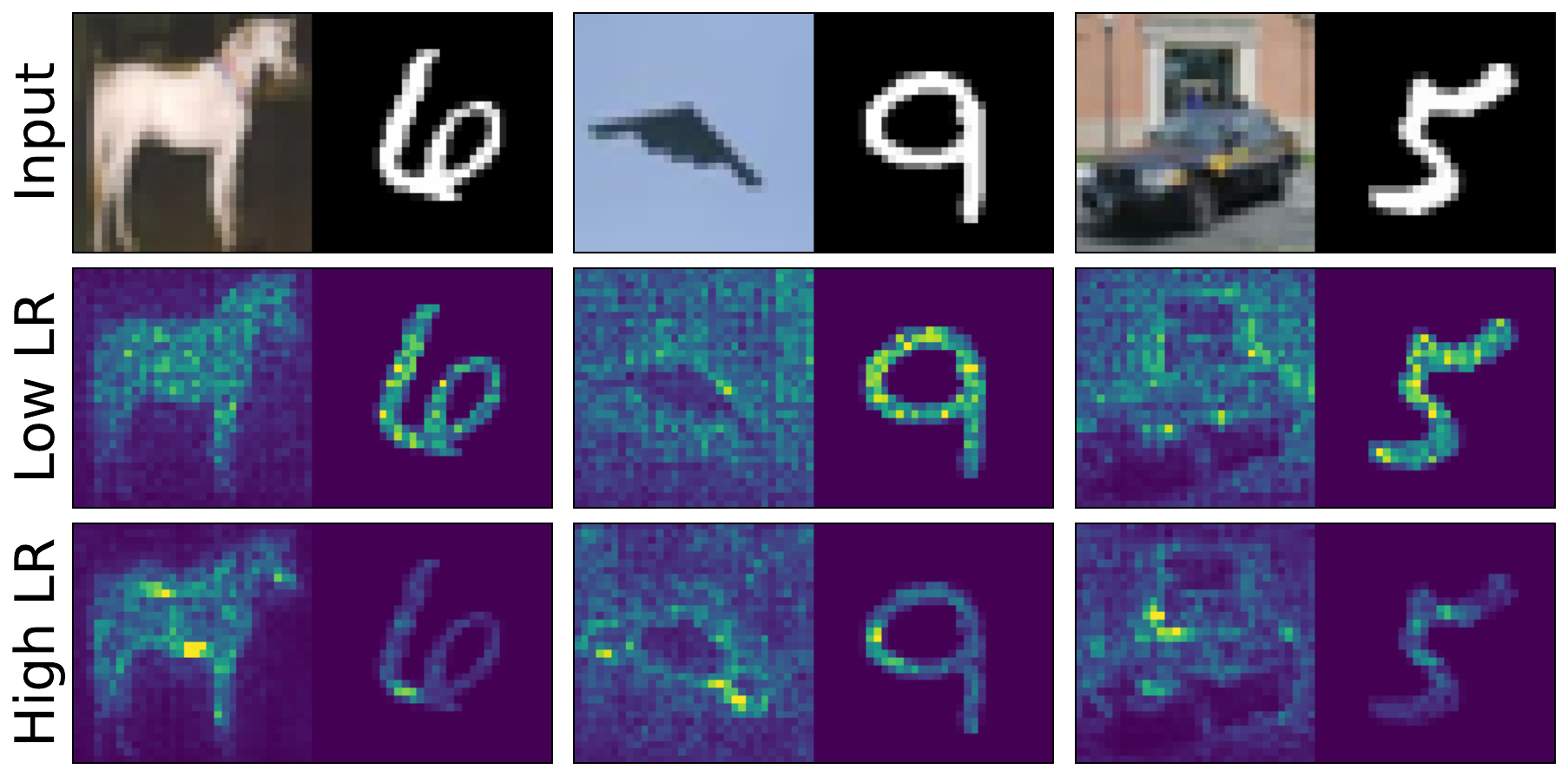}
\caption{MNIST-CIFAR}
\label{fig:attr_mnistcifar}
\end{subfigure}
        \hfil
    \begin{subfigure}{0.32\textwidth}
  \centering
\includegraphics[width=\linewidth]{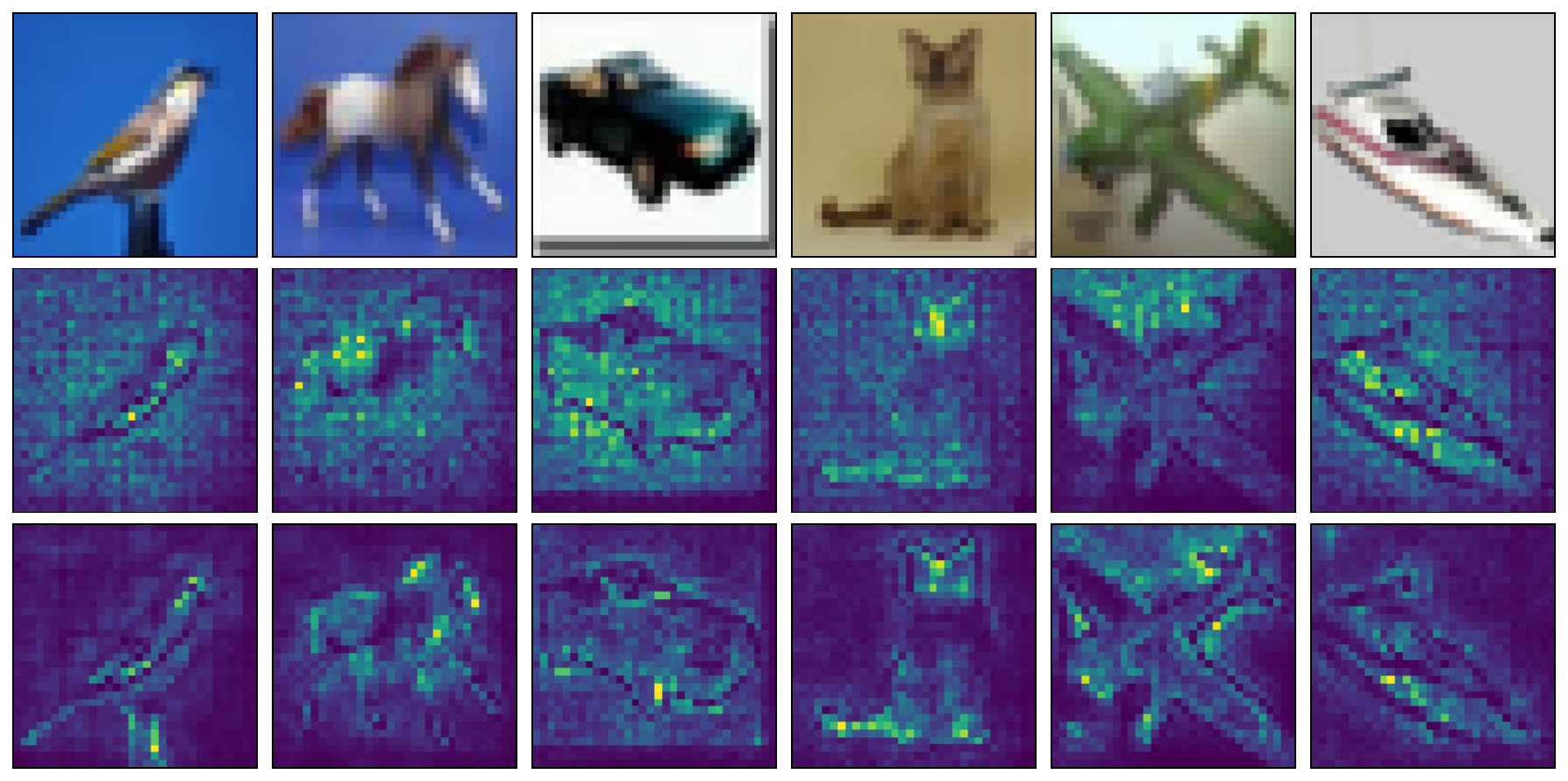}
\caption{CIFAR-10}
\label{fig:attr_cifar10}
\end{subfigure}   
        \hfil 
\begin{subfigure}{0.32\textwidth}
  \centering
\includegraphics[width=\linewidth]{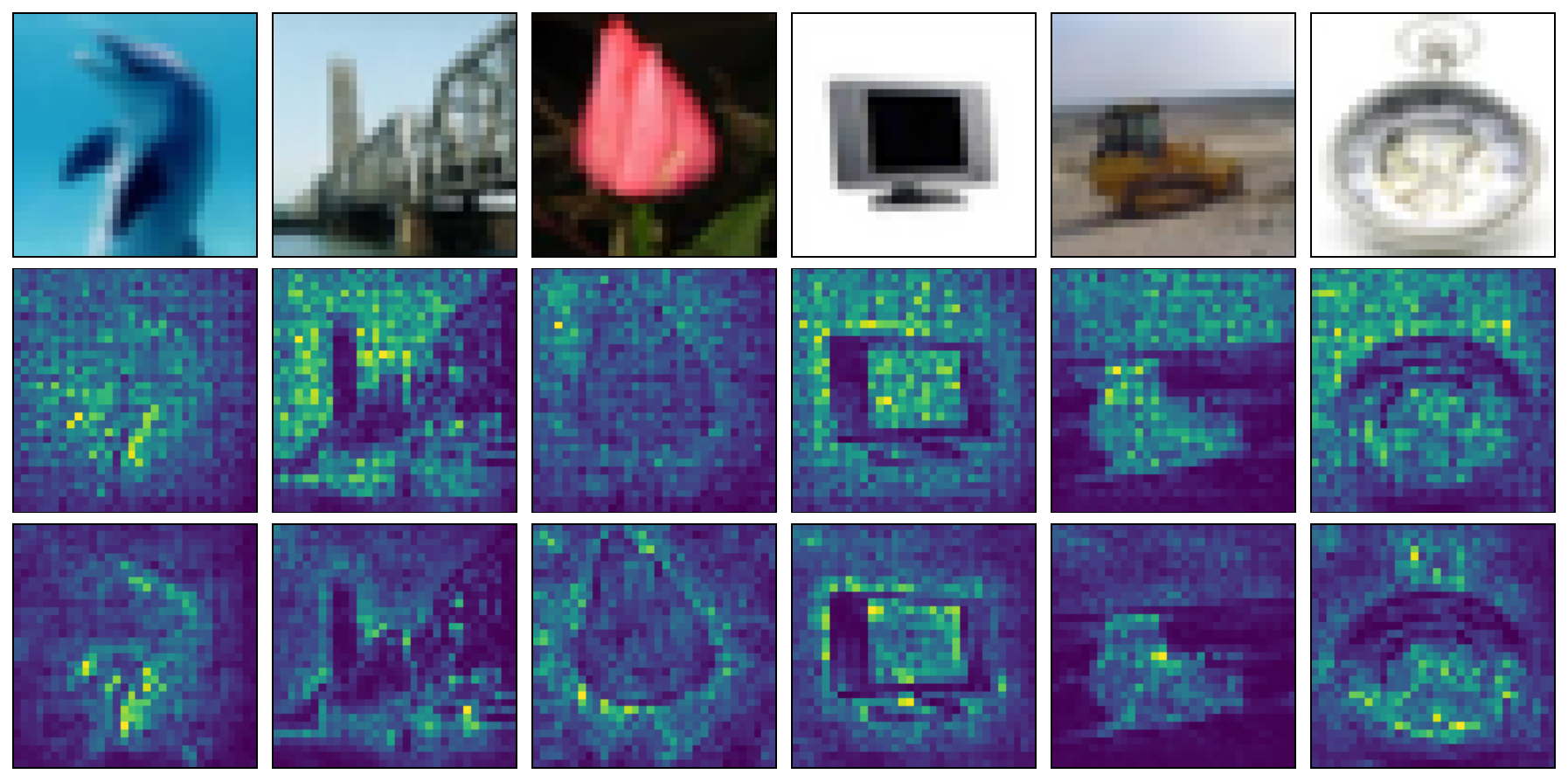}
\caption{CIFAR-100}
\label{fig:attr_cifar100}
\end{subfigure}
\vspace{-3mm}
    \caption{Attributions of a trained ResNet18 model on different datasets. For all datasets, top row includes original images; middle and bottom rows includes attributions of a model trained under low vs. high learning rate respectively. \vspace{-2mm}}  
    \label{fig:attr_demonst}
\end{figure*}

\paragraphQ{How does LR affect learned representations}
Our results regarding representation properties induced by high LRs are presented in \cref{fig:main_results_repr}. The results demonstrate that the increase in unbiased test performance and compressibility is accompanied by an increase in core feature utilization and class separation in the learned representations. Given recent results demonstrating learned representations' insensitivity to explicit SC interventions~\cite{kirichenkoLastLayer2022}, it is important to highlight that the positive effects of high LRs are reflected in the learned representations themselves.

\paragraphQ{Do the observed effects of LR generalize to larger models}
To make sure our results are not limited to small models and (semi-)synthetic data, we investigate the effect of LRs using a larger, transformer architecture and naturalistic data. As described in~\cref{sec:arch_training_impl}, we train a Swin Transformer on the CelebA dataset, and present its effects on worst-group performance, compressibility, and representation properties in \cref{fig:transformer}. The results show that high LRs co-achieve robustness and compressibility in this scenario as well, further supporting our motivating hypothesis. See suppl. material for similar results on Waterbirds dataset.

\begin{figure}[t]
\vspace{-4mm}
        \centering
        \includegraphics[width=\columnwidth]{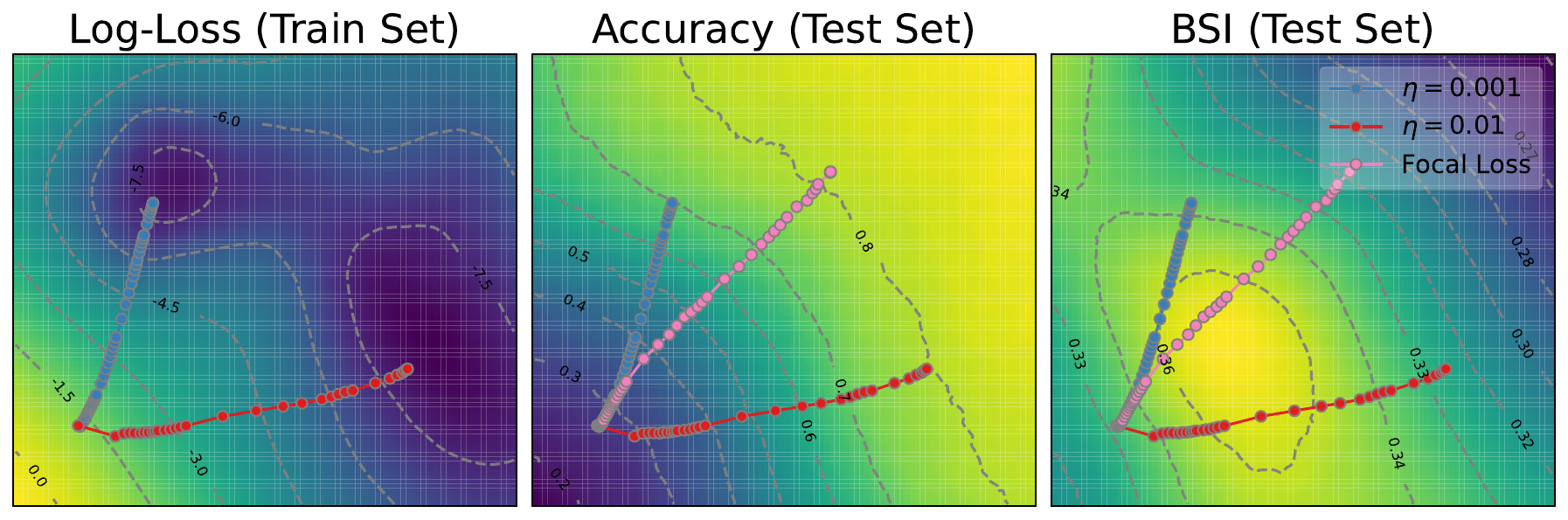}
        \label{fig:resnet18_train_test_landscape}
        \vspace{-6mm}
    \caption{Training loss, unbiased test acc., and spurious feature utilization landscapes for models with different LRs and FL.\vspace{-3mm}} 
    \label{fig:loss_landscape}
\end{figure}

\paragraphQ{How does LR behave when compared to and combined with other hyperparameters}
We then move on to investigate how large LRs compare to other fundamental hyperparameters in a machine learning pipeline.
The comparison set includes batch size, momentum, L1 and L2 regularization of the parameters, an alternative loss function designed for better exploitation of rare difficult examples (focal loss; FL \cite{linFocalLoss2018}), and an optimization procedure for finding well-generalizing (wide) minima in the training loss landscape (adaptive sharpness aware minimization; ASAM \cite{kwonASAMAdaptive2021, foretSharpnessAwareMinimization2021}). We repeat our experiments by varying the each hyperparameter in question (In the case of FL and ASAM these correspond to their central hyperparameters respectively, $\gamma$ and $\rho$). \cref{fig:main_results_reg} demonstrates our results using a ResNet18 model: High LRs consistently facilitate a combination of OOD generalization, network compressibility, and core feature utilization, either on par with or surpassing other interventions. We also observe that other interventions (e.g. L1/L2 regularization) can be combined with high LRs to achieve even better trade-offs. Importantly, in \cref{fig:loss_landscape} we highlight how different methods create robustness through different dynamics in the loss and feature utilization landscape, as discussed further in suppl. material, where we also share further details and results with other models and datasets, as well as more background on FL and SAM.


\begin{figure}[t]
\vspace{-1mm}
    \centering
        \begin{subfigure}{0.49\linewidth}
  \centering
    \includegraphics[width=\linewidth]{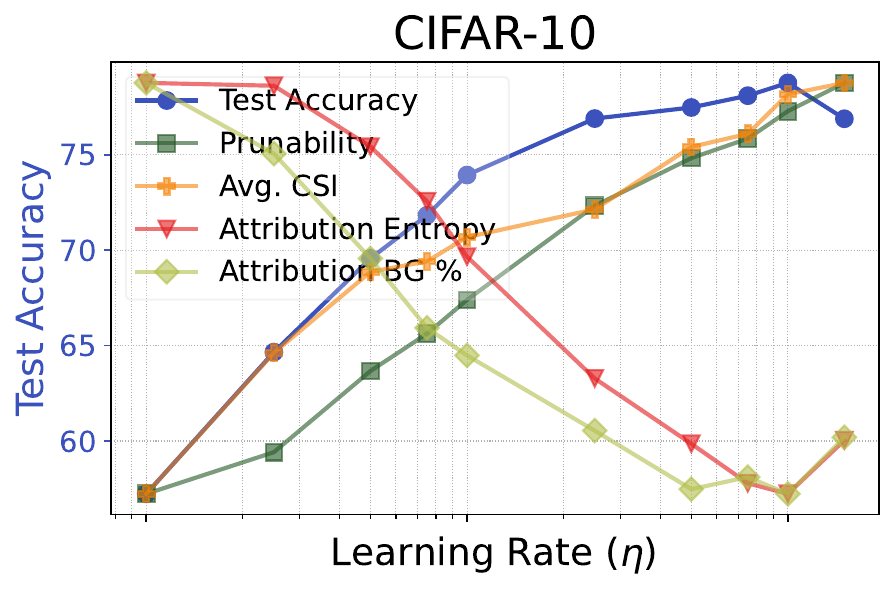}
\end{subfigure}
        \begin{subfigure}{0.49\linewidth}
  \centering
            \hfil
    \includegraphics[width=\linewidth]{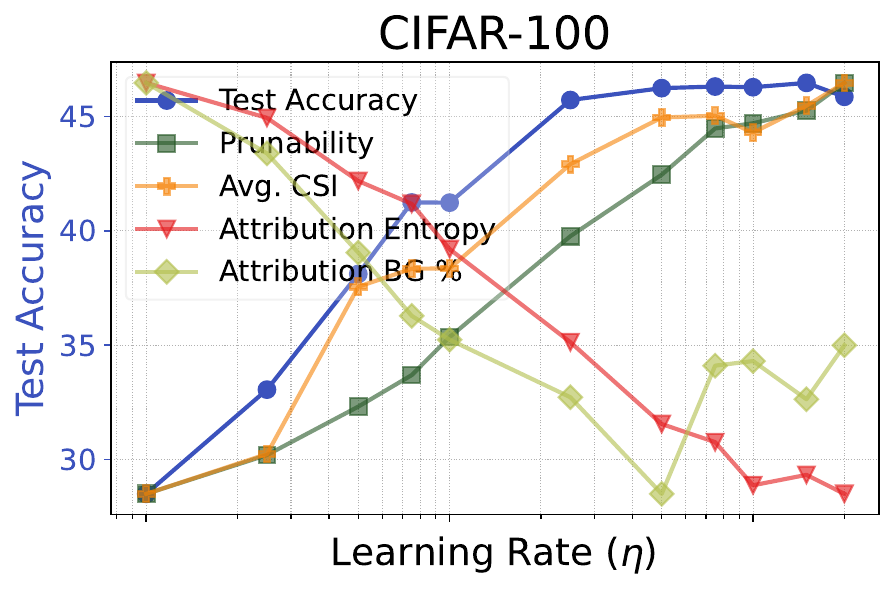}
\end{subfigure}
\vspace{-3mm}
    \caption{Higher LRs yield favorable generalization, compressibility, and core feature utilization for CIFAR datasets.\vspace{-4mm}}
    \label{fig:attr_quant}
\end{figure}

\paragraphQ{Does LR have similar effects on generalization in standard classification tasks}
We next investigate whether the robustness against SCs afforded by high LRs could help explain their success in generalization under standard classification tasks \cite{liExplainingRegularization2019, martinTraditionalHeavyTailed2019a, barsbeyHeavyTails2021}. A positive answer to this question would imply that in realistic machine learning settings, even a seemingly IID generalization task can involve OOD generalization subtasks due to implicit and/or rare SCs in the training distribution (see e.g. \cite{rosenfeldOutliersOpposing2024}).
We start our analyses with CIFAR-10 and CIFAR-100, and then move on to ImageNet-1k.

\noindent\textbf{To qualitatively investigate} whether the generalization advantage of large LRs in naturalistic classification tasks pertain to SCs, under a ResNet18 model, we examine input attributions for samples most likely to be predicted correctly by high LR models compared to low LRs (see suppl. material for details), presented in \cref{fig:attr_demonst}. Our results show that high LR models are likelier to focus on core vs. spurious features not only in the semi-synthetic MNIST-CIFAR dataset, but also in the CIFAR-10 and CIFAR-100 datasets as well. In both these datasets, low LR models focus much more on backgrounds and/or the color/texture of the foreground object, whereas the high LR models are likelier to focus on the object contours. The results strongly suggest that high LRs might be obtaining increased generalization through robustness to hidden SCs in the training data. We thus proceed to test this hypothesis quantitatively. 

\noindent\textbf{Quantitatively measuring } SCs in naturalistic datasets remains an open problem~\cite{yeSpuriousCorrelations2024}.
Nevertheless, here we develop two complementary metrics to assess the effects of LR on spurious feature utilization. The first, \textit{attribution entropy} relies on computing the entropy of a normalized input attribution map  since exploiting background information or object textures/colors would lead to more dispersed attribution maps as opposed focusing on contours of an object. We complement this relatively simple metric with a more targeted one in \textit{background attribution percentage}: Utilizing 
the DeepLabV3 segmentation algorithm to extract background maps \cite{chenRethinkingAtrous2017}, we compute the percentage of attributions that fall within vs. outside these image background segments. We present the results in~\cref{fig:attr_quant}. The results not only replicate the benevolent effects of large LRs seen in previous datasets, they also show that higher LR models are less likely to use background/texture information.

\noindent\textbf{ImageNet-1k.} We lastly examine the effects of large LRs on models trained on ImageNet-1k dataset. The results presented in \cref{fig:imagenet} clearly show that the previously observed effects of LR extend to larger datasets, furthering the implications of our findings to realistic scenarios.
\vspace{-0.35em}
\section{On the mechanism of large LRs}
\label{sec:analysis}
\begin{figure}[t]
    \centering
        \begin{subfigure}{0.48\linewidth}
  \centering
    \includegraphics[width=\linewidth]{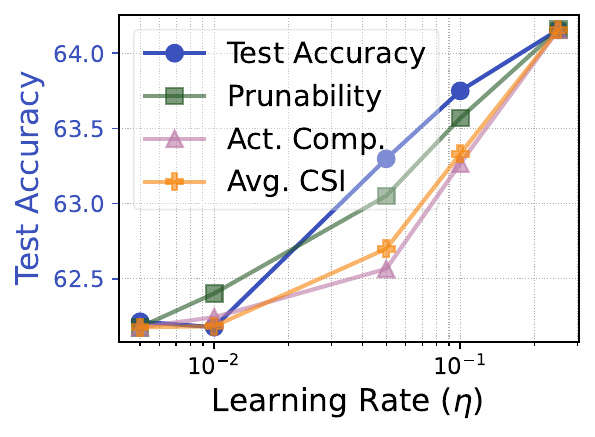}
    
\end{subfigure}
        \begin{subfigure}{0.50\linewidth}
  \centering
    \includegraphics[width=\linewidth]{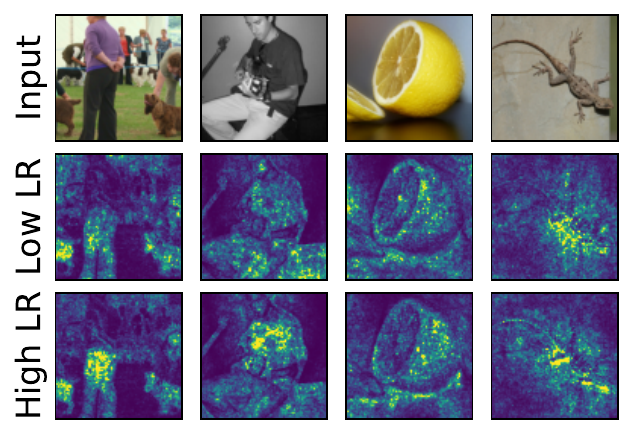}
\end{subfigure}
 \vspace{-2mm}
    \caption{ Effects of LR with ImageNet-1k dataset and ResNet18 model on model statistics and attributions.\vspace{-3mm}}
    \label{fig:imagenet}
\end{figure}
 \begin{figure}[t]
	\centering
    	\begin{subfigure}{0.325\linewidth}
		\centering
        \includegraphics[width=\linewidth]{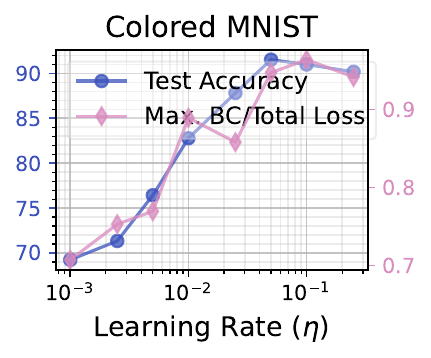}
	\end{subfigure}
    \hspace{-2mm}
    \begin{subfigure}{0.325\linewidth}
		\centering
        \includegraphics[width=\linewidth]{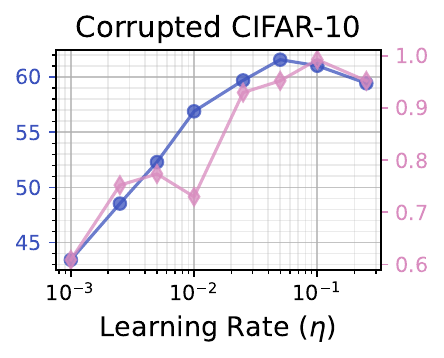}
	\end{subfigure}
    \hspace{-2mm}
    	\begin{subfigure}{0.325\linewidth}
		\centering
        \includegraphics[width=\linewidth]{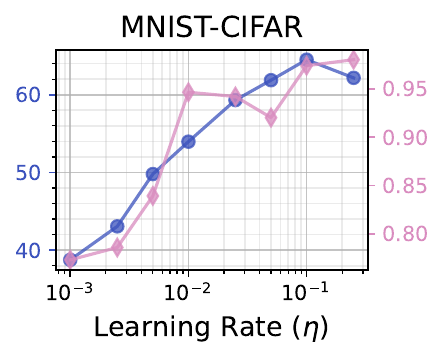}
	\end{subfigure}
	\vspace{-3mm}
	\caption{Unbiased test accuracy strongly correlates with maximum bias-conflicting loss ratio during training. \vspace{-5mm}}
	\label{fig:analysis_max_bcloss}
\end{figure}
While plausible mechanisms through which large LRs create compressibility have been proposed \cite{gurbuzbalabanHeavyTailPhenomenon2021b}, those through which large LRs can create robustness against SCs are much less explored, which we now investigate.
Previous work indicates that simple spurious features are learned earlier than more complex core features \cite{pezeshkiGradientStarvation2021, qiu2024complexity}, regardless of LR \cite[Theorem 4.1]{yangIdentifyingSpurious2024}: this produces a phase in learning where bias-aligned (BA) samples are predicted correctly, while bias-conflicting (BC) samples are misclassified. We find that at this stage, another effect of large LRs, namely fast \textit{norm growth} of model parameters and logits, become important\footnote{While norm growth under large LRs is well-established by previous research \cite{salimansWeightNormalization2016, merrillEffectsParameter2023, lyleNormalizationEffective2024}, its effects are underexplored in robust generalization.}. This is because with cross-entropy (CE) loss, up-scaled logits under high LRs lead to \textit{confident mispredictions of BC samples}. Due to the nonlinearty of CE, the mini-batch loss (and thus the gradient) is then increasingly dominated by mispredicted BC samples, corresponding to an implicit \textit{reweighting of the dataset in favor of BC samples}. We now formalize this statement. 


\vspace{-2mm}

\begin{prop}
	\label{thm:punish_spur}
    Let $f_{\w}$ be a predictor in a binary classification task, trained under the cross-entropy loss $\ell$. Let $\Omega=\{\z_1 \dots \z_{|\Omega|}\}$ be a training mini-batch such that $(\x_i, y_i, b_i)  \stackrel{\mathclap{\mbox{\tiny{i.i.d.}}}}{\sim} \mu_Z^{\mathrm{train}}$. Let $\Omega_{\mathrm{bc}}:=\{\z_i \in \Omega| y_i\neq b_i \}$ and $\Omega_{\mathrm{ba}}:=\{\z_i \in \Omega| y_i=b_i \}$ represent the bias-conflicting and bias-aligned subsets of $\Omega$. If $f_{\w}$ predicts according to the spurious decision rule, \ie $b = \argmax_{j} f_{\w}(\x)[j]$, then
    for some $\alpha > 0$, as the logit-scaling factor $k \to \infty$    \vspace{-2mm}
    \begin{align}
    	\label{eq:loss_scaling}
    	\frac{\sum_{\x, y \in \Omega_{\mathrm{bc}}} \ell(y, k f_{\w}(\x))}{\sum_{\x, y \in \Omega_{\mathrm{ba}}} \ell(y, k f_{\w}(\x))} = \mathcal{O}(ke^{\alpha k}).
    \end{align}
    \vspace{-5mm}
\end{prop}
The proof is deferred to the suppl. material, \changedc{which also presents another proposition, }showing how this can result in an increase in the gradient norms to the subnetworks that utilize spurious vs. core features. 
These imply that large LRs generate strong gradients that discourage reliance on spurious features, whereas small LRs lack this pressure, allowing models to retain spurious subnetworks while memorizing BC samples. 
Through experiments with ResNet18, \cref{fig:analysis_max_bcloss} demonstrate that unbiased test accuracy \emph{very strongly} correlates with the maximum loss ratio (\ie BC loss/total minibatch loss) encountered during training across datasets, supporting our proposed mechanism (see \changedc{suppl. material } for similar results with additional datasets and models). \cref{fig:analysis_resnet18} visualizes training dynamics under a large vs. small LR ($\eta=0.001$ vs $0.1$): early focus on spurious features are weaned off under large LRs, as BC samples dominate the loss due to confident mispredictions, both in Colored MNIST dataset \textit{and} CIFAR-10 datasets. 
 \begin{figure}[t]
	\centering
	\begin{subfigure}{\linewidth}\vspace{1mm}
		\centering \includegraphics[width=\linewidth]{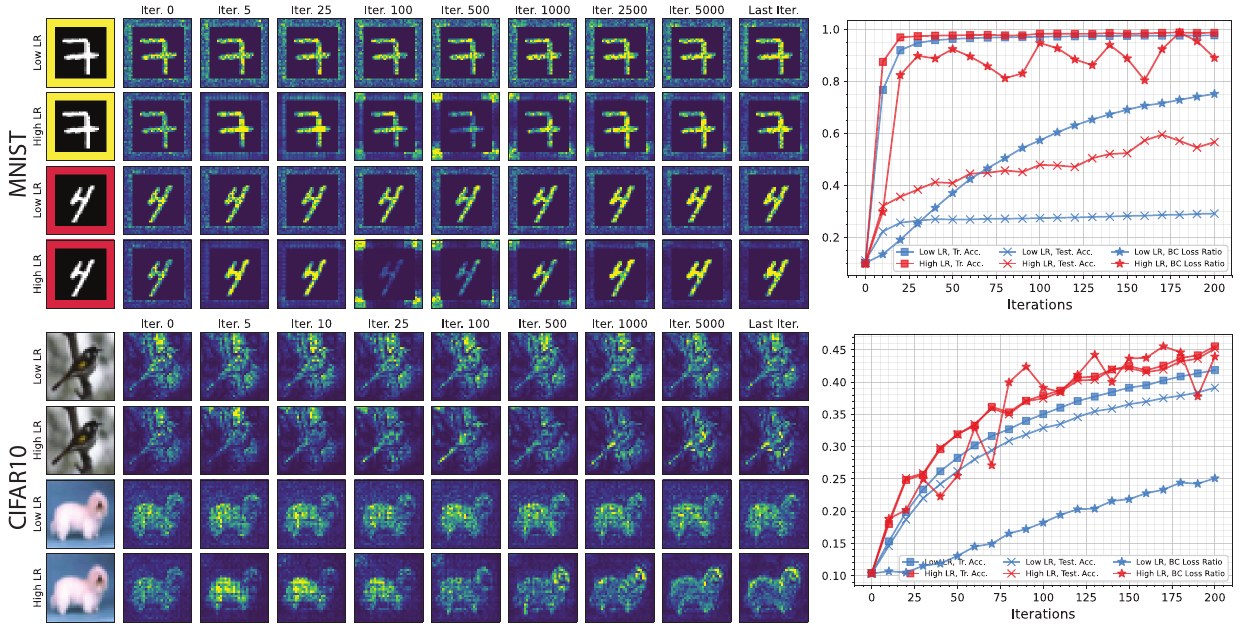}
	\end{subfigure}
	\vspace{-5mm}
	\caption{Large LRs lead to losses from bias-conflicting samples to dominate the gradient, resulting in increased utilization of core features. (Left) Attribution maps (right) model statistics for Colored MNIST and CIFAR-10 through training. \vspace{-4mm}}
	\label{fig:analysis_resnet18}
\end{figure}

We can interpret these large, disruptive updates to the network under the outsized influence of bias-conflicting examples as a recurring \textit{lottery ticket} type scenario \cite{frankleLotteryTicket2019}, where spurious subnetworks are effectively reset by large gradient updates whenever they rely too much on spurious feature and lead to confident mispredictions, similar to the
``catapult mechanism'', proposed by~\cite{lewkowyczLargeLearning2020a}. 
\vspace{-0.5em}
\section{Conclusion}  
\vspace{-1mm}
Robustness to SCs and compressibility are crucial requirements for advancing safe, trustworthy, and resource-efficient deep learning. In the absence of strong theoretical guarantees, our investigation demonstrates the efficacy of LR as a significant inductive bias for addressing these concerns simultaneously. We show that high LRs jointly obtain robustness to spurious correlations, network compressibility, and favorable representation properties.

\paragraph{Limitations and future work}  
Our work is far from an end-to-end, convergent explanation involving optimization dynamics, parameter loss landscapes, and representations. Future work should also study how multiple/hierarchical SCs interact with training dynamics, and design explicit training procedures for more robust, efficient models. 
 

\section*{Acknowledgments} 
\changedc{M. Barsbey was supported  by the EPSRC Project GNOMON
(EP/X011364/1). L. Prieto was supported by the UKRI Centre for Doctoral Training in Safe and Trusted AI
[EP/S0233356/1].} S. Zafeiriou and part of the research was funded by the EPSRC Fellowship DEFORM (EP/S010203/1) and EPSRC Project GNOMON
(EP/X011364/1). T. Birdal acknowledges support from the
Engineering and Physical Sciences Research Council [grant EP/X011364/1]. T. Birdal was supported by a UKRI Future Leaders Fellowship [grant MR/Y018818/1] as well as a Royal Society
Research Grant [RG/R1/241402]. 
{
    \small
    \bibliographystyle{ieeenat_fullname}
    \bibliography{ICCV/ICCV_references.bib}
}

 \clearpage
\setcounter{page}{1}
\setcounter{prop}{0}
\setcounter{section}{0}
\renewcommand\thesection{\Alph{section}}
\maketitlesupplementary

In this document, we present in \cref{apx:extended_related} an extended review of the preceding research to contextualize our paper's motivation and findings. Building on this, \cref{apx:extended_discussion} we highlight how our findings extend and improve upon previous literature, and point toward fruitful future research directions. After providing additional details regarding our experiment settings in \cref{apx:experiment_settings}, we provide additional results and statistics regarding the main paper's findings in \cref{apx:additional_results}. Lastly, \cref{apx:attribution}
reviews our attribution visualization methodology, and provides extensive additional visual evidence for our claims.

\setcounter{page}{1}

\renewcommand{\thesection}{\Alph{section}}
\setcounter{section}{0}
\section{Extended Related Work}
\label{apx:extended_related}
A large amount of recent work on spurious correlations (SCs) have focused on the ``default'' tendency of neural networks, trained under gradient-based empirical risk minimization (ERM), to exploit simple features in the training datasets at the expense of more complex yet robust/invariant ones. These include \cite{beeryRecognitionTerra2018}, who highlight the overreliance of vision models on background information; \cite{shahPitfallsSimplicity2020}, who emphasize the ``simplicity bias'' of neural networks in preferring simple features over more complex and informative ones, and \cite{geirhosShortcutLearning2020}, who emphasize the tendency of neural networks in engaging ``shortcut learning'' in various modalities of application. Ensuing research proposed several explanations for this phenomenon. For example, \cite{sagawaInvestigationWhy2020, nagarajanUnderstandingFailure2022, puliDonBlame2023} highlight the inductive bias of a maximum margin classifier as the primary reason for the exploitation of spurious features. Alternatively, \cite{pezeshkiGradientStarvation2021, nagarajanUnderstandingFailure2022} emphasize the dynamics of gradient-based learning in creating this effect, where early adoption of simple (and spurious) features harms the later learning of more complex and more informative features.

Building on diagnoses, such as those mentioned above, for the cause of this unintended learning of spurious features, other research propose interventions to mitigate this problem. For example, \cite{puliDonBlame2023} propose new losses that optimize for a \textit{uniform} margin solution rather than a max-margin solution. On the other hand, \cite{namLearningFailure2020, liuJustTrain2021} propose two-stage methods that reweight the dataset by deemphasizing samples that are learned earlier, and \cite{tiwariOvercomingSimplicity2023} discourage neural network to produce representations predictive of the label early in the neural network. Other methods assume access to spurious feature labels at training time, and exploit these in various ways to improve robustness \cite{sagawaDistributionallyRobust2020a, idrissiSimpleData2022}. While to our knowledge no previous research \textit{systematically} investigates the effect of LR on generalization under SCs (and in relation to compressibility), some previous research hint at the outsized impact of LR on such behavior.  \cite{liExplainingRegularization2019} examine the effect of large LRs on feature learning and generalization, without explicitly addressing the implications in an OOD generalization context. While \cite{pezeshkiMultiscaleFeature2022} speculate about the potential effects of LR tuning on OOD generalization, \cite{idrissiSimpleData2022} empirically find that LR is most likely to affect robustness to SCs, and \cite{puliDonBlame2023} speculate that large LRs might lead to improved performance due to inability of models trained thereunder to approximate a max-margin solution.

While previous research showed a positive relationship between compressibility and generalization through theoretical and empirical findings \cite{aroraStrongerGeneralization2018a, suzukiCompressionBased2020, suzukiSpectralPruning2020, barsbeyHeavyTails2021, birdalIntrinsicDimension2021}, it is much less clear how well this applies to OOD generalization. Indeed, existing research provides at best an ambivalent picture regarding the simultaneous achieveability of generalization, robustness, and compressibility \cite{wangAdversarialRobustness2018, dinhSparsityMeets2020, diffenderferWinningHand2021, duRobustnessChallenges2023}. Various studies have highlighted the impact of large LRs on generalization \cite{liExplainingRegularization2019, lewkowyczLargeLearning2020a,mohtashamiSpecialProperties2023}, model compressibility \cite{barsbeyHeavyTails2021}, and representation sparsity \cite{andriushchenkoSGDLarge2023}; making it a prime candidate for further investigation regarding its ability to facilitate compressibility and robustness in tandem.  \cite{jastrzebskiBreakEvenPoint2020} point out how large LRs in early training prevent the iterates from being stuck in narrow valleys in the loss landscape, where the curvature in certain directions is high. \cite{ lewkowyczLargeLearning2020a, mohtashamiSpecialProperties2023} point out the importance of large LRs in early training, where basin-jumping behavior leads to better generalizing and/or flatter minima \cite{keskarLargeBatchTraining2017}. While \cite{liExplainingRegularization2019, andriushchenkoSGDLarge2023, zhuCatapultsSGD2024} focus on the effect of large LRs on feature learning, 
\cite{rosenfeldOutliersOpposing2024} demonstrate the crucial role of spurious / opposing signals in early training, and how progressive sharpening \cite{cohenGradientDescent2021, liAnalyzingSharpness2022} of the loss landscape in the directions that pertain to the representation of these features lead to the eventual down-weighting of such non-robust features. \cite{rosenfeldOutliersOpposing2024} further observe that this is due to discrete nature of practical steepest ascent methods (GD, SGD), as it is not observed in gradient flow regime, suggesting learning rate as a prime candidate for modulating this behavior.

\begin{figure*}[t]
    \centering
\begin{subfigure}{0.48\textwidth}
  \centering
\includegraphics[width=\linewidth]{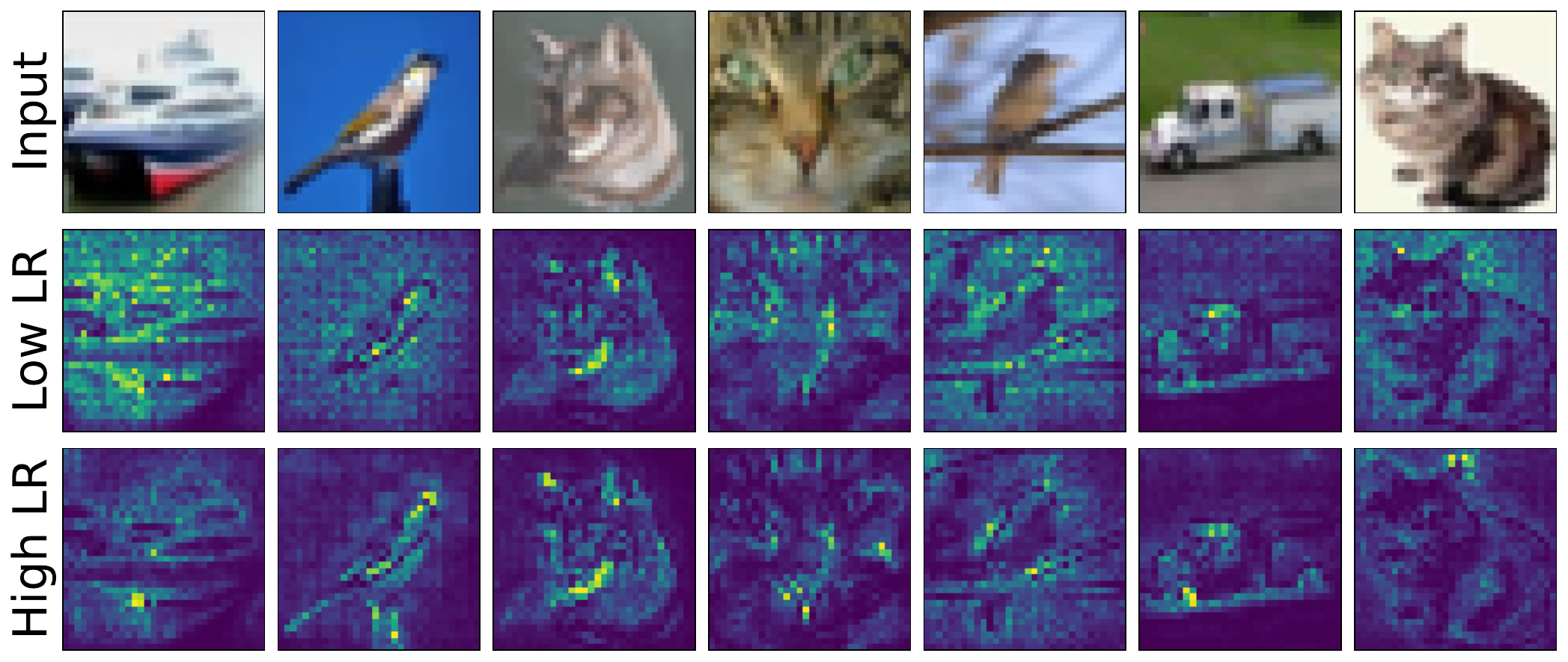}
\end{subfigure}
\hfil
\begin{subfigure}{0.48\textwidth}
  \centering
\includegraphics[width=\linewidth]{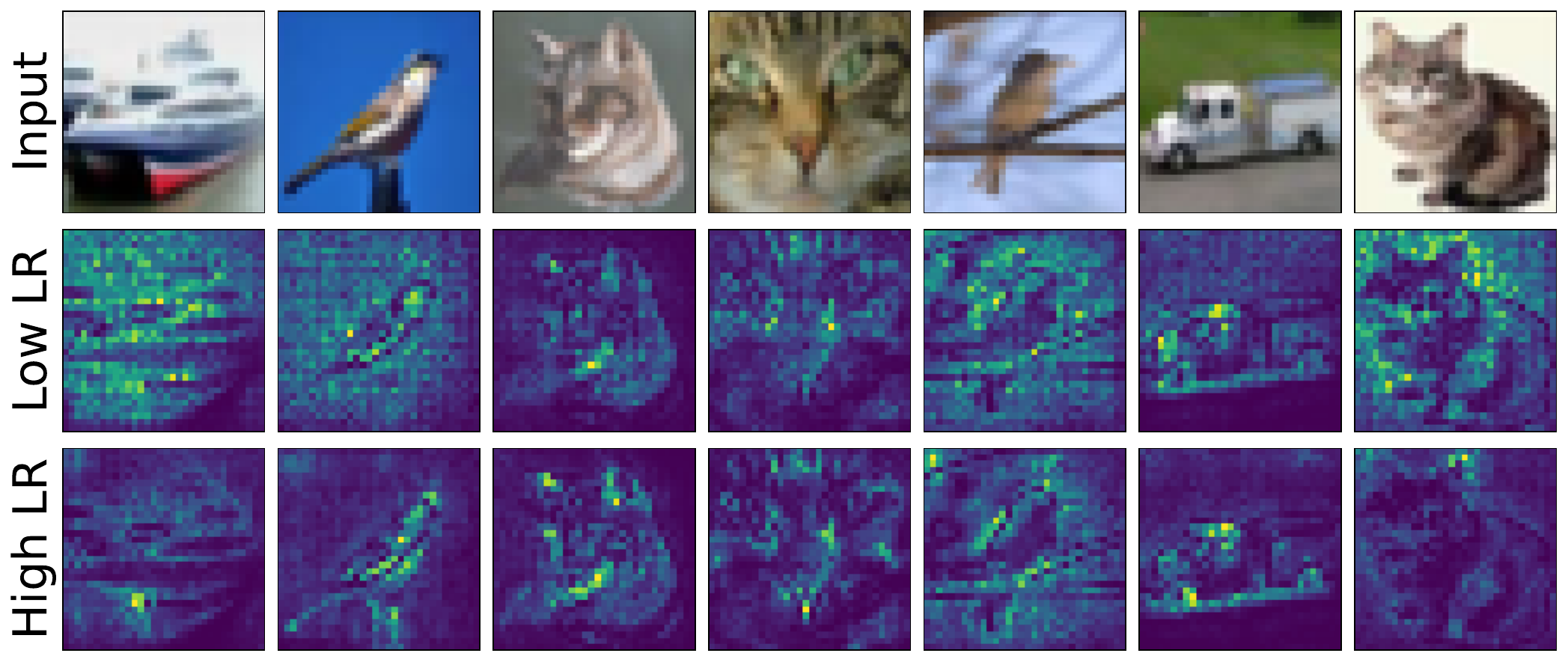}
\end{subfigure}
\caption{Visualizing attributions on a CIFAR-10 dataset with ResNet18 models using Integrated Gradients (left) and DeepLift (right).\vspace{-2mm}}  
    \label{fig:ig_vs_deeplift}
\vspace{3mm}
\end{figure*}

\begin{figure}[t]
    \centering
    \begin{subfigure}{0.45\linewidth}
      \centering
\includegraphics[width=\linewidth]{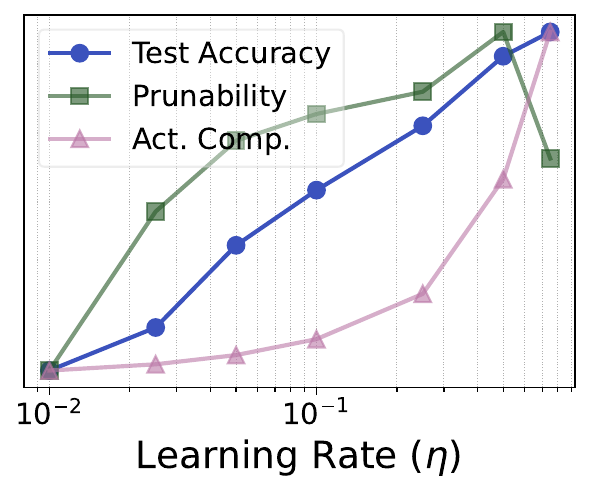}
    \end{subfigure}
    \centering
    \begin{subfigure}{0.45\linewidth}
      \centering
\includegraphics[width=\linewidth]{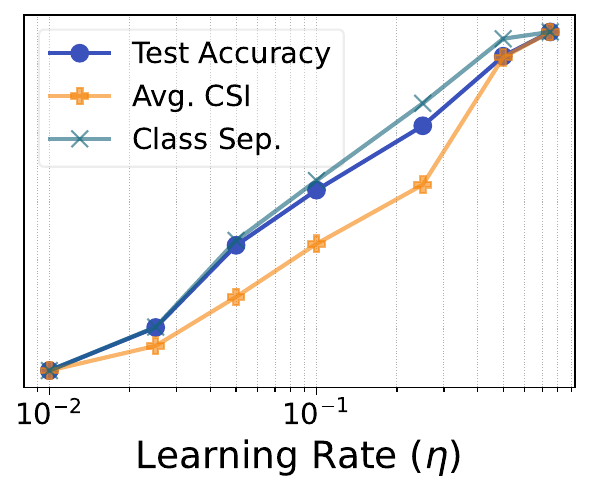}
    \end{subfigure}
    \vspace{-3mm}
    \caption{(Left) Effects of learning rate on OOD performance (unbiased test acc.), network prunability, and representation properties with the moon-star dataset.\vspace{-5mm}}
    \label{fig:moonstar_results}
\end{figure}

\section{Extended Discussion of Our Contributions}
\label{apx:extended_discussion}
In this paper, we demonstrate through systematic experiments the unique role large learning rates (LRs) in simultaneously achieving robustness and resource-efficiency. More concretely, we demonstrate that:

\begin{itemize}
\item Large LRs simultaneously facilitate robustness to SCs and compressibility in a variety of datasets, model architectures, and training schemes.
\item Increase in robustness and compressibility is accompanied by increased core (aka stable, invariant) feature utilization and class separation in learned representations.
\item Large LRs are unique in consistently achieving these properties across datasets compared to other interventions, and can be combined with explicit regularization for even better performance.
\item Large LRs have a similar effect in naturalistic classification tasks by addressing hidden/rare spurious correlations in the dataset.
\item Confident mispredictions of bias-conflicting samples play an important role in conferring robustness to models trained under large LRs.
\end{itemize}

We now discuss further implications of our results in light of recent findings in the literature.

\begin{figure*}[t]
  \centering
\begin{subfigure}{\textwidth}
  \centering
  \includegraphics[width=\textwidth]{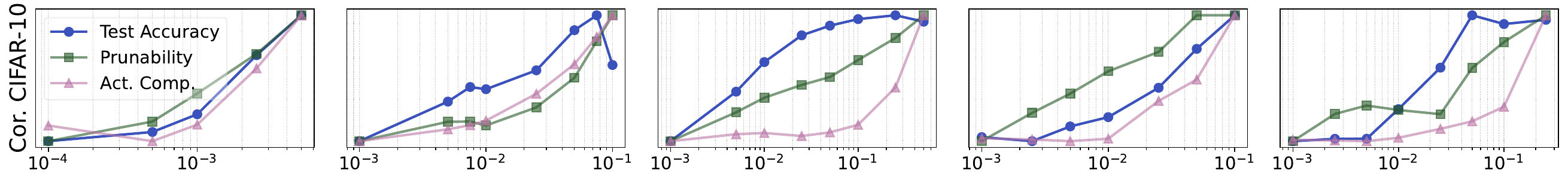}
\end{subfigure}
\vspace{-6mm}
\caption{Effects of LR on OOD performance (unbiased test acc.), network prunability, and representation (activation) compressibility in Corrupted CIFAR-10 dataset. $x$-axes correspond to learning rate ($\eta$), $y$-axes are normalized within each figure for each variable.\vspace{-4mm}}
  \label{fig:corcifar_results}
\end{figure*}

\begin{figure*}[t]
  \centering
\begin{subfigure}{\textwidth}
  \centering
\includegraphics[width=\textwidth]{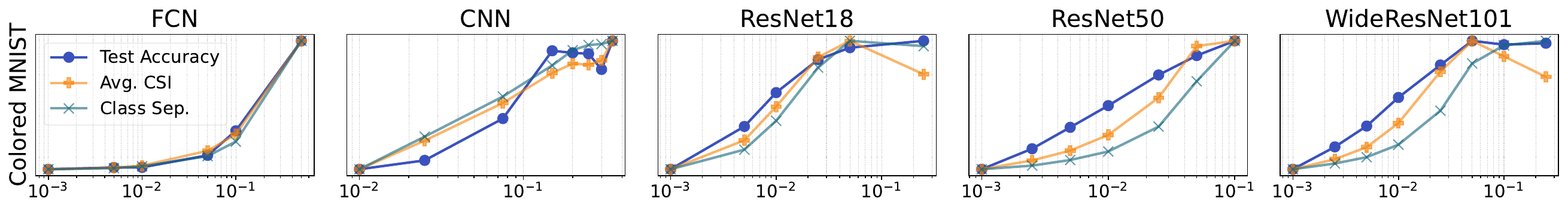}
\end{subfigure}

\begin{subfigure}{\textwidth}
  \centering
\includegraphics[width=\textwidth]{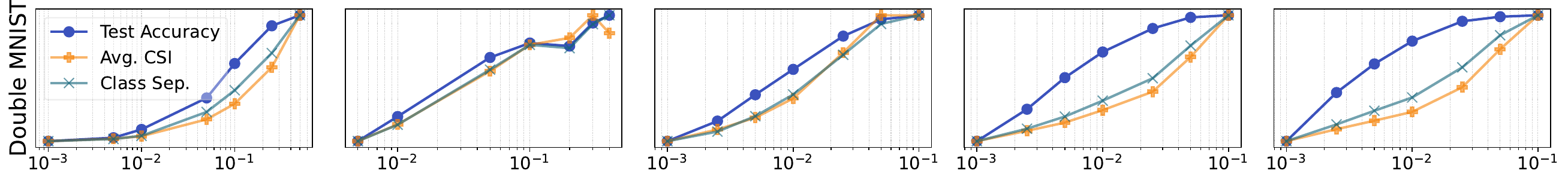}
\end{subfigure}

\begin{subfigure}{\textwidth}
  \centering
\includegraphics[width=\textwidth]{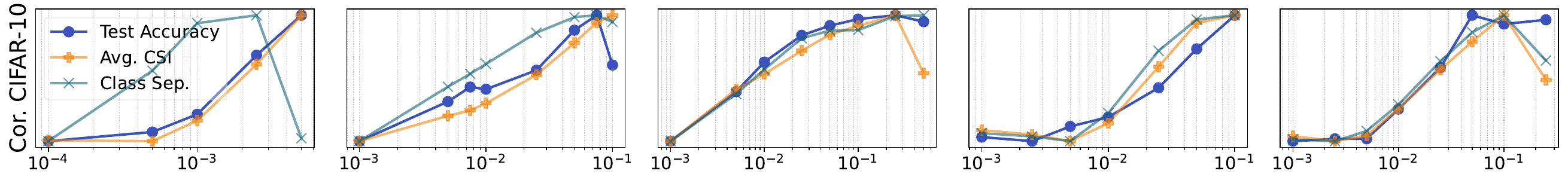}
\end{subfigure}

\begin{subfigure}{\textwidth}
  \centering
\includegraphics[width=\textwidth]{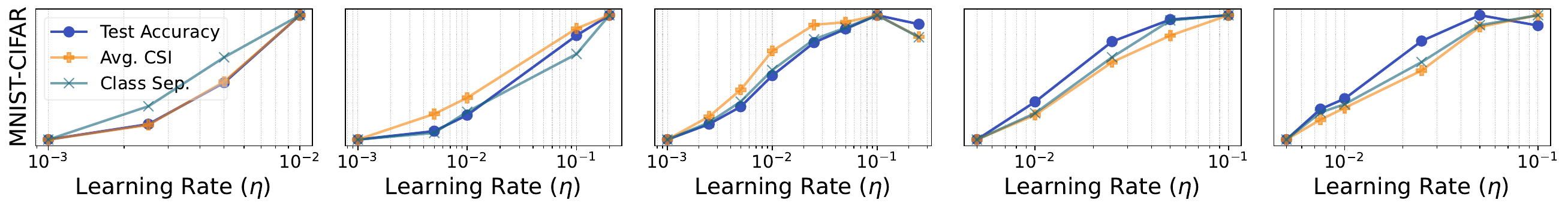}
\end{subfigure}
\vspace{-6mm}
\caption{Effects of learning rate on  representation (activation) statistics for semi-synthetic datasets. $y$-axes are normalized within each figure for each variable.\vspace{-4mm}}
  \label{fig:repr5_results}
\end{figure*}

\paragraph{Inductive biases of SGD} Our findings call into question the assumptions regarding the inductive biases of ``default'' SGD. We find that LR selection can change the unbiased test set accuracy by \textit{up to $\sim$50\%} (See Table S1). This has two major implications: (i) Any method that relies on assumptions regarding default behavior of neural networks (e.g. \cite{namLearningFailure2020, liuJustTrain2021}) should consider the fact that the said defaults can vary considerably based on training hyperparameters, including but not limited to LR. (ii) Any proposed intervention for improving robustness to SCs should consider utilizing large LRs as a strong baseline against the proposed method. 

\paragraph{Overparametrization and robustness} We observe a strong interaction between overparametrization and LR in robustness to SCs. Our findings show that (Table S1) the range of test accuracies that can be obtained by tuning LRs increase as the models get more expressive. For example, in Colored MNIST dataset, while the difference between the highest vs. lowest performing models is $\sim$4\% for a fully connected network (FCN), this increases to $\sim$23\% for ResNet18 and $\sim$31\% for ResNet50. This implies that while overparametrization indeed seems to play an important role in robustness to SCs as suggested by some previous work \cite{sagawaInvestigationWhy2020} (cf. \cite{puliDonBlame2023}), this needs to be considered in the context of central training hyperparameters, as they can modulate this vulnerability to a great degree.

\paragraph{Mechanism of LR's Effects} 
As noted in the main paper, interventions developed to mitigate the effect of spurious correlations constitute two groups based on whether they assume access to spurious feature labels/annotations. Those that assume this, exploit this information to improve worst group or unbiased test set performance \cite{sagawaDistributionallyRobust2020a, idrissiSimpleData2022}. In the absence of group annotations, other methods rely on assumptions about the nature of the spurious features, data distribution, and the inductive biases of the learning algorithms \cite{liuJustTrain2021, puliDonBlame2023, tiwariOvercomingSimplicity2023}. We highlight that the proposed mechanism in this paper, where large LRs cause BC sample losses to dominate acts as an \emph{implicit} re-weighting, which makes it akin to two-stage methods like Learning from Failure \cite{namLearningFailure2020}. This further highlights the importance of establishing the effects of core hyperparameters on robustness to spurious correlations - not only for a deeper understanding of the inductive biases of gradient-based learning under overparametrization, but also as strong baselines to compare against newly developed methods.

As noted above, \cite{puliDonBlame2023} argue that max-margin classification inevitably leads to exploitation of spurious features, and LRs might protect against SCs by creating models closer to a uniform margin solution. To support their conjecture, they present evidence that shows average losses incurred from bias-aligned (BA) vs. bias-conflicting (BC) samples are closer in large LR models compared to small LR models. However, we find that their findings do not generalize across different data distributions. Computing avg. loss for BA samples / avg. loss for BC samples, we find that in both Colored MNIST and Double MNIST datasets, low LR models produce average losses that are closer in ratio ($3.9 \times 10^-3$ vs. $3.4 \times 10^-3$ in Colored MNIST and $9.5 \times 10^-3$ vs. $7.4 \times 10^-6$ in Double MNIST). This highlights the importance of testing such claims across diverse settings, and emphasizes the need for novel and systematic explanations for the effect of large LRs on robustness to SCs. 

\paragraph{Compressibility and generalization} \cite{andriushchenkoSGDLarge2023} argue that large LRs create models with sparse representations. Our findings support their claim across diverse settings. However, we also observe that LRs' effects on unbiased test accuracy and network compressibility (i.e. prunability) precede that of activation sparsity (i.e. there are LRs that are large enough to increase test accuracy and prunability but not large enough to increase representation sparsity). This strongly implies that the representation sparsity is a downstream effect of large LRs, rather than being a mediator of generalization and network compressibility. On the other hand, our findings include initial evidence for wide minima \cite{keskarLargeBatchTraining2017} found by large LRs to be associated with increased core feature utilization. Examining the interaction of parameter and representation space properties produced by LRs (see e.g. \cite{rosenfeldOutliersOpposing2024}) is a promising future direction for understanding the inductive bias of large LRs and SGD in general.

\section{Further Details on Experiment Settings}
\label{apx:experiment_settings}
We use Python programming language for all experiments included in this paper. For experiments with semi-synthetic, realistic SC, and naturalistic datasets we use the versions of ResNet18, ResNet50, Wide ResNet101-2, Swin Transformer (tiny) as included in the Python package \texttt{torchvision}, as well as a FCN with two hidden layers of width 1024, ReLU as the activation activation function, and with no bias. We also use a CNN with a similar architecture to VGG11 \cite{simonyanVeryDeep2015}, with a single linear layer following the convolutional layers instead of three, and this version includes no bias terms. Due to the worse default performance of FCNs in the more difficult semi-synthetic datasets MNIST-CIFAR and Corrupted CIFAR-10, we increase bias-conflicting ratio $\rho$ to 0.25, and increase network width to 2048 for these datasets. The synthetic dataset experiments have been conducted with an FCN of two hidden layers and 200 width. 

For Colored MNIST and Corrupted CIFAR-10 we use the train/test splits from the original papers \cite{liuJustTrain2021}. Double MNIST and MNIST-CIFAR are created using the canonical splits of these datasets. The training/test splits for these datasets are 60000/10000, 50000/10000, 60000/10000, and 50000/10000, respectively. While we use the original splits for CelebA and Waterbirds datasets, we use a 10000/10000 split for the synthetic parity dataset. The learning rate ranges for the experiments are provided in \cref{tbl:lr_result_ranges_main} and \cref{tbl:lr_result_ranges_reg}, while all experiments included a batch size of 100, except for experiments with Swin Transformer, where we utilize a batch size of 16. For computing activation statistics, we obtain the post-activation values for the penultimate layer, and compute the compressibility values for 1000 randomly sampled input from the test set, and present the average of these values. 

The experiments in this paper were run on 4 NVIDIA A100-PCIE-40GB GPUs for 400 total hours of computation. We will make our source code public upon publication to allow for the replication of our results.

\section{Additional Results and Statistics}
\label{apx:additional_results}
\subsection{Proof of Proposition 1}
\begin{proof}
	Let $y$ and $y'$ for the correct and incorrect classes for a sample, \ie $b=y'$ for bias-conflicting samples, and $b=y$ otherwise. For the mispredicted (bias-conflicting) examples, let $f_{\w}[y'] - f_{\w}[y] = \beta > 0$. This implies the following softmax ($\pi$) output for $y$
	\begin{align}
		\pi_y  = \frac{e^{k f_{\w}[y]}}{e^{k f_{\w}[y']} + e^{k f_{\w}[y]}} = \frac{1}{1+e^{k\beta}}.
	\end{align}
	Notice that $\pi_y \approx e^{-k\beta}$, as $k \to \infty$. Then we can say $\ell(y, f_{\w}(\x)) = -\log \pi_y \approx k\beta.$
	
	For the correctly predicted samples, let $f_{\w}[y] - f_{\w}[y'] = \alpha > 0$. Similarly to above, note that
	\begin{align}
		\pi_y = 1 - \pi_{y'} &= 1-\frac{e^{k f_{\w}[y']}}{e^{k f_{\w}[y']} + e^{k f_{\w}[y]}} \\
              &= 1-\frac{1}{1 + e^{k (f_{\w}[y]-f_{\w}[y'])}} \\
              &= 1-\frac{1}{1 + e^{k\alpha}} \\
              &\approx 1- e^{-\alpha k}
	\end{align}
	as $k\to\infty$. As $k\to \infty$, $-\log(\pi_{y}) \approx -\log(1-e^{-k\alpha}) \approx e^{-k\alpha}$.    
	Assume (without loss of generality) that the margins $\beta$ and $\alpha$ are shared by all bias-conflicting and bias-aligned samples in the minibatch. Then, as $k\to\infty$,
	\begin{align}
		\frac{\sum_{\x, y \in \Omega_{\mathrm{bc}}} \ell(y, k f_{\w}(\x))}{\sum_{\x, y \in \Omega_{\mathrm{ba}}} \ell(y, k f_{\w}(\x))} \approx \frac{|\Omega_{\mathrm{bc}}|\cdot k\beta}{|\Omega_{\mathrm{ba}}|\cdot e^{-\alpha k}} = \mathcal{O}(ke^{\alpha k}),
	\end{align}
	proving \cref{eq:loss_scaling}.
\end{proof}
\subsection{Additional Theoretical Results}
To investigate the effects of mispredicted bias-conflicting samples on the gradients of subnetworks that rely on core vs. spurious features, we first define \textit{bias-decomposable network}s. 
\begin{dfn}
$f_\w$ is called a \textit{bias-decomposable network} if $f_\w(\x) = f_{\mathbf{c}}(\x) + f_{\mathbf{s}}(\x) + f_{\mathbf{r}}(\x)$. Here, $f_{\mathbf{c}}$ is the core feature subnetwork, $f_{\mathbf{c}}(\x)[y] - f_{\mathbf{c}}(\x)[b] \geq 0$ with equality iff $y=b$. $f_{\mathbf{c}}(\x)$ is assumed to have converged to a stable decision making rule, i.e. $\nabla_{\mathbf{c}}(f_{\mathbf{c}}(\x)[y] - f_{\mathbf{c}}(\x)[b]) = 0$. In contrast, $f_{\mathbf{s}}$ is the spurious feature subnetwork, $f_{\mathbf{s}}(\x)[y] - f_{\mathbf{s}}(\x)[b] \leq 0$ with equality iff $y=b$. $f_{\mathbf{r}}$ is the remainder subnetwork that does not conform to the behavior described for $f_{\mathbf{c}}, f_{\mathbf{s}}$.
\end{dfn}

Before discussing the motivation for this idealization, in the following proposition, we investigate how the gradient norms for core and spurious subnetworks scale based on this definition and the results in the main paper.
\setcounter{prop}{1}
\begin{prop}
    Assume $f_\w$ is bias-decomposable network. If $f_{\w}$ predicts according to the spurious decision rule, \ie $b = \argmax_{j} f_{\w}(\x)[j]$, then
    for some $\alpha > 0$, as the logit-scaling factor $k \to \infty$: 
    \begin{align}
    	\label{eq:grad_scaling}
    	\frac{\sum_{\x, y \in \Omega} \|\nabla_{\mathbf{s}}\ell(y, k f_{\w}(\x))\|}{\sum_{\x, y \in \Omega} \|\nabla_{\mathbf{c}}\ell(y, k f_{\w}(\x))\| } = \mathcal{O}(e^{\alpha k}),
    \end{align}
    for some $\alpha > 0$, where $\|\cdot\|$ stands for Frobenius norm.
\end{prop}
\begin{proof}
The softmax probability $\pi_j$ for a class $j$ is given by:
\begin{align}
\pi_j = \frac{e^{z_j}}{\sum_i e^{z_i}} \quad \text{where } z_i = k f_w(\mathbf{x})[i]
\end{align}
With some abuse of notation, the gradient of the loss $\ell$ can be expressed as:
\begin{align}
    \nabla_c \ell &= k \sum_{j} (\pi_j - \delta_{jy}) \nabla_c f_c(\mathbf{x})[j] \\
    \nabla_s \ell &= k \sum_{j} (\pi_j - \delta_{jy}) \nabla_s f_s(\mathbf{x})[j]
\end{align}

where $\delta_{jy}$ is the Kronecker delta, defined as:

\begin{align}
\delta_{jy} = \begin{cases} 1 & \text{if } j=y \\ 0 & \text{if } j \neq y \end{cases}
\end{align}

For bias-conflicting samples, i.e. $\mathbf{x} \in \Omega_{bc}$, as $k \to \infty$, the softmax probability $\pi_b \to 1$ and $\pi_j \to 0$ for $j \neq b$. The gradient for the core feature subnetwork converges to:

\begin{align}
\nabla_c \ell &\approx k \left( \nabla_c f_c[b] - \nabla_c f_c[y] \right) \label{eq:diff_grads}\\  &\approx 0,
\end{align}
with the latter due to Definition 1. Note that for the spurious feature subnetwork \cref{eq:diff_grads} implies that the gradient norm scales linearly with $k$:
\begin{align}
\lVert \nabla_s \ell_{bc} \rVert = \mathcal{O}(k)
\end{align}

For the correctly classified bias-aligned samples with margin $\alpha$, the gradient norms for both subnetworks are scaled by this vanishing factor:
\begin{align}
    \lVert \nabla_c \ell_{ba} \rVert &= \mathcal{O}(k e^{-k\alpha}) \\
    \lVert \nabla_s \ell_{ba} \rVert &= \mathcal{O}(k e^{-k\alpha})
\end{align}

As we sum the norms over the minibatch $\Omega = \Omega_{bc} \cup \Omega_{ba}$, the total spurious gradient norm (numerator) is:
\begin{align}
\sum_{\Omega} \lVert \nabla_s \ell \rVert = \sum_{\Omega_{bc}} \mathcal{O}(k) + \sum_{\Omega_{ba}} \mathcal{O}(k e^{-k\alpha}) = \mathcal{O}(k)
\end{align}
The total core gradient norm (denominator) is:
\begin{align}
\sum_{\Omega} \lVert \nabla_c \ell \rVert = \sum_{\Omega_{bc}} 0 + \sum_{\Omega_{ba}} \mathcal{O}(k e^{-k\alpha}) = \mathcal{O}(k e^{-k\alpha})
\end{align}
The final ratio is the ratio of their asymptotic behaviors:
\begin{align}
\frac{\sum_{\Omega} \lVert \nabla_s \ell \rVert}{\sum_{\Omega} \lVert \nabla_c \ell \rVert} = \frac{\mathcal{O}(k)}{\mathcal{O}(k e^{-k\alpha})} = \mathcal{O}(e^{\alpha k}) 
\end{align}
\end{proof}

Note that although the decomposition in question is an idealization, it follows in line of previous work that demonstrate such modularity \cite{ kirichenkoLastLayer2022}, and similar decompositions have been utilized in related research previously \cite{qiu2024complexity}. However, note that our assumption regarding the stability of the core vs. spurious subnetworks differ from that of \cite{qiu2024complexity}, who assume a Bayes-optimal, stable spurious subnetwork. Our assumption is motivated by our empirical observations. In \cref{fig:rebuttal_figs_mech} (top), we examine two ResNet18 models trained on Colored MNIST dataset, and investigate their learned representations, based on the pooled outputs of the last layer convolution filters. We compute the CSI - BSI values for each filter/neuron at different points in the training. Assume we categorize neurons for which CSI $>$ BSI as ``class-dominant'', and ``bias-dominant'' otherwise. We then ask, ``At iteration $t$, what percentage of the neurons that were class dominant \textit{remained} class-dominant by the end of training, and what percent of the neurons that were bias-dominant \textit{remained} so?'' As the results show, for high learning rate ($\eta=0.1$ vs. $\eta=0.001$), after iteration $\sim 10$ the overwhelming majority (not infrequently 100\%) of the neurons who were class-dominant remained so. This is not true for the bias-dominant neurons. Note that this is not explained solely by the final ratio of class-dominant neurons: 55.4\% vs 85.9\% for the two models respectively, which falls short of explaining the behavior observed. To provide a more microscopic examination of this, in \cref{fig:rebuttal_figs_mech} (bottom) we examine how neuron-specific gradients impact spurious feature utilization (computed through feature attribution) within that neuron, from a high LR ($\eta=0.1$) experiment within our FCN + Colored MNIST setting. The updates clearly show that the increasing spurious feature utilization is ``reset'' by a large gradient update, driven presumably by mispredictions from bias-conflicting samples. We consider further examination of these phenomena as an important future research direction.

\begin{figure}[ht]
\centering
\begin{subfigure}{0.45\linewidth}
  \includegraphics[width=\linewidth]{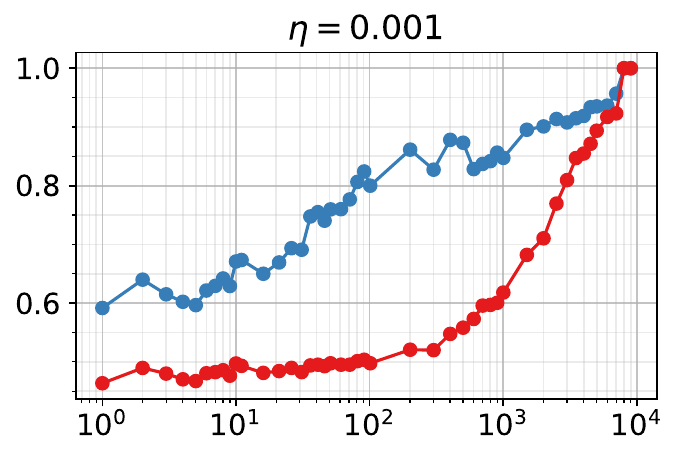} 
\end{subfigure}
\begin{subfigure}{0.45\linewidth}
  \includegraphics[width=\linewidth]{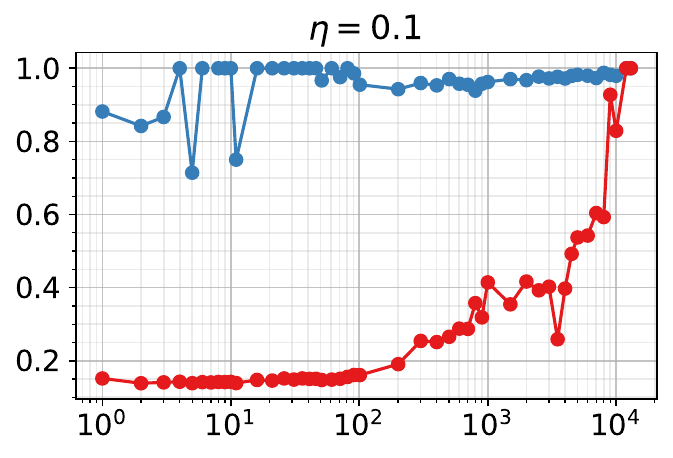} 
\end{subfigure}
\begin{subfigure}{\linewidth}
  \includegraphics[width=\linewidth]{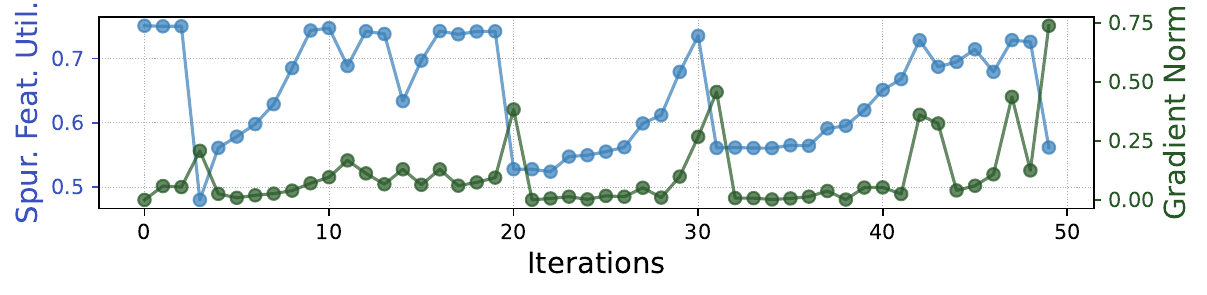} \vspace{-2mm}
\end{subfigure}\vspace{-6mm}
\caption{(Top) Ratio of class-dominant vs. bias-dominant neurons that survived as such by the end of training. (Bottom) Examining the impact of large gradient updates on feature attributions.}\vspace{-2mm}
\label{fig:rebuttal_figs_mech}
\end{figure}

\subsection{Value Ranges for Figures}
Given that our figures depict multiple variables at the same time, and the results are normalized according to experiments to illuminate the patterns that LR and other interventions create, we present the min. and max. values the independent and dependent variables take in \cref{tbl:lr_result_ranges_main} and \cref{tbl:lr_result_ranges_reg}.

\begin{table*}[htbp]
\centering
\caption{Minimum and maximum values for each dataset-model combination included in our main experiments.}
\vspace{-2mm}
\resizebox{\textwidth}{!}{%
\begin{tabular}{|l|l|cc|cc|cc|cc|cc|cc|}
\textbf{Dataset} & \textbf{Model} & \multicolumn{2}{c|}{\textbf{LR}} & \multicolumn{2}{c|}{\textbf{Test Acc.}} & \multicolumn{2}{c|}{\textbf{Prunability}} & \multicolumn{2}{c|}{\textbf{Act. Comp.}} & \multicolumn{2}{c|}{\textbf{Avg. CSI}} & \multicolumn{2}{c|}{\textbf{Class Sep.}} \\
\cline{3-14}
 &  & \textbf{Min.} & \textbf{Max.} & \textbf{Min.} & \textbf{Max.} & \textbf{Min.} & \textbf{Max.} & \textbf{Min.} & \textbf{Max.} & \textbf{Min.} & \textbf{Max.} & \textbf{Min.} & \textbf{Max.} \\
\toprule
MNIST-CIFAR & FCN & 0.001 & 0.01 & 35.247 & 35.369 & 0.930 & 0.940 & 0.191 & 0.215 & 0.108 & 0.136 & 0.13 & 0.143 \\
MNIST-CIFAR & CNN & 0.001 & 0.2 & 24.507 & 41.717 & 0.326 & 0.902 & 0.173 & 0.513 & 0.079 & 0.225 & 0.049 & 0.149 \\
MNIST-CIFAR & ResNet18 & 0.001 & 0.25 & 26.233 & 47.513 & 0.311 & 0.737 & 0.294 & 0.343 & 0.225 & 0.371 & 0.115 & 0.255 \\
MNIST-CIFAR & ResNet50 & 0.005 & 0.1 & 23.287 & 34.358 & 0.378 & 0.515 & 0.25 & 0.34 & 0.143 & 0.25 & 0.065 & 0.143 \\
MNIST-CIFAR & Wide ResNet101-2 & 0.005 & 0.1 & 24.855 & 33.537 & 0.418 & 0.522 & 0.252 & 0.34 & 0.166 & 0.247 & 0.082 & 0.149 \\
CelebA & Swin Transformer & 1e-07 & 0.0001 & 40.889 & 48.21 & 0.344 & 0.609 & 0.265 & 0.584 & 0.122 & 0.305 & 0.06 & 0.164 \\
Colored MNIST & FCN & 0.001 & 0.5 & 72.727 & 76.11 & 0.935 & 0.968 & 0.285 & 0.389 & 0.243 & 0.381 & 0.291 & 0.387 \\
Colored MNIST & CNN & 0.01 & 0.35 & 88.2 & 91.48 & 0.532 & 0.931 & 0.372 & 0.643 & 0.302 & 0.576 & 0.376 & 0.692 \\
Colored MNIST & ResNet18 & 0.001 & 0.25 & 68.343 & 91.543 & 0.252 & 0.771 & 0.381 & 0.527 & 0.38 & 0.63 & 0.212 & 0.635 \\
Colored MNIST & ResNet50 & 0.001 & 0.1 & 55.687 & 86.38 & 0.284 & 0.489 & 0.334 & 0.524 & 0.198 & 0.458 & 0.07 & 0.509 \\
Colored MNIST & ResNet50 (PSGD) & 1e-6 & 0.01 & 24.16 & 93.09 & 0.249 & 0.959 & 0.334 & 1.0 & 0.159 & 0.888 & 0.126 & 0.712 \\
Colored MNIST & Wide ResNet101-2 & 0.001 & 0.25 & 61.643 & 85.25 & 0.258 & 0.559 & 0.337 & 0.787 & 0.215 & 0.473 & 0.084 & 0.569 \\
Moon-Star & FCN & 0.01 & 0.75 & 74.315 & 81.156 & 0.954 & 0.971 & 0.375 & 0.457 & 0.257 & 0.34 & 0.2 & 0.326 \\
Cor. CIFAR-10 & FCN & 0.0001 & 0.005 & 46.92 & 51.93 & 0.574 & 0.847 & 0.202 & 0.252 & 0.131 & 0.183 & 0.169 & 0.182 \\
Cor. CIFAR-10 & CNN & 0.001 & 0.1 & 43.555 & 48.503 & 0.39 & 0.861 & 0.189 & 0.341 & 0.144 & 0.387 & 0.14 & 0.329 \\
Cor. CIFAR-10 & ResNet18 & 0.001 & 0.5 & 35.153 & 47.795 & 0.267 & 0.803 & 0.382 & 0.634 & 0.266 & 0.393 & 0.115 & 0.205 \\
Cor. CIFAR-10 & ResNet50 & 0.001 & 0.1 & 36.3 & 45.52 & 0.391 & 0.556 & 0.334 & 0.475 & 0.203 & 0.304 & 0.081 & 0.134 \\
Cor. CIFAR-10 & Wide ResNet101-2 & 0.001 & 0.25 & 37.827 & 47.153 & 0.399 & 0.559 & 0.339 & 0.653 & 0.215 & 0.348 & 0.09 & 0.189 \\
Double MNIST & FCN & 0.001 & 0.5 & 69.287 & 72.47 & 0.986 & 0.999 & 0.236 & 0.406 & 0.22 & 0.485 & 0.428 & 0.481 \\
Double MNIST & CNN & 0.005 & 0.4 & 83.987 & 94.57 & 0.383 & 0.916 & 0.235 & 0.772 & 0.187 & 0.476 & 0.204 & 0.55 \\
Double MNIST & ResNet18 & 0.001 & 0.1 & 85.44 & 95.577 & 0.285 & 0.706 & 0.306 & 0.365 & 0.467 & 0.602 & 0.485 & 0.776 \\
Double MNIST & ResNet50 & 0.001 & 0.1 & 42.045 & 95.733 & 0.228 & 0.498 & 0.16 & 0.256 & 0.152 & 0.392 & 0.172 & 0.68 \\
Double MNIST & Wide ResNet101-2 & 0.001 & 0.1 & 43.12 & 96.08 & 0.195 & 0.508 & 0.161 & 0.298 & 0.156 & 0.416 & 0.188 & 0.703 \\
Parity & FCN & 0.01 & 0.75 & 55.31 & 85.663 & 0.56 & 0.782 & 0.333 & 0.679 & 0.032 & 0.152 & 0.008 & 0.136 \\
\bottomrule
\end{tabular}
}
\label{tbl:lr_result_ranges_main}
\end{table*}

\begin{table*}[htbp]
\centering
\caption{Minimum and maximum values for each dataset-model combination included in our regularization experiments.}
\resizebox{\textwidth}{!}{%
\begin{tabular}{|l|l|cc|cc|cc|cc|}
\textbf{Dataset} & \textbf{Model} & \multicolumn{2}{c|}{\textbf{HP}} & \multicolumn{2}{c|}{\textbf{Test Acc.}} & \multicolumn{2}{c|}{\textbf{Prunability}} & \multicolumn{2}{c|}{\textbf{Avg. CSI}} \\
\cline{3-10}
 &  & \textbf{Min.} & \textbf{Max.} & \textbf{Min.} & \textbf{Max.} & \textbf{Min.} & \textbf{Max.} & \textbf{Min.} & \textbf{Max.} \\
\toprule
Colored MNIST & Learning Rate & 0.001 & 0.15 & 72.297 & 92.203 & 0.262 & 0.8 & 0.441 & 0.644 \\
Colored MNIST & High LR + L2 Reg. & 1e-05 & 0.001 & 92.183 & 94.25 & 0.79 & 0.87 & 0.651 & 0.848 \\
Colored MNIST & Batch Size & 0.01 & 0.3333 & 68.363 & 90.303 & 0.253 & 0.661 & 0.383 & 0.587 \\
Colored MNIST & Momentum & 0.1 & 0.99 & 68.38 & 91.197 & 0.256 & 0.748 & 0.387 & 0.633 \\
Colored MNIST & L1 Regularization & 1e-05 & 0.001 & 72.887 & 90.347 & 0.275 & 0.864 & 0.458 & 0.859 \\
Colored MNIST & L2 Regularization & 0.0001 & 0.05 & 71.727 & 89.687 & 0.263 & 0.76 & 0.432 & 0.873 \\
Colored MNIST & $\gamma$ (Focal Loss) & 0.001 & 25.0 & 68.203 & 91.8 & 0.254 & 0.645 & 0.377 & 0.626 \\
Colored MNIST & $\rho$ (ASAM) & 0.01 & 10.0 & 68.723 & 90.977 & 0.256 & 0.522 & 0.38 & 0.508 \\
MNIST-CIFAR & Learning Rate & 0.001 & 0.1 & 24.327 & 47.497 & 0.293 & 0.69 & 0.24 & 0.376 \\
MNIST-CIFAR & High LR + L2 Reg. & 1e-06 & 0.0001 & 46.947 & 48.327 & 0.677 & 0.727 & 0.384 & 0.497 \\
MNIST-CIFAR & Batch Size & 0.01 & 0.3333 & 25.95 & 46.24 & 0.29 & 0.529 & 0.201 & 0.262 \\
MNIST-CIFAR & Momentum & 0.1 & 0.99 & 25.887 & 46.397 & 0.29 & 0.666 & 0.226 & 0.354 \\
MNIST-CIFAR & L1 Regularization & 1e-05 & 0.001 & 24.65 & 46.097 & 0.28 & 0.845 & 0.24 & 0.486 \\
MNIST-CIFAR & L2 Regularization & 0.0001 & 0.075 & 24.18 & 47.13 & 0.287 & 0.619 & 0.189 & 0.533 \\
MNIST-CIFAR & $\gamma$ (Focal Loss) & 0.01 & 25.0 & 25.87 & 40.393 & 0.283 & 0.597 & 0.223 & 0.314 \\
MNIST-CIFAR & $\rho$ (ASAM) & 0.01 & 10.0 & 26.3 & 34.053 & 0.286 & 0.769 & 0.225 & 0.261 \\
Double MNIST & Learning Rate & 0.001 & 0.25 & 88.13 & 96.39 & 0.282 & 0.801 & 0.529 & 0.631 \\
Double MNIST & High LR + L2 Reg. & 1e-05 & 0.001 & 96.68 & 97.75 & 0.664 & 0.89 & 0.641 & 0.836 \\
Double MNIST & Batch Size & 0.01 & 0.3333 & 85.39 & 96.62 & 0.279 & 0.514 & 0.375 & 0.471 \\
Double MNIST & Momentum & 0.1 & 0.99 & 85.58 & 95.71 & 0.278 & 0.755 & 0.471 & 0.57 \\
Double MNIST & L1 Regularization & 1e-06 & 0.001 & 87.64 & 95.87 & 0.28 & 0.932 & 0.517 & 0.914 \\
Double MNIST & L2 Regularization & 0.0001 & 0.05 & 86.83 & 97.08 & 0.285 & 0.727 & 0.504 & 0.896 \\
Double MNIST & $\gamma$ (Focal Loss) & 0.001 & 25.0 & 85.01 & 95.89 & 0.275 & 0.661 & 0.456 & 0.58 \\
Double MNIST & $\rho$ (ASAM) & 0.01 & 10.0 & 85.71 & 95.82 & 0.277 & 0.626 & 0.467 & 0.488 \\
\bottomrule
\end{tabular}
}
\label{tbl:lr_result_ranges_reg}
\end{table*}

\begin{figure}[t]
    \centering
    \begin{subfigure}{0.45\linewidth}
      \centering
\includegraphics[width=\linewidth]{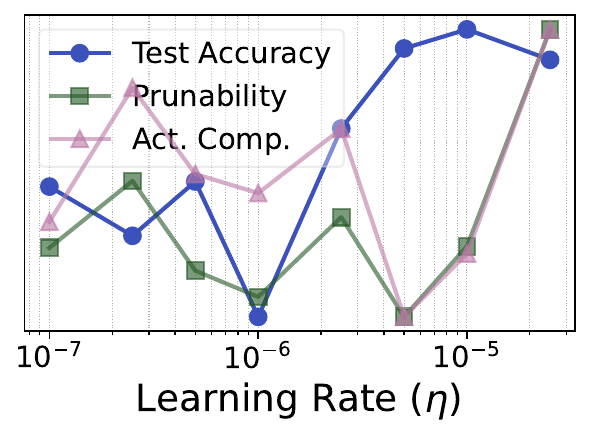}
    \end{subfigure}
    \centering
    \begin{subfigure}{0.45\linewidth}
      \centering
\includegraphics[width=\linewidth]{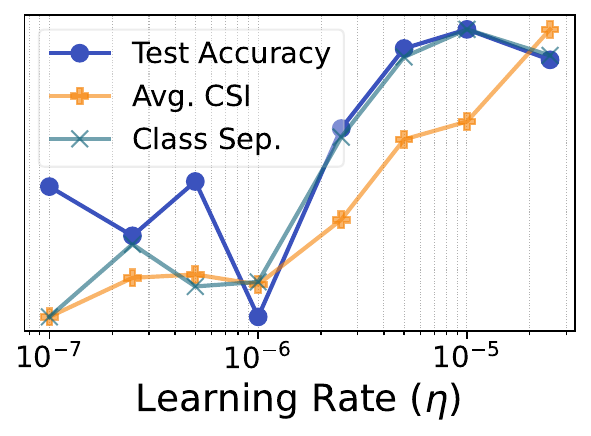}
    \end{subfigure}
    \vspace{-3mm}
    \caption{(Left) Effects of learning rate on OOD performance (unbiased test acc.), network prunability, and representation properties with the Waterbirds dataset.\vspace{-5mm}}
    \label{fig:waterbirds_results}
\end{figure}

\begin{figure*}[t]
  \centering
  \begin{subfigure}{\textwidth}
  \centering
  \includegraphics[width=\textwidth]{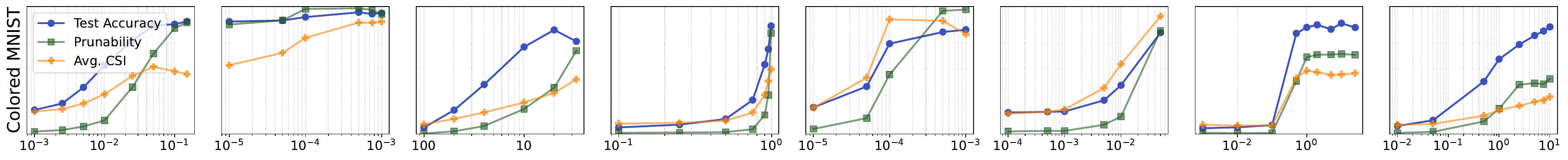}
\end{subfigure}
\begin{subfigure}{\textwidth}
  \centering
  \includegraphics[width=\textwidth]{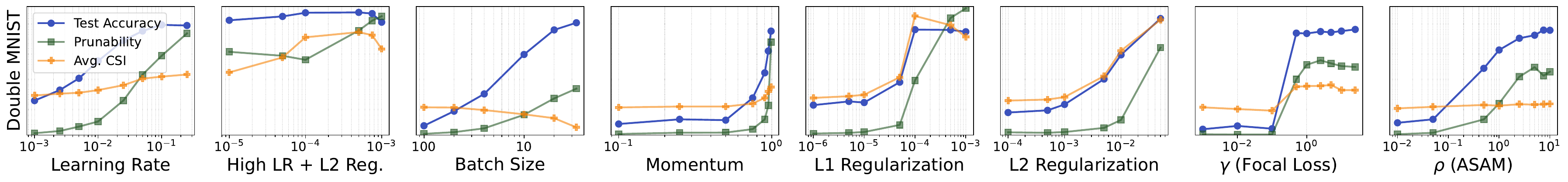}
\end{subfigure}
\begin{subfigure}{\textwidth}
  \centering
\includegraphics[width=\textwidth]{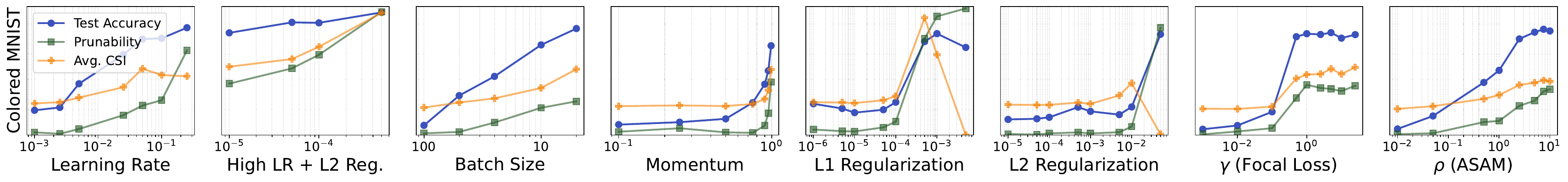}
\end{subfigure}
\vspace{-6mm}
\caption{Comparing various hyperparameters, regularization methods, and losses in terms of OOD robustness, compressibility, and core feature utilization in Double MNIST dataset with a ResNet18 model (top), and Colored MNIST dataset with a ResNet50 model (bottom). $y$-axes are normalized within each figure for each variable.\vspace{-3mm}}
  \label{fig:reg_results_suppl}
\end{figure*}

\begin{figure*}[t]
  \centering
  \begin{subfigure}[b]{\textwidth}
    \centering
    \includegraphics[width=\textwidth]{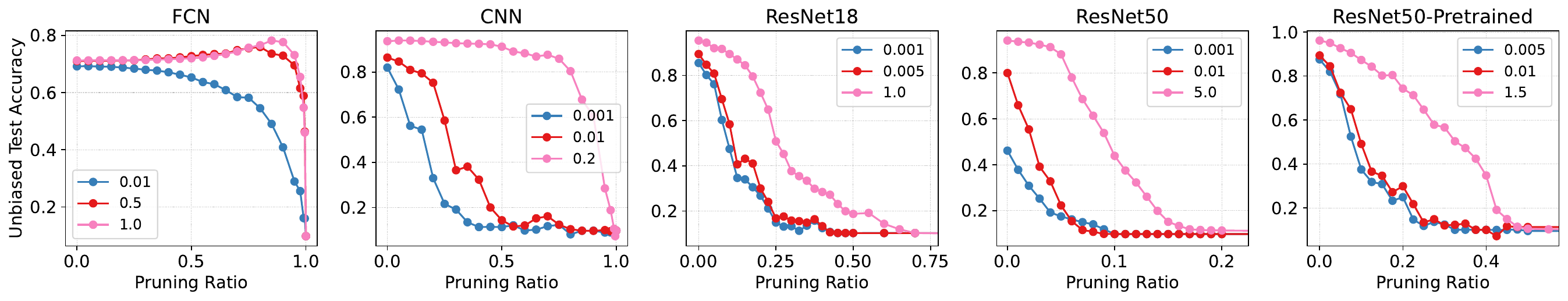}
  \end{subfigure}
  \begin{subfigure}[b]{\textwidth}
    \centering
    \includegraphics[width=\textwidth]{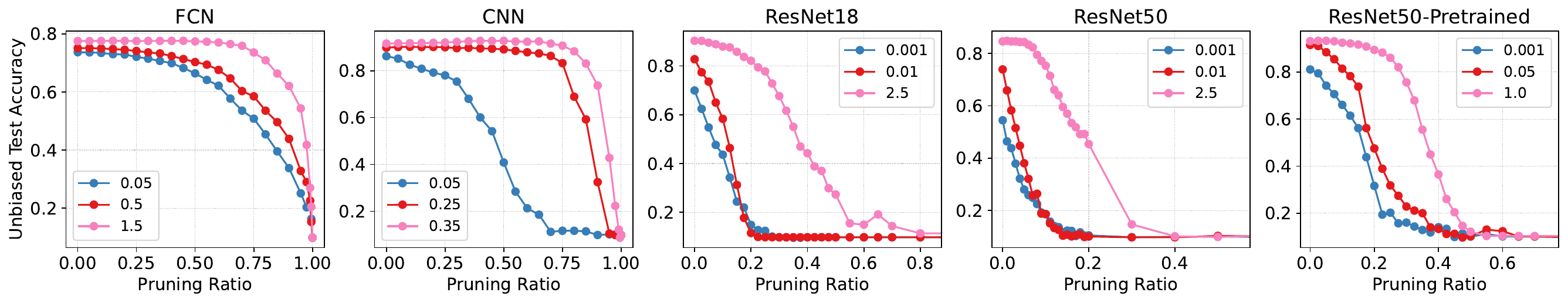}
  \end{subfigure}
  \vspace{-5mm}
  \caption{Effects of column pruning on models trained on Double MNIST (top) and Colored MNIST (bottom) datasets, under various learning rates. $x$-axes are modified for visualization.}
  \label{fig:column_pruning}
  \vspace{-2mm}
\end{figure*}
\begin{figure*}[h!]
  \centering
  \begin{subfigure}[b]{\textwidth}
    \centering
    \includegraphics[width=\textwidth]{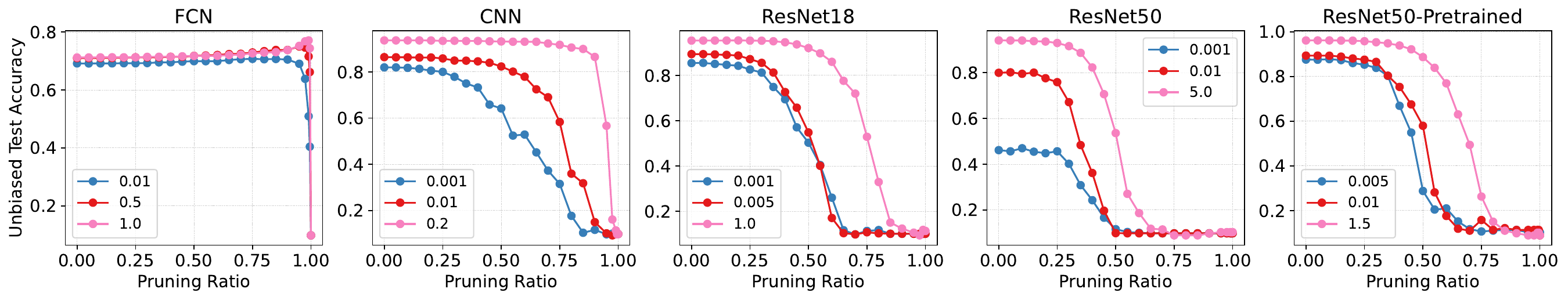}
  \end{subfigure}
  \begin{subfigure}[b]{\textwidth}
    \centering
    \includegraphics[width=\textwidth]{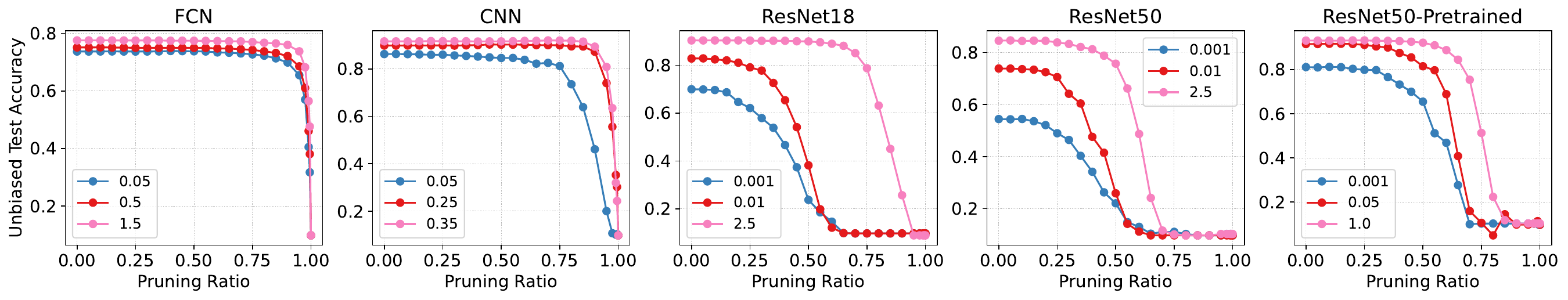}
  \end{subfigure}
  \vspace{-5mm}
  \caption{Effects of magnitude pruning on models trained on Double MNIST (top) and Colored MNIST (bottom), under various LRs.}
  \label{fig:magnitude_pruning}
  \vspace{-6mm}
\end{figure*}

\subsection{Additional Experiment Results}
Here we present additional experimental results that were omitted in the main paper due to space concerns. \cref{fig:moonstar_results} present our results with the synthetic moon-star dataset, \cref{fig:corcifar_results} presents our results with the Corrupted CIFAR-10 dataset, \cref{fig:waterbirds_results} presents our results with the Waterbirds dataset, and \cref{fig:reg_results_suppl} includes additional results for comparing the effects of various hyperparameters and regularization methods. Moreover, \cref{fig:column_pruning} and \cref{fig:magnitude_pruning} provide a more in-depth look at the performance of models in terms of unbiased test accuracy under various pruning ratios, using column and magnitude pruning respectively. 

To show that our results are not limited to training using SGD with a constant LR, we present qualitatively identical experiment results in \cref{fig:lr_annealing} using the Colored MNIST dataset and ResNet18 model, where the initial learning rate ($x$-axis) is multiplied by 0.1 after 1000th iteration. Similar results using Adam optimization algorithm are presented in  \cref{fig:lr_adam}. Additionally, using the same setting, in \cref{fig:lr_longtrain} we show that training models for longer according to an additional criterion (CE loss $< 1e-5$) produces qualitatively identical results as test accuracy changes very little beyond convergence for both low and high LR models. Finally, \cref{fig:cmnist_altcomp} demonstrates that alternative choices to characterize parameter and representation compressibility, such as \qkcomp, sparsity, and the recently proposed PQ-Index \cite{diaoPruningDeep2023} produce qualitatively identical results.

\begin{figure}[t]
    \centering
    \begin{subfigure}{0.45\linewidth}
      \centering
\includegraphics[width=\linewidth]{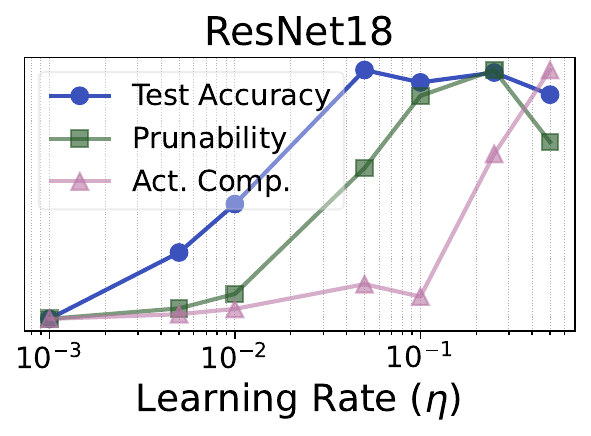}
    \end{subfigure}
    \centering
    \begin{subfigure}{0.45\linewidth}
      \centering
\includegraphics[width=\linewidth]{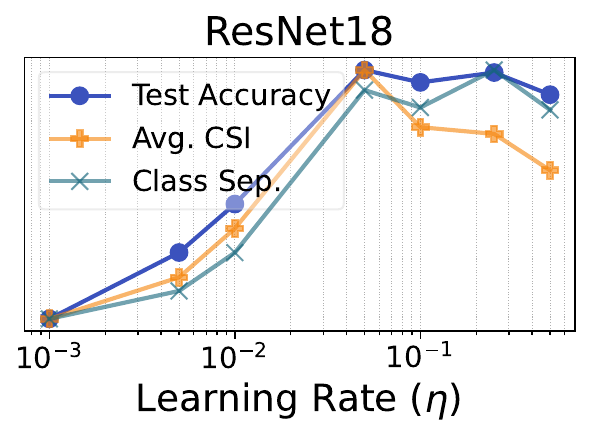}
    \end{subfigure}
    \caption{Effects of learning rate on OOD performance (unbiased test acc.), network prunability, and representation properties with a learning rate annealing setting, where the LR is multiplied by 0.1 after 1000th iteration.}
    \label{fig:lr_annealing}
    \vspace{-4mm}
\end{figure}

\begin{figure}[t]
    \centering
    \begin{subfigure}{0.45\linewidth}
      \centering
\includegraphics[width=\linewidth]{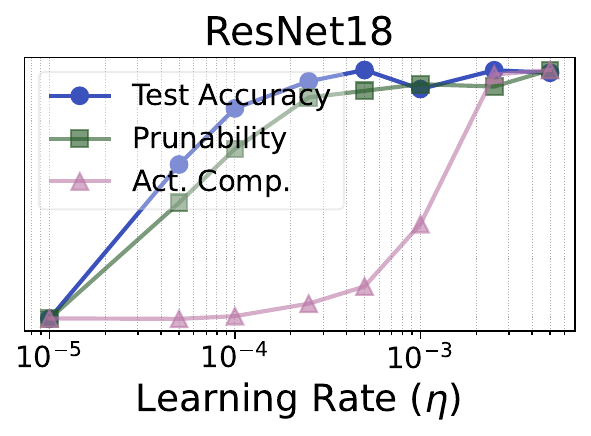}
    \end{subfigure}
    \centering
    \begin{subfigure}{0.45\linewidth}
      \centering
\includegraphics[width=\linewidth]{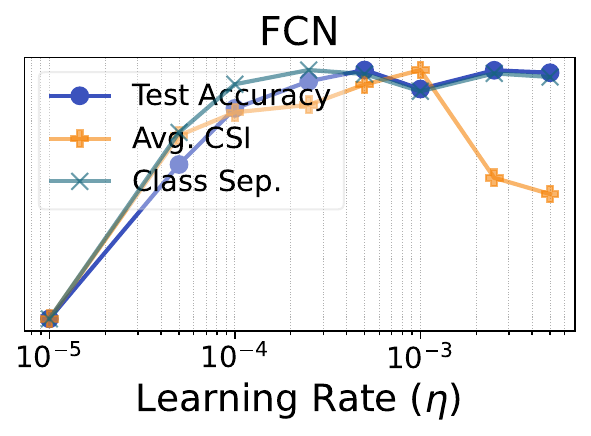}
    \end{subfigure}
    \caption{Effects of learning rate on OOD performance (unbiased test acc.), network prunability, and representation properties with an Adam optimizer, with $\beta_1=0.9, \beta_2=0.999$.}
    \label{fig:lr_adam}
    \vspace{-4mm}
\end{figure}

\begin{figure}[t]
    \centering
    \begin{subfigure}{0.45\linewidth}
      \centering
\includegraphics[width=\linewidth]{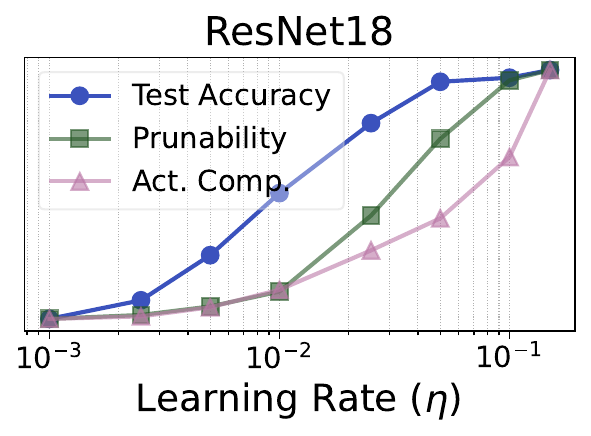}
    \end{subfigure}
    \centering
    \begin{subfigure}{0.45\linewidth}
      \centering
\includegraphics[width=\linewidth]{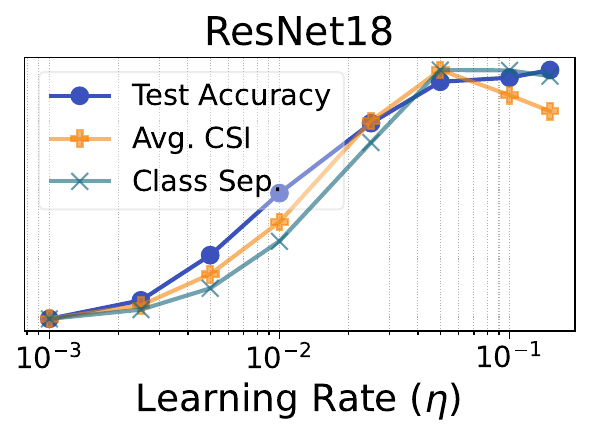}
    \end{subfigure}
    
    \begin{subfigure}{0.95\linewidth}
      \centering
\includegraphics[width=\linewidth]{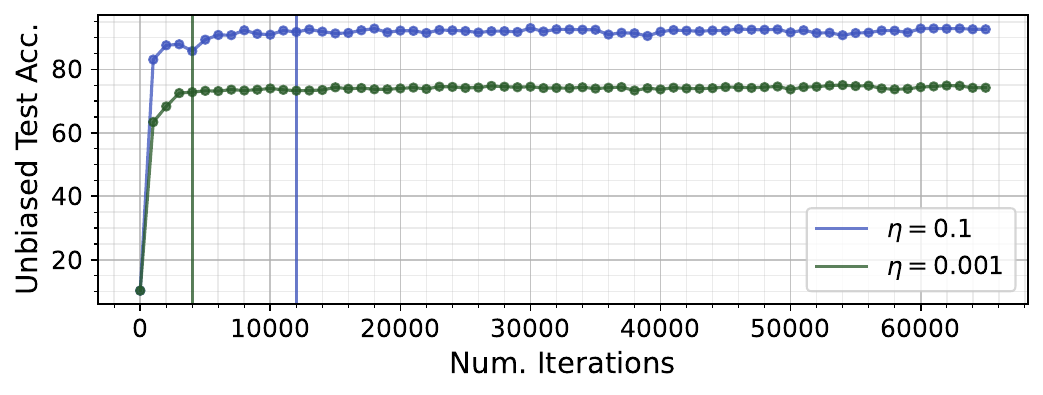}
    \end{subfigure}
    \caption{(Top) Effects of learning rate on OOD performance (unbiased test acc.), network prunability, and representation properties when trained for $100\%$ training accuracy and $<0.00001$ training loss. (Bottom) Test accuracy does not meaningfully change beyond convergence (vertical lines correspond to the point where 100\% was reached).}
    \label{fig:lr_longtrain}
    \vspace{-4mm}
\end{figure}

\begin{figure*}[t]
    \centering
    \begin{subfigure}{\linewidth}
      \centering
\includegraphics[width=\linewidth]{img/cmnist_main_prune.pdf}
    \end{subfigure}
    \centering
    \begin{subfigure}{\linewidth}
      \centering
\includegraphics[width=\linewidth]{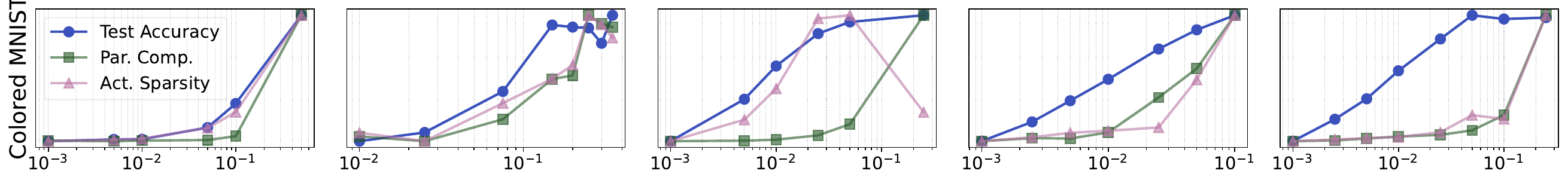}
    \end{subfigure}
    \begin{subfigure}{\linewidth}
      \centering
\includegraphics[width=\linewidth]{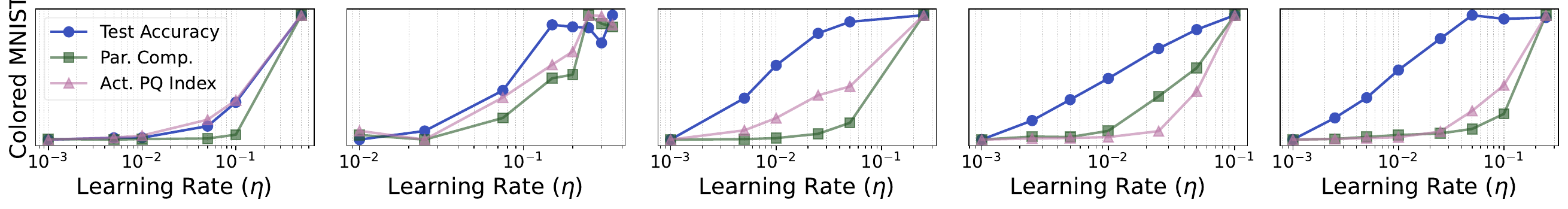}
    \end{subfigure}
    \caption{Utilizing alternative notions of parameter and representation compressibility such as prunability,  \qkcomp \ (with $q=2, \kappa=0.1$), sparsity, and the recently proposed PQ-Index (with $p=2, q=1$).}
    \label{fig:cmnist_altcomp}
    \vspace{-1mm}
\end{figure*}
\begin{figure*}[t]
\centering
\includegraphics[width=\linewidth]{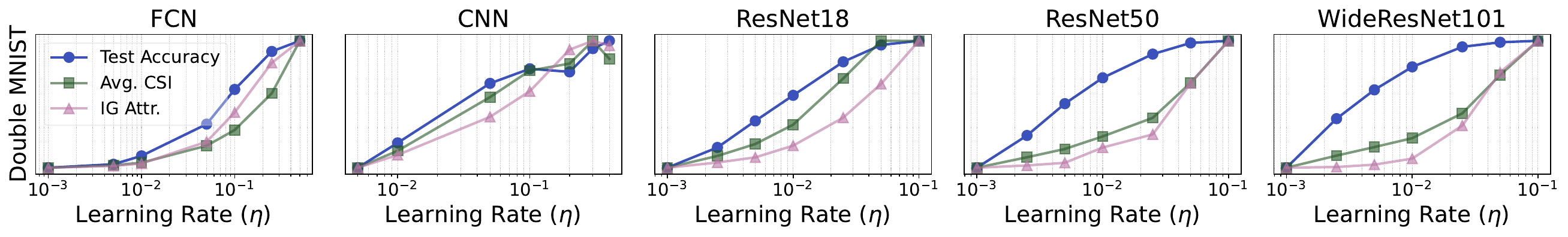}
\vspace{-6mm}
    \caption{Comparing CSI vs. input attribution to core features (\%), using Integrated Gradients.\vspace{-2mm}}
    \label{fig:cifar10_attr_vs_csi}
\vspace{-2mm}
\end{figure*}
\begin{figure}[t]
\centering
    \begin{subfigure}{0.49\linewidth}
\includegraphics[width=\linewidth]{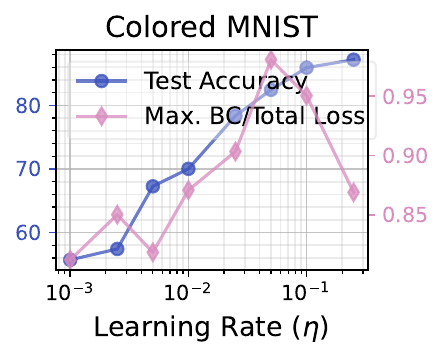}
\end{subfigure}
    \begin{subfigure}{0.45\linewidth}
\includegraphics[width=\linewidth]{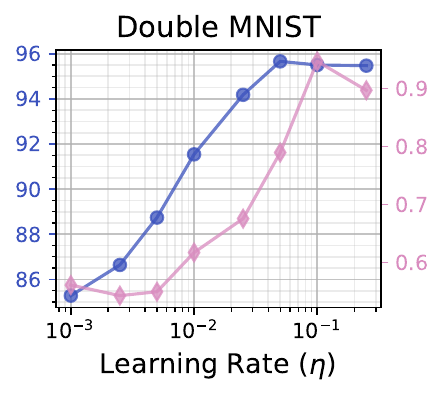}
\end{subfigure}
\vspace{-4mm}
    \caption{Examining the effect of LR on unbiased test accuracy and BC loss ratio in ResNet18 and Double MNIST dataset, as well as ResNet50 and Colored MNIST dataset.}
    \label{fig:lr_unbiased}
\vspace{-5mm}
\end{figure}

\paragraph{Optimizers \& LR schedules}
We confirm our findings on robustness to SCs and comp. extend to modern and standard training setups. \cref{fig:rebuttal_figs_optimizers} (left) shows that core benefits persist with ResNet50 (Colored MNIST) using the PSGD (Kron) optimizer and a WSD LR schedule. \cref{fig:rebuttal_figs_optimizers} (right) shows that our key findings also hold under a standard CIFAR-100 setup with AdamW, cosine annealing, weight decay, and validation set based model selection, addressing concerns about reliance on constant LR SGD.

\cref{fig:rebuttal_figs_lr_soup} (left) compares ResNet18 models trained on CIFAR-10 with a starting LR of $0.1$, which is decreased to $0.001$ by step $S$. We plot the eventual unbiased test performance of a model as function of $S$, with the performance of a model trained by a constant LR of $0.1$ vs. $0.001$ depicted as horizontal lines for reference. The results show that the effects of LR are almost completely integrated by $S = 1000$. \cref{fig:rebuttal_figs_lr_soup} (right) shows that creating a ``model soup'' \cite{wortsmanModelSoups2022} is another way of obtaining robustness vs. compressibility disentanglement.

\begin{figure}[ht]
\centering
\begin{subfigure}{0.45\linewidth}
  \includegraphics[width=\linewidth]{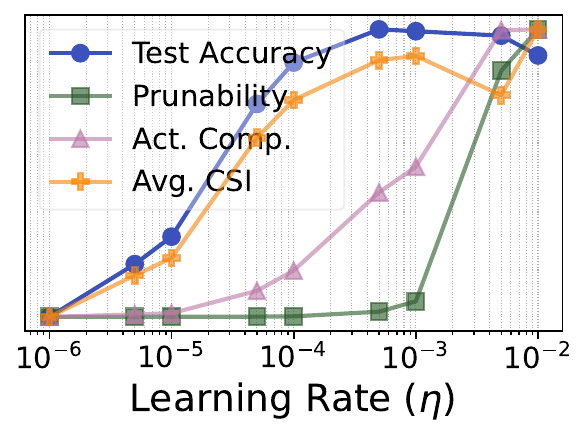} 
\end{subfigure}
\hfill %
\begin{subfigure}{0.45\linewidth}
  \includegraphics[width=\linewidth]{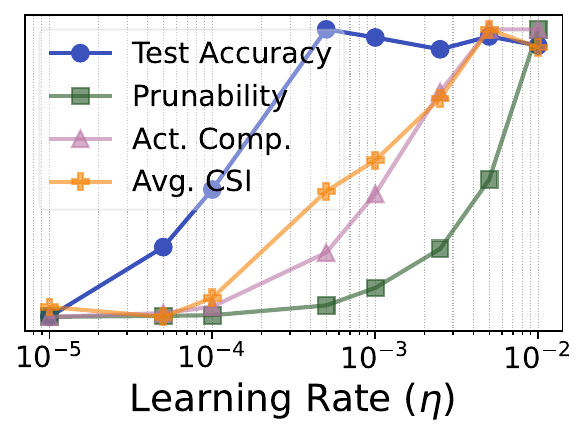} 
\end{subfigure}
\caption{Experiments with alternative optimizers, schedulers, and convergence criteria.}\vspace{-7mm}
\label{fig:rebuttal_figs_optimizers}
\end{figure}

\begin{figure}[ht]
\centering
\begin{subfigure}{0.45\linewidth}
  \includegraphics[width=\linewidth]{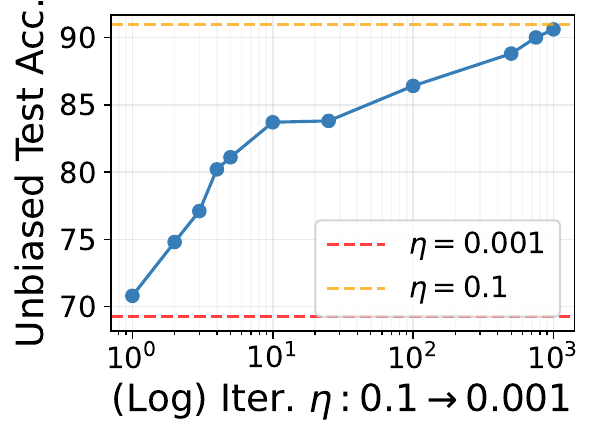} 
\end{subfigure}
\begin{subfigure}{0.45\linewidth}
  \includegraphics[width=\linewidth]{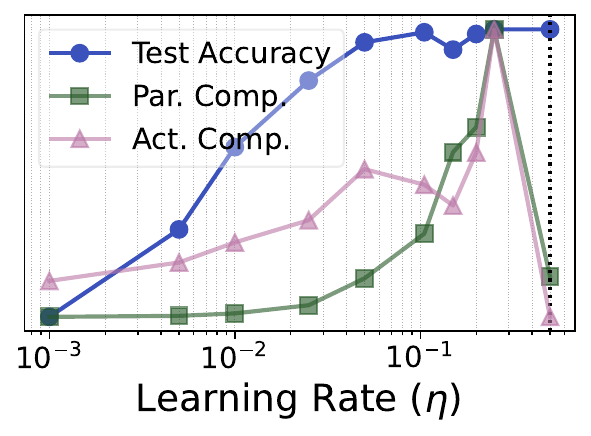} 
\end{subfigure}
\caption{(Left) Unbiased test set performance as a function of LR reduction. (Right) Disentangling compressibility and robustness through a ``model soup'' (the rightmost model).}\vspace{-7mm}
\label{fig:rebuttal_figs_lr_soup}
\end{figure}
\section{Additional Details and Results Regarding Neural Network Attribution}
\label{apx:attribution}

One of the most commonly used methods include Integrated Gradients (IG) \cite{sundararajanAxiomaticAttribution2017}. Given a predictor $f$, an input $\vc{x}$, and a baseline input $\vc{x}'$, IG for the $i$'th component of $\vc{x}$ is computed as follows:
$$
\mathrm{IG}_i(\vc{x}) := (x_i - x'_i) \times \int_{\alpha=0}^1 \frac{\partial f(\vc{x}' + \alpha (\vc{x} - \vc{x}'))}{\partial x_i} d\alpha.
$$
Intuitively, this corresponds to integrating the sensitivity of the output to changes in $x_i$ throughout linear interpolation from $x'$ to $x$. See \cite{sundararajanAxiomaticAttribution2017} for a justification of IG's methodology, and see \cite{mandlerReviewBenchmark2024} for strengths and weaknesses of various attribution methods. To investigate whether our results are an artifact of using IG as our attribution method, we visually compare the attributions computed by Integrated Gradients (IG) and another prominent attribution method, DeepLift \cite{shrikumarLearningImportant2019}, for CIFAR-10 samples under ResNet18 models in \cref{fig:ig_vs_deeplift}. The two methods produce identical results for the purposes of this paper. Both methods are implemented using the \texttt{captum} package for \texttt{PyTorch} framework \cite{kokhlikyan2020captum}.

We can utilize attribution methods for convergent validation of class-selectivity index as a measure of spurious feature utilization. Although in datasets such as Colored MNIST pixels for spurious and core features overlap, they are distinct in others such as Double MNIST. Thus, we can compute input attribution on Double MNIST and through normalization we can determine how much (\ie what percentage) of models' attribution is on the spurious vs. core feature. We can then see whether the patterns demonstrated by CSI parallel that computed through input attribution. \cref{fig:cifar10_attr_vs_csi} shows a comparison of the two metrics across five datasets and LRs for Double MNIST dataset. Remarkably, the two demonstrate qualitatively identical patterns, confirming CSI as a useful metric of core feature utilization.

\paragraph{Creation of attribution maps for CIFAR datasets} We train a ResNet18 model using a low vs. high LR with \textit{10 different seeds} on CIFAR-10 and CIFAR-100 datasets. Then, we extract those samples in the test set which have been correctly predicted by $>.75$ of the high LR models and $<.25$ of low LR ones. Then, we investigate the attribution maps of low vs. high LR models in~\cref{fig:attr_demonst}.

\vspace{-1mm}
\subsection{Additional Attribution Visualizations}
We provide additional visualization of attributions for our experiments in the main paper; for Colored MNIST (\cref{fig:cmnist_attr_extended}),  MNIST-CIFAR (\cref{fig:mnistcifar_attr_extended}), Double MNIST (\cref{fig:dmnist_attr_extended}), CelebA (\cref{fig:celeba_attr_extended}), CIFAR-10 (\cref{fig:cifar10_attr_extended}), and CIFAR-100 (\cref{fig:cifar100_attr_extended}) datasets. Notice that as in the main paper, low LR models are more likely to focus on spurious features compared to high LR models.
\begin{figure*}[t]
    \centering
    \begin{subfigure}{\textwidth}
  \centering
\includegraphics[width=\linewidth]{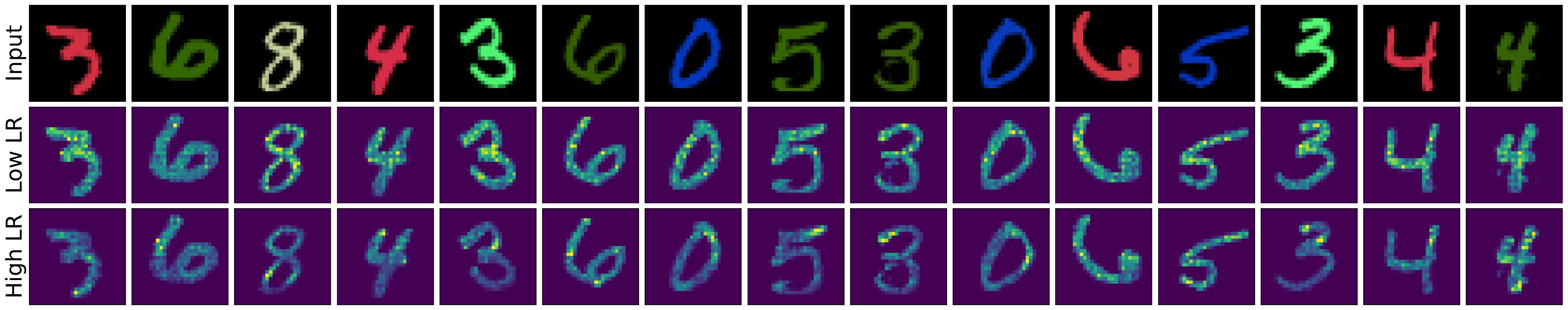}
\end{subfigure}   

\begin{subfigure}{\textwidth}
  \centering
\includegraphics[width=\linewidth]{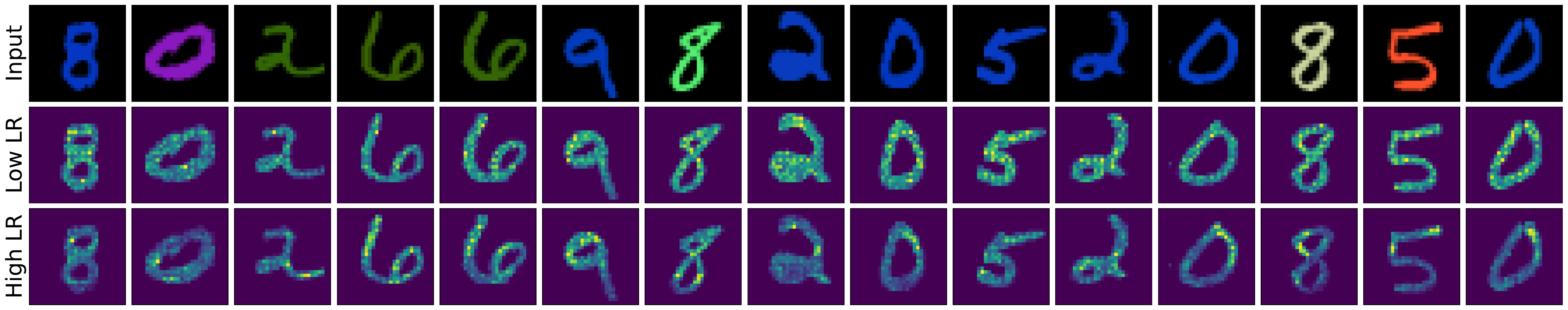}
\end{subfigure}

\begin{subfigure}{\textwidth}
  \centering
\includegraphics[width=\linewidth]{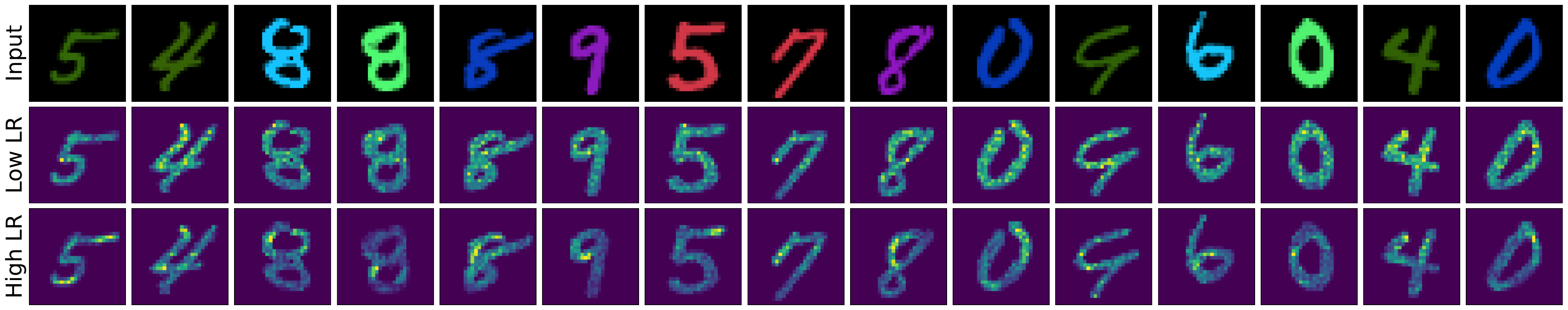}
\end{subfigure}
\vspace{-6mm}
    \caption{Attributions of trained ResNet18 models on Colored MNIST dataset.\vspace{-2mm}}  
    \label{fig:cmnist_attr_extended}
\vspace{3mm}
\end{figure*}
\begin{figure*}[t]
    \centering
    \begin{subfigure}{\textwidth}
  \centering
\includegraphics[width=\linewidth]{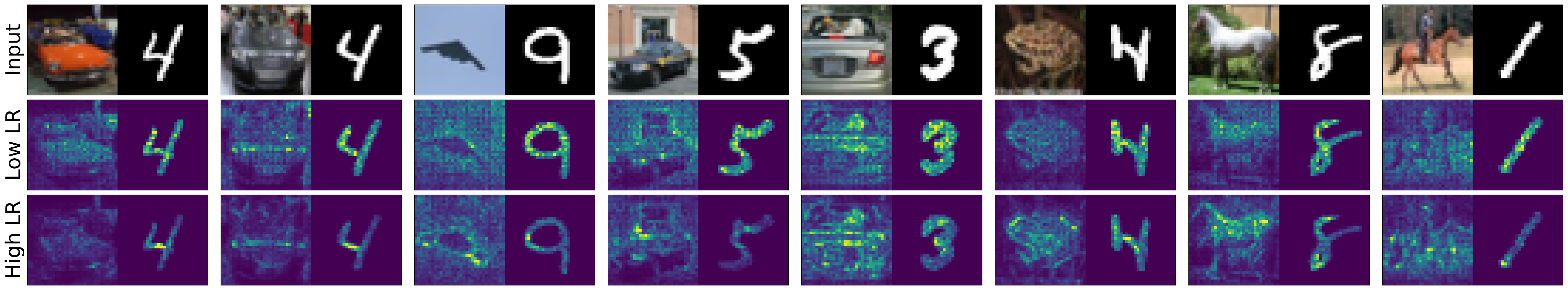}
\end{subfigure}   

\begin{subfigure}{\textwidth}
  \centering
\includegraphics[width=\linewidth]{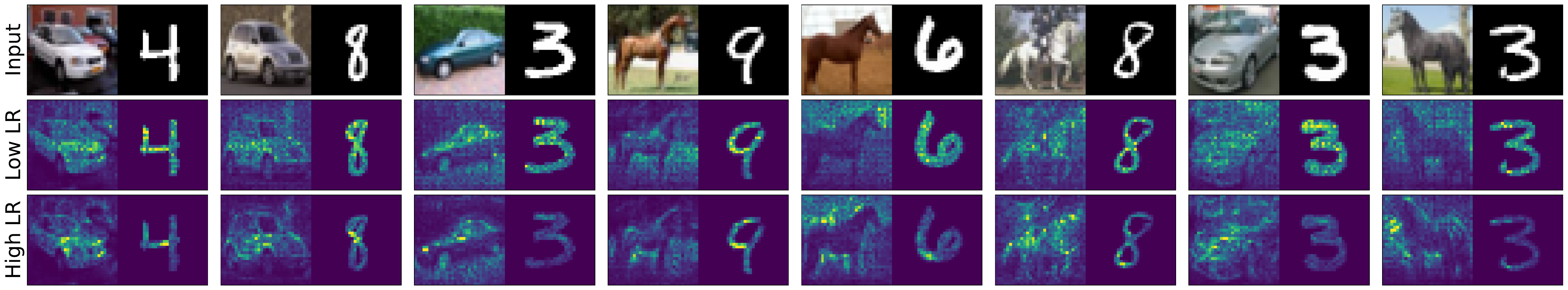}
\end{subfigure}

\begin{subfigure}{\textwidth}
  \centering
\includegraphics[width=\linewidth]{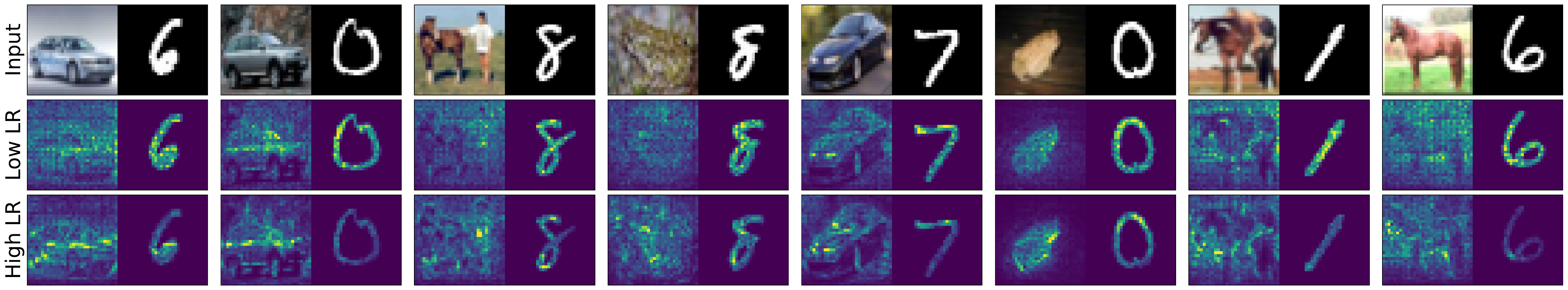}
\end{subfigure}
\vspace{-6mm}
    \caption{Attributions of trained ResNet18 models on MNIST-CIFAR dataset.\vspace{-2mm}} 
    \label{fig:mnistcifar_attr_extended}
\end{figure*}

\begin{figure*}[t]
    \centering
    \begin{subfigure}{\textwidth}
  \centering
\includegraphics[width=\linewidth]{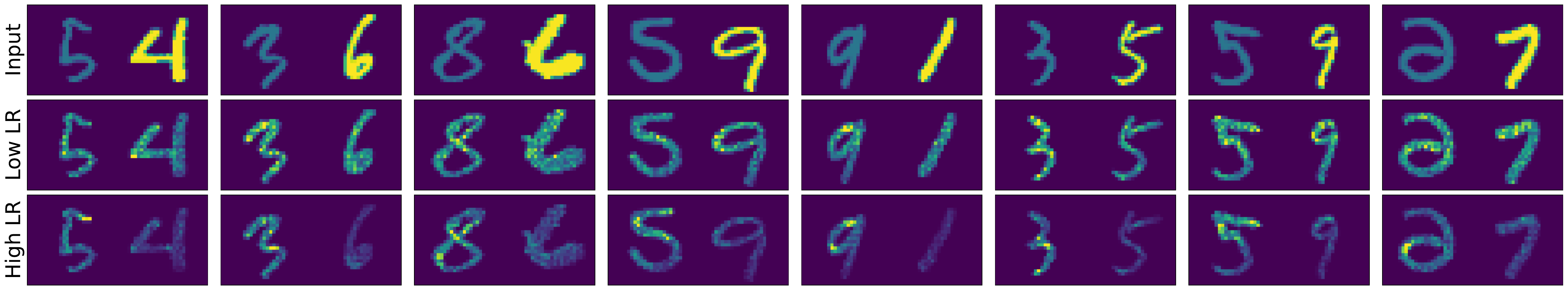}
\end{subfigure}   

\begin{subfigure}{\textwidth}
  \centering
\includegraphics[width=\linewidth]{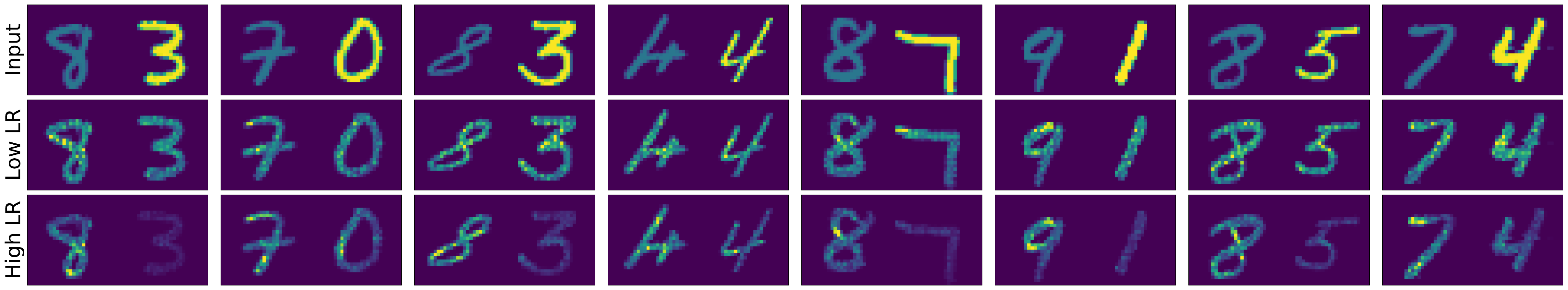}
\end{subfigure}

\begin{subfigure}{\textwidth}
  \centering
\includegraphics[width=\linewidth]{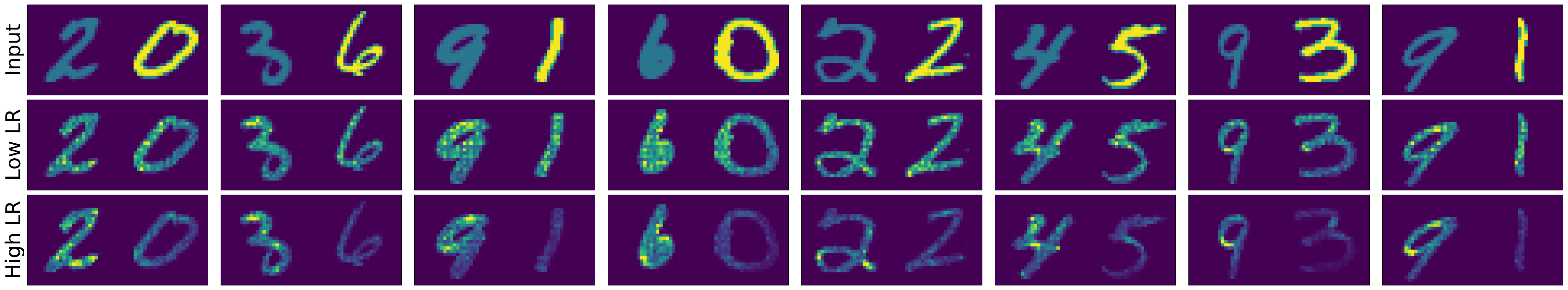}
\end{subfigure}
\vspace{-6mm}
    \caption{Attributions of trained ResNet18 models on MNIST-CIFAR dataset.\vspace{-2mm}}
    \label{fig:dmnist_attr_extended}
\vspace{3mm}
\end{figure*}
\begin{figure*}[t]
    \centering
    \begin{subfigure}{\textwidth}
  \centering
\includegraphics[width=\linewidth]{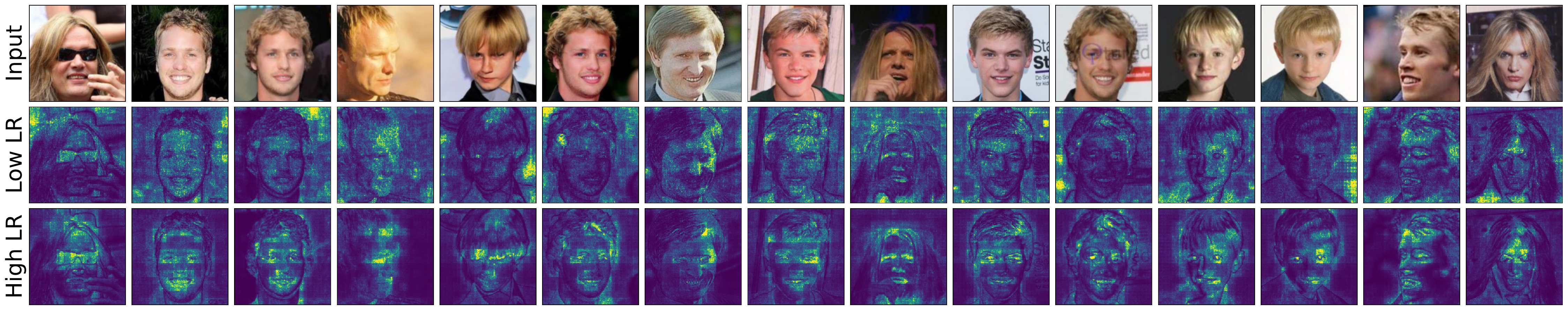}
\end{subfigure}   

\begin{subfigure}{\textwidth}
  \centering
\includegraphics[width=\linewidth]{img/celeba_attr_demonst_15_1.pdf}
\end{subfigure}

\begin{subfigure}{\textwidth}
  \centering
\includegraphics[width=\linewidth]{img/celeba_attr_demonst_15_1.pdf}
\end{subfigure}
\vspace{-6mm}
    \caption{Attributions of trained Swin Transformer models on CelebA dataset.\vspace{-2mm}} 
    \label{fig:celeba_attr_extended}
\end{figure*}

\begin{figure*}[t]
    \centering
    \begin{subfigure}{\textwidth}
  \centering
\includegraphics[width=0.98\linewidth]{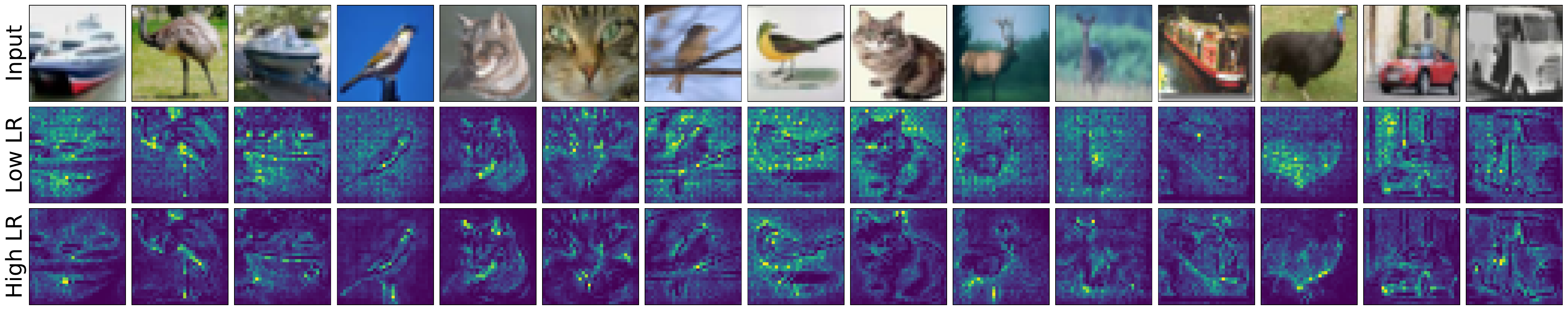}
\end{subfigure}   

\begin{subfigure}{\textwidth}
  \centering
\includegraphics[width=0.98\linewidth]{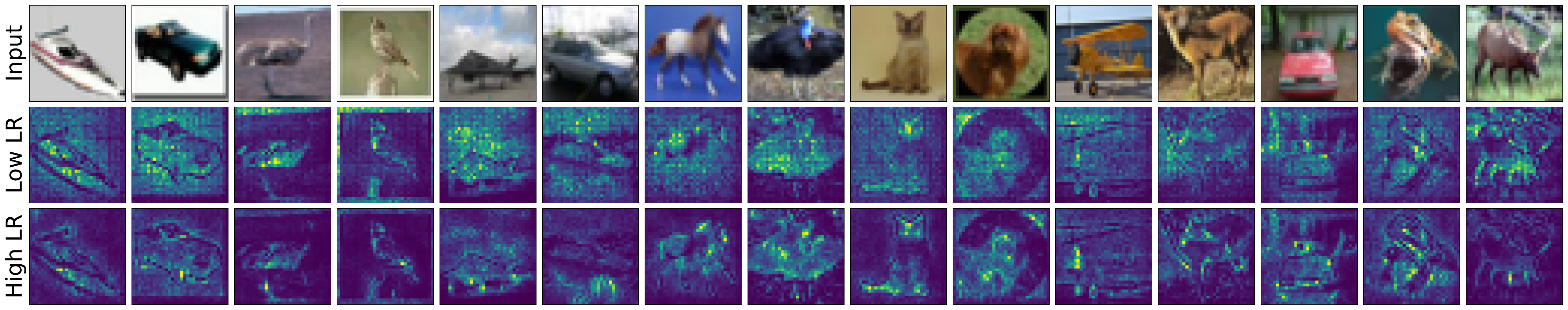}
\end{subfigure}

\begin{subfigure}{\textwidth}
  \centering
\includegraphics[width=0.98\linewidth]{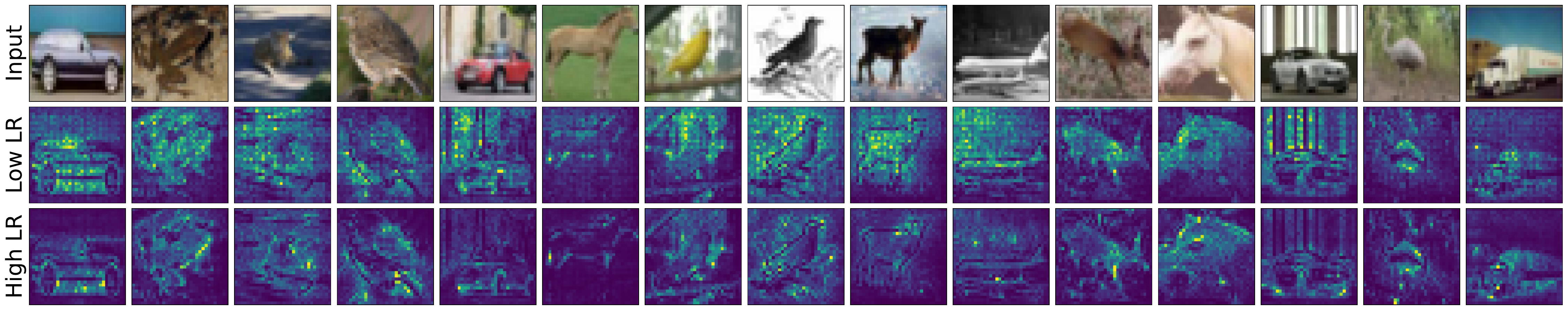}
\end{subfigure}
\vspace{-6mm}
    \caption{Attributions of trained ResNet18 models on CIFAR-10 dataset.\vspace{-2mm}}
    \label{fig:cifar10_attr_extended}
\vspace{3mm}
\end{figure*}
\begin{figure*}[t]
    \centering
    \begin{subfigure}{\textwidth}
  \centering
\includegraphics[width=0.98\linewidth]{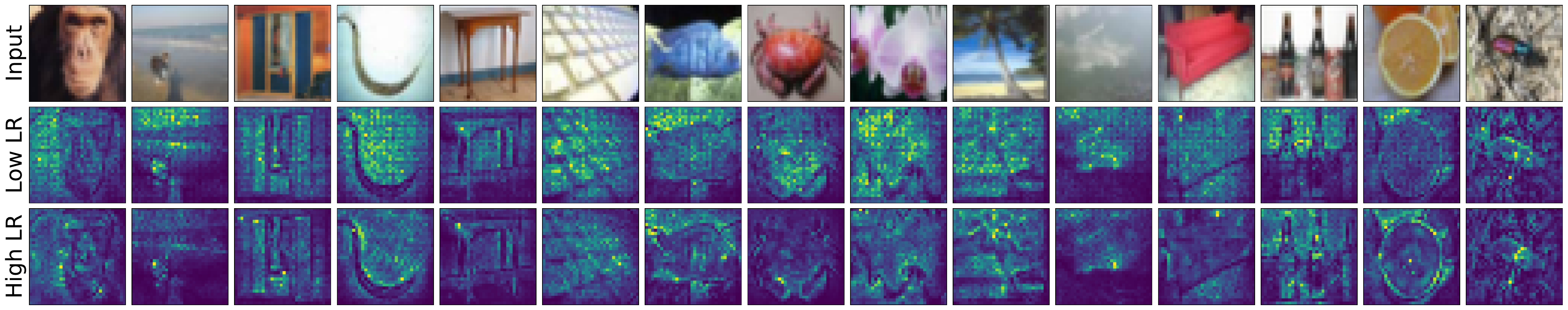}
\end{subfigure}   

\begin{subfigure}{\textwidth}
  \centering
\includegraphics[width=0.98\linewidth]{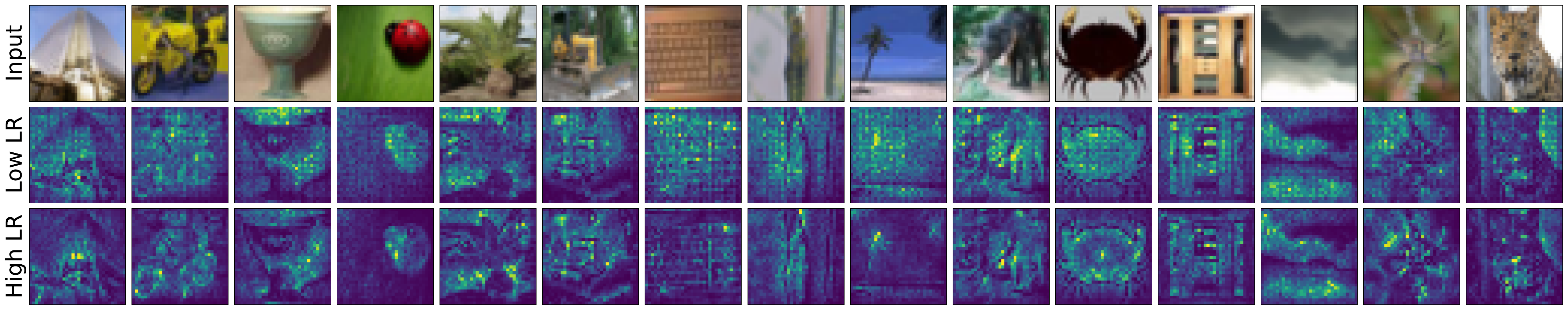}
\end{subfigure}

\begin{subfigure}{\textwidth}
  \centering
\includegraphics[width=0.98\linewidth]{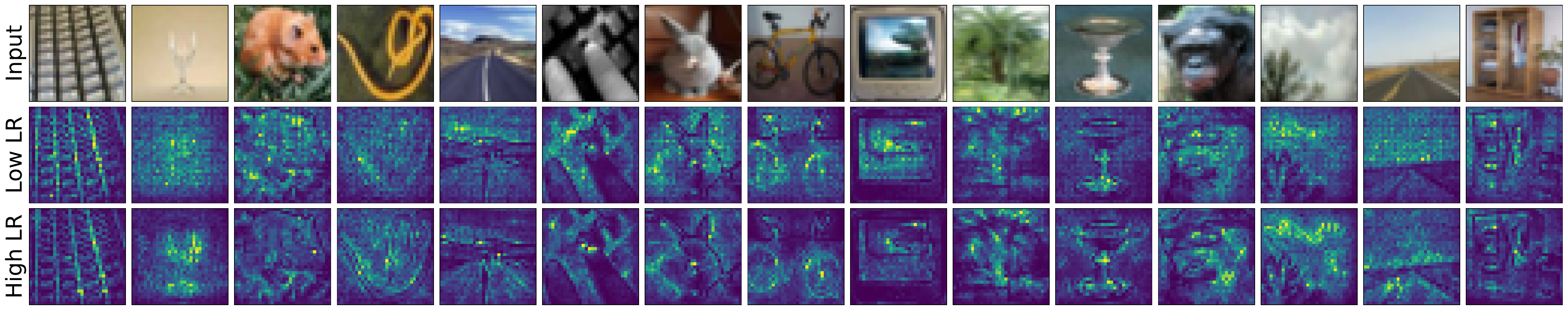}
\end{subfigure}
\vspace{-6mm}
    \caption{Attributions of trained ResNet18 models on CIFAR-100 dataset.\vspace{-2mm}} 
    \label{fig:cifar100_attr_extended}
\end{figure*}

\end{document}